%% file: main.tex
\documentclass[12pt]{caltech/caltech_thesis}

\PassOptionsToPackage{pdfsubject={Caltech PhD thesis}}{hyperref}
\PassOptionsToPackage{pdfkeywords={neural networks, kernel methods, Gaussian processes, optimisation, generalisation, majorise-minimise, functional majorisation, architectural perturbation bounds, Bayes point machines, normalised margin, hyperparameter transfer, neural hardware, uncertainty quantification}}{hyperref}

\input{preamble}
\input{macros}

\begin{document}

\unilogo{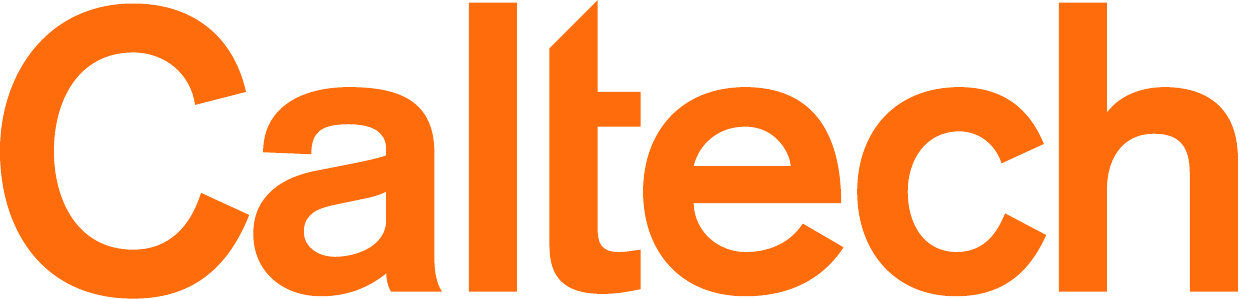}
\input{frontmatter/title}
\input{frontmatter/acknowledgements}
\input{frontmatter/abstract}
\newpage\tableofcontents
\listoffigures
\input{frontmatter/notation}

\mainmatter

\partimage{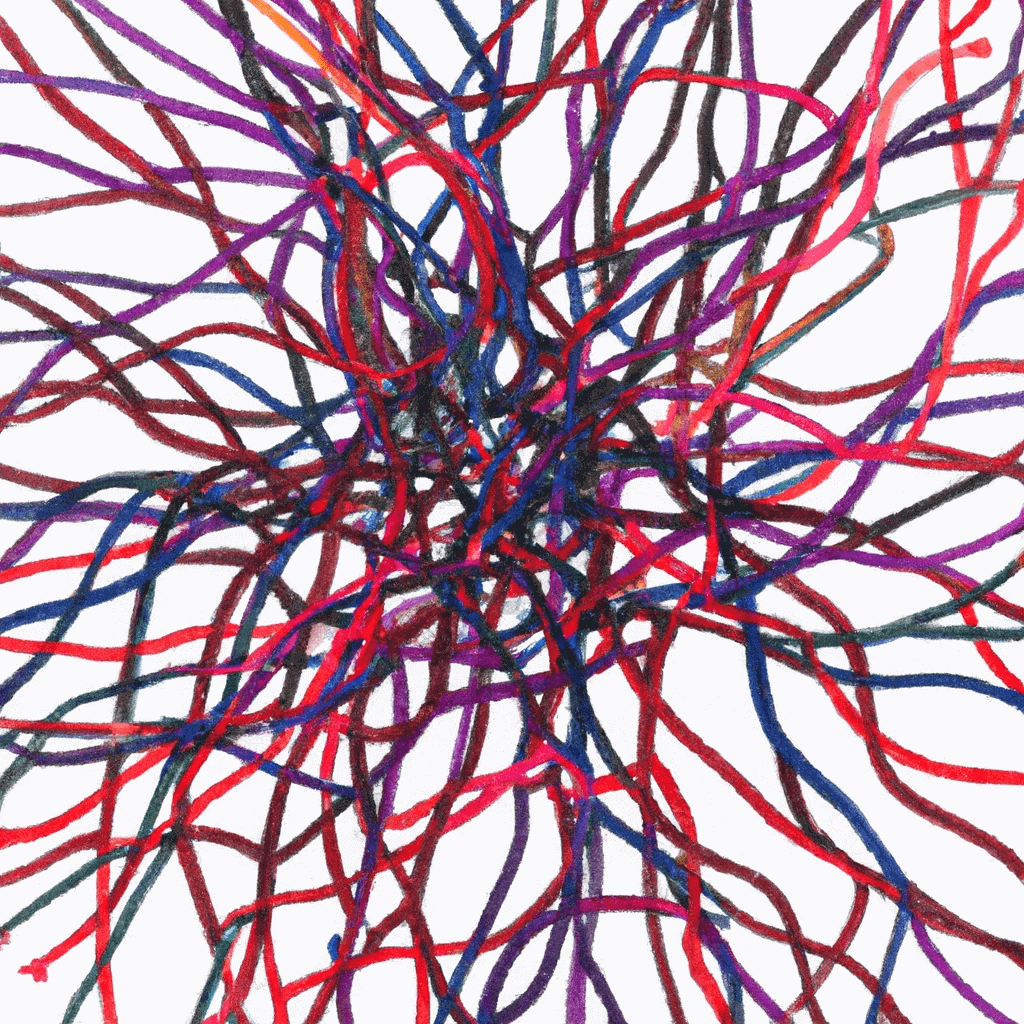}
\partquote{It’s much closer to Freud, the idea that there’s this thin film of consciousness and deliberate reasoning and all this seething stuff underneath.}{Geoffrey E.~Hinton, 2018}
\part{Introduction}
\label{part:intro}

\input{chapters/1-introduction}
\input{chapters/2-function-spaces}
\input{chapters/3-correspondences}

\partimage{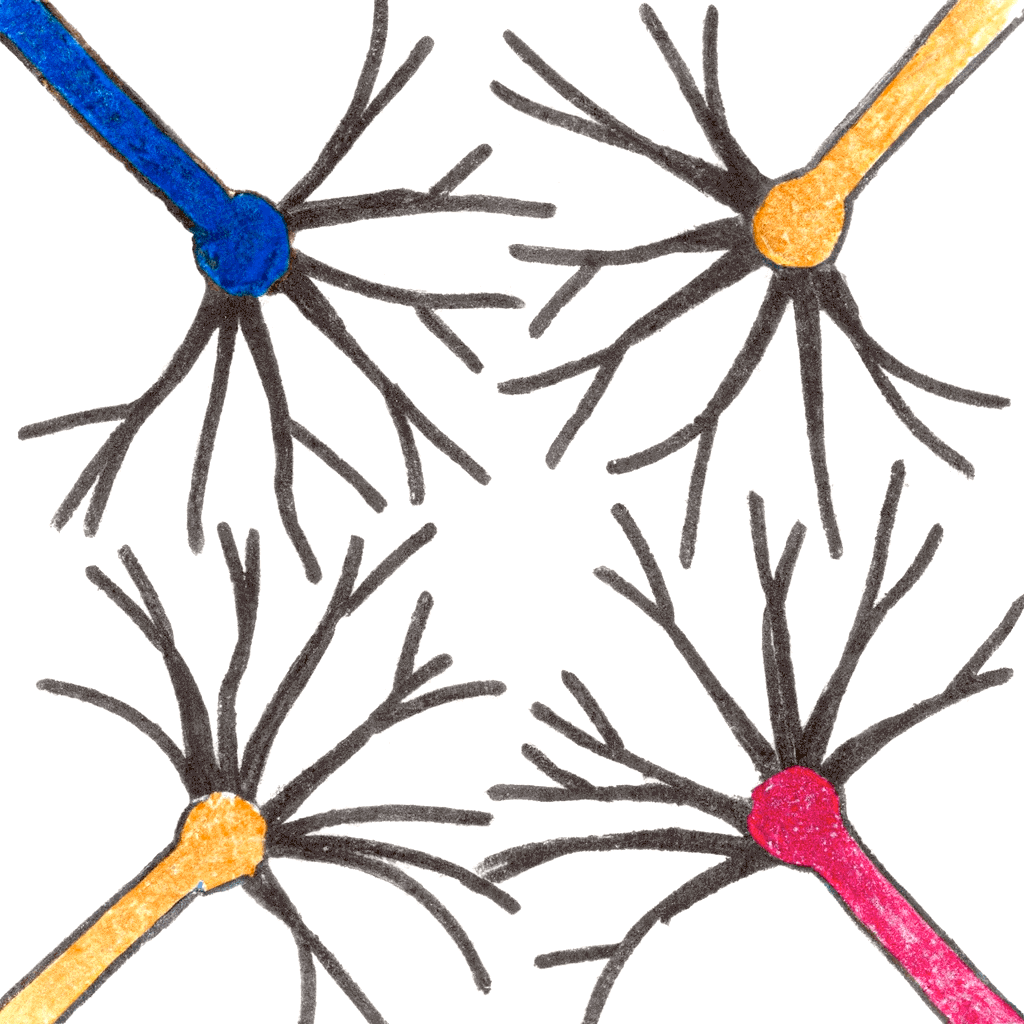}
\partquote{The number of open questions concerning steepest ascent, its ubiquity in the literature of optimization, and its continued use in computation might lead the unwary to think that it was a good thing to do in practice; but we think that in the art of computation it should be considered as a last resort, faute de mieux, as Cauchy might have said.}{Philip Wolfe, 1969}
\part{Optimisation}
\label{part:opt}

\input{chapters/4-optimisation-frameworks}

\input{chapters/5-majorisation-minimisation}
\input{chapters/6-neural-net-optimiser}

\partimage{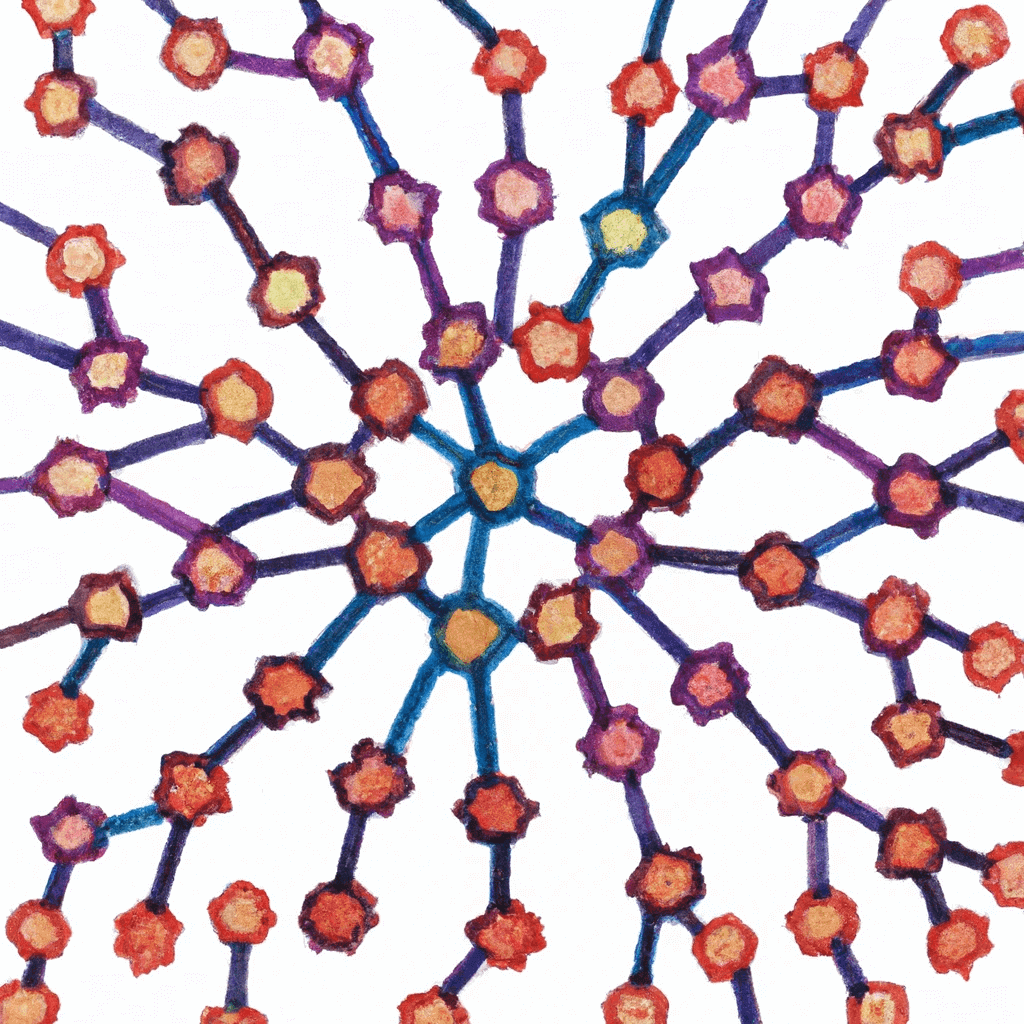}
\partquote{[The brain’s] underlying physical storage devices are capable of a thousand to a million times the capacity manifest in learned behavior... Possibly we should not be looking for models and mechanisms that produce storage economies, but rather ones in which marvels are produced by profligate use of capacity.}{Thomas K.~Landauer, 1986}
\part{Generalisation}
\label{part:gen}

\input{chapters/7-learning-theory-frameworks}
\input{chapters/8-pac-bayes-for-gps}
\input{chapters/9-nn-bpm-correspondence}

\end{document}

%% file: preamble.tex
\usepackage[utf8]{inputenc}
\usepackage[T1]{fontenc}
\usepackage{lmodern}
\usepackage{microtype}
\usepackage{ragged2e}

\usepackage[hyphens]{url}
\usepackage[backend=biber,natbib,style=ext-authoryear,maxbibnames=99,maxcitenames=2]{biblatex}
\DeclareNameAlias{sortname}{given-family}
\DeclareFieldFormat
  [article,inbook,incollection,inproceedings,patent,thesis,unpublished]
  {titlecase:title}{\MakeSentenceCase*{#1}}

\usepackage[hidelinks,pdfusetitle]{hyperref}

\usepackage{enumitem}

\usepackage{tcolorbox}
\tcbset{colback={rgb,255:red,204;green,204;blue,255}, sharp corners, top=10pt, bottom=10pt, boxrule=1pt}

\usepackage{tikz}
\usetikzlibrary{tikzmark,calc}

\usepackage{mathtools}
\usepackage{amsthm}
\usepackage{amsmath}
\usepackage{amssymb}
\usepackage{xcolor}

\theoremstyle{definition}
\newtheorem{definition}{\sffamily Definition}
\numberwithin{definition}{chapter}
\newtheorem{example}{\sffamily Example}
\numberwithin{example}{chapter}
\newtheorem{assumption}{\sffamily Assumption}
\numberwithin{assumption}{chapter}

\numberwithin{approximation}{chapter}
\newtheorem{ansatz}{\sffamily Ansatz}
\numberwithin{ansatz}{chapter}

\theoremstyle{plain}
\newtheorem{theorem}{\sffamily Theorem}
\numberwithin{theorem}{chapter}
\newtheorem{proposition}{\sffamily Proposition}
\numberwithin{proposition}{chapter}
\newtheorem{lemma}{\sffamily Lemma}
\numberwithin{lemma}{chapter}
\newtheorem{corollary}{\sffamily Corollary}
\numberwithin{corollary}{chapter}

%% file: macros.tex
\newcommand\mydots{\makebox[1em][c]{.\hfil.\hfil.}}

\DeclareMathOperator*{\argmin}{arg\,min}

\newcommand{\el}{\mathcal{L}}
\newcommand{\gp}{\textsc{gp}}
\newcommand{\pgp}{P_\mathrm{GP}}
\newcommand{\qgp}{Q_\mathrm{GP}}
\newcommand{\qsph}{Q_\mathrm{sph}}

\usepackage{dsfont}
\usepackage{stmaryrd}

\usepackage{verbatim}

\newcommand{\gibbserr}{\eps_{\mathrm{Gibbs}}}
\newcommand{\bayeserr}{\eps_{\mathrm{Bayes}}}
\newcommand{\bpmerr}{\eps_{\mathrm{BPM}}}
\newcommand{\kl}{\mathrm{KL}}

\newcommand{\econst}{\mathrm{e}}

\newcommand{\eps}{\varepsilon}

\renewcommand{\phi}{\varphi}

\newcommand{\half}{\tfrac{1}{2}}

\newcommand{\R}{\mathbb{R}}
\newcommand{\Sph}{\mathbb{S}}

\newcommand{\sign}{\operatorname{sign}}

\newcommand{\trace}{\operatorname{tr}}

\newcommand{\Id}{\mathbf{I}}

\newcommand{\abs}[1]{\vert {#1} \vert}
\newcommand{\norm}[1]{\Vert {#1} \Vert}

\newcommand{\diff}{\mathrm{d}}
\newcommand{\idiff}{\,\diff}

\newcommand{\Expect}{\operatorname{\mathbb{E}}}

\newcommand{\Probe}{\mathbb{P}}

\newcommand{\normal}{\textsc{normal}}
\newcommand{\bern}{\textsc{bern}}

\renewcommand{\normal}{\textsc{normal}}
\newcommand{\uniform}{\textsc{unif}}

%% file: frontmatter/title.tex
\title{Optimisation \& Generalisation\\in Networks of Neurons}
\author{Jeremy Bernstein}

\degreeaward{Doctor of Philosophy}
\university{California Institute of Technology}
\address{Pasadena, California}
\copyyear{2023}
\defenddate{September 23, 2022}

\orcid{0000-0001-9110-7476}
\website{https://jeremybernste.in/}
\rightsstatement{All rights reserved}

\maketitle[logo]

%% file: frontmatter/acknowledgements.tex
\begin{acknowledgements}

I am grateful to the following people: 
My dear friends and my dear family, without whom this thesis would not have been written. My advisor Yisong Yue. The Yue Crew. My close collaborators Kamyar Azizzadenesheli, Dawna Bagherian, Alex Farhang, Kevin Huang, Yang Liu, Kushal Tirumala and Jiawei Zhao. My internship mentors Ming-Yu Liu, Arash Vahdat, Yu-Xiang Wang and Greg Yang. My thesis committee Ming-Yu Liu, Markus Meister, Matt Thomson and Joel Tropp. My Computation \& Neural Systems cohort Jon Kenny, Matt Rosenberg, Anish Sarma and Tony Zhang---as well as head honchos Pietro Perona and Thanos Siapas. Ollie Stephenson and everyone at Caltech Letters. My co-conspirators David Brown and Tatyana Dobreva. Laura Flower Kim and Daniel Yoder at International Student Programs. Natalie Gilmore in the Graduate Studies Office. Claire Ralph in Computing \& Mathematical Sciences. Athena Castro and Greg Fletcher at Caltech Y. Thank you for your presence, advice, friendship and support, which has enriched my life.

The artwork in this thesis was created by OpenAI's DALL$\cdot$E diffusion model.

\end{acknowledgements}

%% file: frontmatter/abstract.tex
\begin{abstract}
    The goal of this thesis is to develop the optimisation and generalisation theoretic foundations of learning in artificial neural networks. The thesis tackles two central questions. Given training data and a network architecture:
    \begin{enumerate}
        \item Which weight setting will generalise best to unseen data, and why?
        \item What optimiser should be used to recover this weight setting?
    \end{enumerate}
    
    On optimisation, an essential feature of neural network training is that the network weights affect the loss function only indirectly through their appearance in the network architecture. This thesis proposes a three-step framework for deriving novel ``architecture aware'' optimisation algorithms. The first step---termed \textit{functional majorisation}---is to majorise a series expansion of the loss function in terms of functional perturbations. The second step is to derive \textit{architectural perturbation bounds} that relate the size of functional perturbations to the size of weight perturbations. The third step is to substitute these architectural perturbation bounds into the functional majorisation of the loss and to obtain an optimisation algorithm via minimisation. This constitutes an application of the \textit{majorise-minimise meta-algorithm} to neural networks.
    
    On generalisation, a promising recent line of work has applied PAC-Bayes theory to derive non-vacuous generalisation guarantees for neural networks. Since these guarantees control the average risk of ensembles of networks, they do not address which individual network should generalise best. To close this gap, the thesis rekindles an old idea from the kernels literature: the \textit{Bayes point machine}. A Bayes point machine is a single classifier that approximates the aggregate prediction of an ensemble of classifiers. Since aggregation reduces the variance of ensemble predictions, Bayes point machines tend to generalise better than other ensemble members. The thesis shows that the space of neural networks consistent with a training set concentrates on a Bayes point machine if both the network width and normalised margin are sent to infinity. This motivates the practice of returning a wide network of large normalised margin.
    
    Potential applications of these ideas include novel methods for uncertainty quantification, more efficient numerical representations for neural hardware, and optimisers that transfer hyperparameters across learning problems.
\end{abstract}

%% file: frontmatter/notation.tex
\extrachapter{Notation}

\subsection*{Measuring size}

\bgroup
\def\arraystretch{1.5}
\begin{tabular}{p{1in}p{4in}}
$\displaystyle \|x\|_2$ & Euclidean norm of vector $x$\\
$\displaystyle \|W\|_F$ & Frobenius norm of matrix $W$\\
$\displaystyle \|W\|_*$ & operator norm of matrix $W$\\
$\displaystyle \|f\|_\mathrm{RKHS}$ & reproducing kernel Hilbert space norm of function $f$\\
\end{tabular}
\egroup
\vspace{0.5em}

\subsection*{Describing data}
\bgroup
\def\arraystretch{1.5}
\begin{tabular}{p{1in}p{4in}}
$\displaystyle \mathcal{X}$ & input space\\
$\displaystyle \mathcal{Y}$ &
output space\\
$\displaystyle X$ & collection of $m$ train inputs $X=\{x_1,...,x_m\}\in\mathcal{X}^m$\\
$\displaystyle Y$ & vector of $m$ train labels $Y=[y_1,...,y_m]\in\mathcal{Y}^m$\\
$\displaystyle S$ & train set $S = (X,Y) \equiv \{(x_1,y_1),...,(x_m,y_m)\}$\\
$\displaystyle f_X$ & projected function $f_X = [f(x_1),...,f(x_m)]\in\mathcal{Y}^m$\\
\end{tabular}
\egroup
\vspace{0.5em}

\subsection*{Working with kernels and Gaussian processes}
\bgroup
\def\arraystretch{1.5}
\begin{tabular}{p{1in}p{4in}}
$\displaystyle k(\cdot,\cdot)$ & kernel function $k:\mathcal{X}\times\mathcal{X}\to\R$\\
$\displaystyle K_{XX^\prime}$ & Gram matrix $K_{XX^\prime}^{ij} \coloneqq k(x_i,x^\prime_j)$\\
$\displaystyle K_{xX}$ & Gram vector $K_{xX}^{i} \coloneqq k(x,x_i)$\\
$\displaystyle K_{xx}$ & Gram scalar $K_{xx} \coloneqq k(x,x)$\\
\end{tabular}
\egroup
\vspace{0.5em}

\subsection*{Describing neural architecture}
\bgroup
\def\arraystretch{1.5}
\begin{tabular}{p{1in}p{4in}}
$\displaystyle \mathcal{W}$ & weight space\\
$L$ & number of layers\\
$\displaystyle d_l$ & width of $l$th layer\\
$\displaystyle f(\cdot;\cdot)$ & neural network $f: \R^{d_0}\times\mathcal{W}\to\R^{d_L}$\\
\end{tabular}
\egroup

%% file: chapters/1-introduction.tex
\begin{refsection}

\chapter{Finding the Foundations}

\begin{tcolorbox}
This chapter introduces the central goal of this thesis: to find the foundations of optimisation and generalisation in artificial neural networks.
\end{tcolorbox}

Research into artificial neural networks has drawn on various disciplines of science and engineering. Neuroscience and psychology have inspired basic learning frameworks \citep{Sutton1998} as well as specific neural architectures \citep{Fukushima2004NeocognitronAS}. Computer engineering has yielded hardware accelerators that enable both experimentation and applications at larger and larger scale \citep{steinkrau}. And mathematics and statistics offer the toolkits needed to understand the basic properties of learning systems. 

From this breadth of scientific input, a paradigm known as \textit{deep learning} has emerged \citep{SCHMIDHUBER201585,deeplearning}. Deep learning has been driving progress in machine learning applications across science and industry over the course of the last decade. While certain applications deviate from the following schema, at its core, deep learning involves three steps:
\begin{enumerate}
    \item A large dataset of training examples is collected. These examples should in a sense span the richness of behaviour present in the task of interest.
    \item An expressive neural network is constructed. A neural network consists of simple linear building blocks chained together and interspersed with simple elementwise nonlinearities to yield an overall highly complex and nonlinear function. The neural network is parameterised by the weights of the linear building blocks, and adjusting these weights adjusts the function that the network implements.
    \item The error---otherwise known as the \textit{loss}---of the network over the training examples is evaluated, and the mathematical gradient of this error with respect to the network's weights is computed. The weights are then adjusted according to this gradient so as to reduce the error. This step is iterated until the error on the training examples has been made small. 
\end{enumerate}

The fascinating aspect of this procedure, and of learning in general, is that minimising error on \textit{train examples} is often sufficient to attain good performance on previously unseen \textit{test examples}. Furthermore, there is a certain conceptual simplicity: deep learning is just gradient descent on a neural network's error over a set of examples. Despite this simplicity, some of the most basic questions surrounding its underlying mathematics are not resolved. For example:

\begin{enumerate}
    \item[$\langle?\rangle$] \textit{Optimisation.} Given the gradient of a neural network's error, how far and in which direction should the network weights be best adjusted?
    \item[$\langle?\rangle$] \textit{Generalisation.} Which of the functions that a neural network implements will perform best on unseen data. And why?
\end{enumerate}

In practice, these questions are usually addressed by trial-and-error over a set of heuristic techniques. For instance, a few variants of gradient descent are known to work quite well for neural networks. On a given task, one such variant will often perform well, but it is not known in advance which it will be \citep{crowded_valley}. Such trial-and-error has been highly successful---it is responsible for the wealth of deep learning applications that are seen today. Nevertheless, it is a contention of this thesis that, by answering these questions by way of a formal theory, there is potential both to simplify practical workflows as well as to unlock fundamentally new deep learning functionalities.

Part of the reason that these questions are still open is that researchers do not know which theoretical framework should be used to answer them. For optimisation, researchers have attempted to apply such varied frameworks as information geometry \citep{amari} and mirror descent \citep{azizan2018stochastic}. For generalisation, the situation is similar \citep{Prez2020GeneralizationBF}. The aim of this thesis, then, is to develop foundational frameworks and principles that should be used to study learning in artificial neural networks. 

These principles could be useful to the machine learning practitioner, since they could provide her with learning algorithms that generalise better while requiring less arbitrary tuning of hyperparameters. They could be useful to the computer hardware engineer, since they could help him design chips that more effectively support learning. And the principles could be useful to the neuroscientist who is seeking to transfer ideas ``upstream'' to the study of biological neural networks.

The next two sections introduce the steps taken by this thesis toward tackling these questions of optimisation and generalisation in networks of neurons.

\section{Optimisation via perturbation}

\textit{Hyperparameter tuning} is the bane of every deep learning practitioner's existence. A large number of optimisation algorithms have been proposed for neural networks \citep{crowded_valley}, and each has a set of adjustable parameters known as \textit{hyperparameters} that affect the performance of the method. The \textit{learning rate} is the canonical example of a hyperparameter---this controls how strongly the network weights are adjusted in response to the gradient of the network's error. In the absence of compelling theoretical guidance on how to set the learning rate, best engineering practice is to try a logarithmic grid of possibilities and to see what works best \citep{Goodfellow-et-al-2016}. This tuning process inflates the computational cost of applications since a network must be trained many times in order to find a single network that works well.

This thesis argues that the reason learning rate tuning in deep learning is so cumbersome is that a proper \textit{perturbation analysis} of neural architecture is missing. Roughly, what this means is that there is not a simple, computationally tractable means of estimating how sensitive the network's function is to adjustments of its weights. And even given such a sensitivity measure, there is no way to apply it to  derive optimisation algorithms for learning problems. To move beyond this situation, this thesis poses the following question:
\begin{quote}
    How far can the weights of a neural network be perturbed before the function of the network is damaged?
\end{quote}

Answering this question is important for optimisation, since an optimiser must not damage the network that it is training. But the question could be of more general interest, too. It gets at the \textit{precision} with which weights need to be stored, so it could be important for the computer hardware engineer to consider. Furthermore, the question could be interesting to the neuroscientist studying the dynamics of synaptic plasticity in living brains.

This question is tackled in Part \ref{part:opt}. Chapter \ref{chap:perturb} surveys classic iterative optimisation algorithms and shows how they may be put on common footing by way of a perturbation analysis operating in the \textit{weight space} of the optimisation problem. Next, Chapter \ref{chap:maj-min} restricts attention to machine learning optimisation problems where the weights enter the optimisation problem via the machine learning model architecture. The chapter develops a novel technique termed \textit{functional majorisation} that is essentially a perturbation analysis of the loss function operating in the function space of the machine learning model. Finally, Chapter \ref{chap:nn-maj-min} develops novel \textit{architectural perturbation bounds} for deep neural networks. These bounds connect the size of weight perturbations to the size of the induced perturbation in the network function. They may be substituted into the functional majorisation of the loss function to obtain novel \textit{architecture aware} optimisation methods for deep neural networks. These methods address how the learning rate should depend on details such as the depth of the neural network that is being trained.

In short, the thesis develops a perturbation analysis of deep networks through \textit{architectural perturbation bounds}, and shows how these bounds interact with the neural network's error via \textit{functional majorisation}. The resulting optimisation algorithms, obtained by minimising the functional majorisation of the error with respect to weight perturbations, constitute an application of the \textit{majorise-minimise meta-algorithm} \citep{mm} to neural networks.

\section{Generalisation via aggregation}

What allows a machine learning model that has been fit to a finite set of training data to generalise to test examples that it has never seen before? This is, in a sense, the fundamental question of learning. This question is particularly interesting in the case of neural networks that are vastly \textit{over-parameterised}, meaning that they have far more weights than training data. In this case, the neural network may have enough capacity to simply \textit{memorise} its training data, without performing any useful computational processing that could lead to generalisation \citep{Zhang2017UnderstandingDL}. So why, when these kinds of vastly over-parameterised networks are trained, do they generalise regardless?

The study of generalisation in machine learning algorithms has a rich history. For instance, \textit{uniform convergence theory}---dating back to the work of \citet{vcpaper}---attempts to bound the difference between train and test error for all functions in the space of functions in which one is interested. Meanwhile \textit{PAC-Bayes theory} \citep{mcallester1999some} provides another means of bounding the generalisation gap of machine learning algorithms. Unlike uniform convergence bounds, PAC-Bayes bounds are on the average generalisation gap over a distribution of functions. PAC-Bayes bounds have been found to be significantly tighter than uniform convergence bounds \citep{seeger}, while incorporating information about both the training set and the machine learning model architecture in a natural way.

Unfortunately, since PAC-Bayes bounds hold in expectation over distributions of functions, they say nothing about the generalisation of an individual function. To address this shortcoming, this thesis poses the following question:

\begin{quote}
    Given a neural network with the capacity to fit a set of training data in many ways, which of these functions should generalise best?
\end{quote}
    
Answering this question is important for two practical reasons. First, most directly, in many applications one is interested in returning the single network that makes the best possible predictions. And second, less directly, in some applications one is interested in obtaining some measure of the uncertainty of the predictions of this best network. One idea for assessing uncertainty involves training an ensemble of networks in order to measure the variance across their predictions. But then it is important to have a means of ensuring that ensemble members do not all just collapse on to the single best generalising network.

These issues are tackled in Part \ref{part:gen}. Chapter \ref{chap:g-theory} surveys some classic ideas in generalisation theory, including uniform convergence theory and PAC-Bayes theory. Chapter \ref{chap:gp-pac-bayes} develops a PAC-Bayes theory of Gaussian process classification. While this chapter contains little conceptual novelty in comparison to prior work \citep{seeger}, the chapter derives some novel analytical results that will be useful later on. Finally, Chapter \ref{chap:bpm} rekindles the idea of the \textit{Bayes point machine} \citep{bpms}. A Bayes point machine is a single classifier that approximates an ensemble's aggregate prediction. Since aggregation tends to improve ensemble performance, Bayes point machines are thought to generalise significantly better than other predictors. Via a detour through the neural network--Gaussian process correspondence \citep{radford}, the thesis finds that maximising the \textit{normalised margin} of a neural network's training predictions causes the network function to concentrate on a Bayes point machine.

The main conceptual agenda of this part of the thesis is to put forward a novel perspective on generalisation in artificial neural networks as arising from a specific form of approximate Bayesian inference. In particular, by leveraging a statistical characterisation of the neural network function space known as the neural network--Gaussian process correspondence, it can be seen that a single neural network may itself approximate an aggregated predictor with good generalisation properties.

But before all that, the remaining chapters of this first part of the thesis will formally introduce the machine learning problem, as well as the technical tools needed to study it. Chapter \ref{chap:functionspaces} formally introduces neural networks, Gaussian processes and kernel methods, while Chapter \ref{chap:correspondences} introduces various correspondences between these spaces of functions.

\printbibliography[heading=subbibliography]
\end{refsection}

%% file: chapters/2-function-spaces.tex
\begin{refsection}

\chapter{Constructing Spaces of Functions}
\label{chap:functionspaces}

\begin{tcolorbox}
This chapter provides an introduction to three popular machine learning techniques: kernel methods, Gaussian processes and neural networks. The material is expositional and is included for the reader's aid.
\end{tcolorbox}

A machine learning algorithm uses training data to select a function from a space of functions. In order for this procedure to work well, it is important to construct an appropriate space of functions for the algorithm to select from. When choosing a space of functions, there are three main considerations:
\begin{enumerate}
    \item \textit{Leveraging prior information.} How can prior knowledge or belief about the structure of the data be incorporated into the function space?
    
    \item \textit{Measuring complexity.} How can one assess the complexity of a function within the space, in order to choose a simple explanation of the data?
    
    \item \textit{Computational efficiency.} How can one design a space of functions that is cheap to select from in the face of data?
\end{enumerate}

This chapter surveys three popular techniques for constructing spaces of functions: kernel methods, Gaussian processes and neural networks. Each technique embodies a different philosophical approach to the above considerations. Kernel methods use \textit{functional analysis} to incorporate prior knowledge and measure complexity. Gaussian processes employ the tools of \textit{Bayesian probability}. While for neural networks, these considerations are mainly addressed via \textit{empiricism}---the proof of a method's validity is in its pudding.

While kernel methods, Gaussian processes and neural networks each take a different approach to the above questions, correspondences exist between their respective function spaces which allow tools from one to be ported to another. These correspondences are surveyed in Chapter \ref{chap:correspondences}. The present chapter focuses on introducing the function spaces themselves, their complexity measures and their approaches to fitting data.

Before moving on to the techniques, it will first be useful to define the notion of \textit{projecting} a function on to a set of inputs.

\begin{definition}[Function projection]\label{def:project} Given a function $f:\mathcal{X}\to\mathcal{Y}$ and a collection of $m$ inputs $X=\{x_1,...,x_m\}$, the \textit{projection} $f_X$ of the function $f$ on to the input data $X$ is given by:
\begin{equation}
    f_X \coloneqq \big[f(x_1), ..., f(x_m)\big] \in \mathcal{Y}^m.
\end{equation}
\end{definition}
In words: the projection of a function on to a set of inputs is the vector of function outputs across the inputs. Given a function $f(\cdot,w):\mathcal{X}\to\mathcal{Y}$ that is parameterised by a weight vector $w$, the projection on to $X$ is denoted $f_X(w)$.

\section{Kernel methods}

A kernel $k(\cdot,\cdot)$ is a function that measures the degree of similarity between its two arguments. In particular, when $k(x,x^\prime)$ is large then inputs $x$ and $x^\prime$ are similar, and when $k(x,x^\prime)$ is small then the two inputs are dissimilar. The notion of similarity is encoded by the choice of kernel function---a simple example being the Gaussian kernel:

\begin{example}[Gaussian kernel]\label{def:kgauss} For a pair of inputs $x,x^\prime\in\R^n$ and a length scale $\sigma>0$, the Gaussian kernel is given by:
\begin{equation}
    k_{\mathrm{Gaussian}}(x,x^\prime)\coloneqq \exp \left( - \frac{\|x-x^\prime\|_2^2}{2\sigma^2}\right).
\end{equation}
\end{example}

Beyond measuring similarity, a kernel may be used to construct a space of functions. To see this, consider that the Gaussian kernel viewed as a function of $x$ may be interpreted as an unnormalised Gaussian measure centred at the point $x^\prime$. The basic idea is that one may construct a more-or-less arbitrary function by superposing many Gaussian kernels centred at different locations. This idea is illustrated in Figure \ref{fig:kernel}. In accordance with this construction, the function $k_{\mathrm{Gaussian}}(\cdot,x^\prime)$ is referred to as the \textit{kernel basis function} centred at $x^\prime$.

\begin{figure}
    \centering
    \begin{minipage}[c]{0.38\textwidth}
    \includegraphics{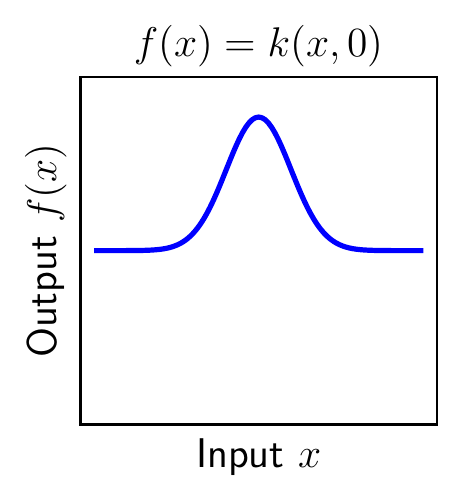}
    \end{minipage}
    \begin{minipage}[c]{0.22\textwidth}
    \begin{center}
        \Huge $\xrightarrow{\text{\normalsize superpose}}\;$
    \end{center}
    \end{minipage}
    \begin{minipage}[c]{0.38\textwidth}
    \includegraphics{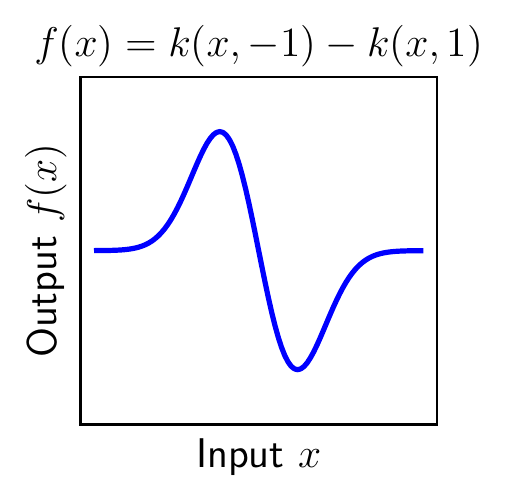}
    \end{minipage}
    \caption[Constructing a function by superposing kernel basis functions]{Constructing a function by superposing kernel basis functions. The left panel displays a kernel basis function for the Gaussian kernel (Example \ref{def:kgauss}). The right panel shows how two kernel basis functions can be superposed to build a more complicated function. A \textit{reproducing kernel Hilbert space} (Definition \ref{def:rkhs}) consists of superpositions of arbitrarily many kernel basis functions.}
    \label{fig:kernel}
\end{figure}

By superposing kernel basis functions in this way, any kernel $k(\cdot,\cdot)$ can be used to construct a space of functions. The functions are parameterised by the centres and strengths of each kernel basis function within the superposition. Such a space of functions is known as a \textit{reproducing kernel Hilbert space} (RKHS). The remainder of this section will introduce the concept of an RKHS more formally, including how to measure the complexity of a function within an RKHS, and how to select a function from an RKHS in light of data.

\subsection{Using a kernel to construct a space of functions} 

For a function $k(\cdot,\cdot)$ to qualify as a kernel, it must satisfy two conditions:

\begin{definition}[Kernel]\label{def:kernel} A function $k:\mathcal{X}\times\mathcal{X}\to\R$ is a \textit{kernel} provided that:
\begin{enumerate}[label=\roman*)]
    \item $k$ is \textit{symmetric}: for any pair of inputs $x,x^\prime \in \mathcal{X}$, $k(x,x^\prime) = k(x^\prime, x)$;
    \item $k$ is \textit{positive definite}: for any set of $m$ distinct inputs $X=\{x_1,...,x_m\}$, the corresponding \textit{Gram matrix} $K_{XX}^{ij}\coloneqq k(x_i,x_j)$ is positive definite.
\end{enumerate}
\end{definition}
These conditions imply that important computations involving the kernel are well-defined. For instance, the inverse $K_{XX}^{-1}$ exists. Given Definition \ref{def:kernel}, it is simple to construct a space of functions by superposing kernel basis functions:

\begin{definition}[Pre-RKHS]\label{def:pre-rkhs} Given a kernel $k: \mathcal{X} \times \mathcal{X} \to \R$, the \textit{precursor} to a \textit{reproducing kernel Hilbert space} (pre-RKHS) consists of all linear combinations of finitely many kernel basis functions---that is, all functions of the form:
\begin{equation}
    f(\cdot)=\sum_{i=1}^m \alpha_i \, k(\cdot,x_i),
\end{equation}
for any weights $\alpha_1,...,\alpha_m\in\R$, centres $x_1, ..., x_m \in \mathcal{X}$, and positive integer $m$.
\end{definition}

One of the attractive features of a pre-RKHS is that one may measure the similarity of two functions within it by taking an inner product:

\begin{definition}[Pre-RKHS inner product] Given a kernel $k$ and two functions $f(\cdot) = \sum_{i=1}^m \alpha_i \,k(\cdot,x_i)$ and $g(\cdot) = \sum_{i=1}^{m^\prime} \beta_i\,k(\cdot,x_i^\prime)$, the \textit{inner product} of $f$ and $g$ in the pre-RKHS induced by $k$ is given by:
\begin{equation*}
    \langle f, g \rangle_\mathrm{RKHS} \coloneqq \sum_{i=1}^m \sum_{j=1}^{m^\prime} \alpha_i k(x_i,x_j^\prime) \beta_j \eqqcolon \alpha^\top K_{XX^\prime} \beta.
\end{equation*}
\end{definition}
It can be checked that, by the symmetry and positive definiteness of $k$, this definition serves as a valid inner product.

A standard treatment of kernel methods would now proceed to \textit{complete} the pre-RKHS to obtain a true \textit{Hilbert space} of functions. Completion essentially involves augmenting the pre-RKHS with certain limits of sequences of functions. This process adds significant technical overhead, while the pre-RKHS is already sufficient for the techniques studied in this thesis. As such, the term RKHS will be used herein merely as convenient shorthand for a pre-RKHS.

\begin{definition}[RKHS]\label{def:rkhs} By an abuse of terminology, an \textit{RKHS} is a pre-RKHS.
\end{definition}
\begin{definition}[RKHS inner product] By an abuse of terminology, an \textit{RKHS inner product} is a pre-RKHS inner product.
\end{definition}

Notice that the inner product of a function $f(\cdot) = \sum_{i=1}^m \alpha_i \,k(\cdot,x_i)$ with the kernel basis function $k(\cdot,x)$ \textit{reproduces} the function evaluated at $x$:
\begin{equation}\label{eq:reproducing}
    \langle f, k(\cdot,x) \rangle_{\mathrm{RKHS}} = \sum_{i=1}^m \alpha_i \,k(x,x_i) = f(x).
\end{equation}
Equation \ref{eq:reproducing} is known as the \textit{reproducing property} of the RKHS.

\subsection{Measuring complexity via RKHS norm}

The RKHS inner product also leads to a natural tool for measuring the complexity of functions within the RKHS: the RKHS norm.
\begin{definition}[RKHS norm]\label{def:rkhs-norm} Given a kernel $k$, the \textit{RKHS norm} of a function $f(\cdot) = \sum_{i=1}^m \alpha_i\,k(\cdot,x_i)$ in the RKHS induced by $k$ is given by:
\begin{equation*}
    \norm{f}_\mathrm{RKHS}^2 \coloneqq \langle f, f \rangle_\mathrm{RKHS} = \sum_{i=1}^m \sum_{j=1}^m \alpha_i k(x_i,x_j) \alpha_j \eqqcolon \alpha^\top K_{XX} \alpha.
\end{equation*}
\end{definition}

To demonstrate that the RKHS norm constrains the complexity of a function within an RKHS, the following lemma shows that the RKHS norm limits how fast a function can vary as its input is varied:

\begin{lemma} Consider a function $f:\mathcal{X}\to\R$ in an RKHS induced by kernel $k$. For any two inputs $x,x^\prime\in\mathcal{X}$, the variation in $f$ satisfies:
\begin{equation}
    \left|f(x) - f(x^\prime)\right| \leq \|f\|_\mathrm{RKHS} \cdot \mathrm{d}(x,x^\prime),
\end{equation}
where the distance function $\mathrm{d}(x,x^\prime)\coloneqq\sqrt{k(x,x) + k(x^\prime,x^\prime) - 2\cdot k(x,x^\prime)}$.
\end{lemma}
\begin{proof} By the reproducing property and the Cauchy-Schwarz inequality:
\begin{align*}
    \left|f(x) - f(x^\prime)\right| &= \left|\langle f, k(\cdot,x)-k(\cdot,x^\prime) \rangle_{\mathrm{RKHS}}\right|\\
    &\leq \|f\|_\mathrm{RKHS} \cdot \|k(\cdot,x)-k(\cdot,x^\prime)\|_\mathrm{RKHS}.
\end{align*}
The proof is completed by observing that, by the definition of RKHS norm, it holds that $\|k(\cdot,x)-k(\cdot,x^\prime)\|_\mathrm{RKHS}=\sqrt{k(x,x) + k(x^\prime,x^\prime) - 2\cdot k(x,x^\prime)}.$
\end{proof}
So a function's RKHS norm serves as a kind of \textit{Lipschitz constant} for the function's continuity. The distance between inputs in this notion of continuity is measured according to a special distance function $\mathrm{d}(\cdot,\cdot)$ related to the degree of kernel similarity between inputs. To gain further intuition about this lemma, it may help to consider its specialisation to the Gaussian kernel:
\begin{corollary}Consider a function $f:\mathcal{X}\to\R$ in the RKHS induced by the Gaussian kernel (Definition \ref{def:kgauss}) with length scale $\sigma=1$. For any two inputs $x,x^\prime\in\mathcal{X}$, the variation in $f$ satisfies:
\begin{equation}
    \left|f(x) - f(x^\prime)\right| \leq \sqrt{2} \cdot \|f\|_\mathrm{RKHS} \cdot \sqrt{1 - \econst^{-\half\|x-x^\prime\|_2^2}}.
\end{equation}
\end{corollary}

\subsection{Fitting data subject to minimum RKHS norm}

So far, this section has defined a space of functions called an RKHS, and a way to measure the complexity of functions in that space called the RKHS norm. It now makes sense to think of finding the least complex function in the RKHS that fits a set of data. This object admits a simple description, as follows:

\begin{theorem}[Minimum RKHS norm kernel interpolation]\label{thm:min-rkhs} Consider a set of $m$ distinct training inputs $X=\{x_1,...,x_m\}$ with respective training labels arranged into a vector $Y=[y_1,...,y_m]\in\R^m$. Say that a function $f$ \textit{interpolates} $(X,Y)$ if the projection $f_X$ (Definition \ref{def:project}) satisfies $f_X = Y$. In the RKHS induced by kernel $k$, the interpolator $f_\star$ of $(X,Y)$ with minimum RKHS norm is given by:
\begin{equation}\label{eq:min-rkhs-interpolator}
    f_\star(x) = \sum_{i=1}^m (K_{XX}^{-1} Y)_i\,k(x,x_i) \eqqcolon K_{xX} K_{XX}^{-1} Y,
\end{equation}
for Gram vector $K_{xX}^{i}\coloneqq k(x,x_i)$ and Gram matrix $K_{XX}^{ij}\coloneqq k(x_i,x_j)$.
\end{theorem}
\begin{proof}To see that $f_\star$ interpolates $(X,Y)$, observe that
\begin{equation*}
    [f_\star(x_1), ..., f_\star(x_m)] = K_{XX} K_{XX}^{-1} Y = Y.
\end{equation*}
To see that no interpolator exists with smaller RKHS norm, consider that another interpolator $g(\cdot)$ may be decomposed as:
\begin{equation*}
    g(\cdot) = f_\star(\cdot) + (g-f_\star)(\cdot).
\end{equation*}
But since $f_\star$ and $g$ both interpolate the training set, for any training input $x_i$:
\begin{equation*}
    \langle g-f_\star, k(\cdot,x_i) \rangle_\mathrm{RKHS}= g(x_i) - f_\star(x_i) =0,
\end{equation*}
where the first equality follows by the reproducing property. So $g-f_\star$ is orthogonal to kernel basis functions $k(\cdot,x_1),...,k(\cdot,x_m)$. Since $f_\star$ is constructed purely from those basis functions, this implies that $g-f_\star$ and $f_\star$ are themselves orthogonal: $\langle g-f_\star,f_\star\rangle_\mathrm{RKHS} = 0$. In turn:
\begin{equation}\label{eq:orthog-functions}
    \|g\|_\mathrm{RKHS}^2 = \|f_\star\|_\mathrm{RKHS}^2 + \|g-f_\star\|_\mathrm{RKHS}^2 \geq \|f_\star\|_\mathrm{RKHS}^2,
\end{equation}
which establishes the result.
\end{proof}

The implication of this theorem is that given a training set $S=(X,Y)$ of $m$ examples, the minimum RKHS norm interpolator $f_\star$ may be constructed by simple linear algebra operations $f_\star(x)=K_{xX} K_{XX}^{-1} Y$ involving kernel Gram matrix $K_{XX}$ and Gram vector $K_{xX}$. The dominant computational cost is that of inverting an $m \times m$ matrix, which is naïvely $\mathcal{O}(m^3)$. For this reason, the cost of kernel methods is typically cubic in the size of the training set.

Furthermore, Equation \ref{eq:min-rkhs-interpolator} shows that the kernel interpolator of minimum RKHS norm may be represented using only finitely many kernel basis functions. As such, Theorem \ref{thm:min-rkhs} is connected to a classic result known as the \textit{representer theorem} \citep{lwk}. 

\section{Gaussian processes}

The previous section showed that a function space may be constructed by superposing kernel basis functions. But there are other means of constructing a function space from a kernel. In the case of \textit{Gaussian processes}, a kernel is used to place a probability measure over a space of functions. The functions themselves may be sampled from this measure.

A probability measure on function space can be thought of as encoding \textit{prior belief} about which kinds of functions would have a good chance of explaining the data when it arrives. Phrased another way, functions of smaller prior probability might be considered more complex. An advantage of this approach is that, given the training data, a posterior distribution over functions consistent with the data may be constructed simply by turning the handle of Bayes' rule. This posterior provides not just a single prediction for a new data point, but a range of predictions each accompanied by a posterior probability.

\subsection{Constructing a space of functions by random sampling}

Given an input space $\mathcal{X}$, consider drawing a Gaussian random variable at each point $x\in\mathcal{X}$ and recording the value of each random draw as $f(x)\in\R$. To make life most interesting, one may choose not to draw these random variables independently, but rather from a joint Gaussian distribution. The covariance of this joint Gaussian could encode, for instance, that the closer an input $x$ is to another input $x^\prime$, the more likely it is that $f(x)$ would be similar to $f(x^\prime)$. This construction, of jointly Gaussian \textit{function values} $f(\cdot)$ with covariance structure based on similarity in the input space $\mathcal{X}$, is known as a Gaussian process:
\begin{definition}[Gaussian process]\label{def:gp} Consider an input space $\mathcal{X}$ and a kernel $k:\mathcal{X}\times\mathcal{X}\to\R$. If for any finite set of inputs $X=\{x_1, ..., x_m\}\in\mathcal{X}^m$, the distribution of function values $f(x_1),...,f(x_m) \sim \normal(0, K_{XX})$, then the function $f$ is drawn from a \textit{Gaussian process} with covariance $k$: $f\sim\gp(0,k)$.
\end{definition}
This definition is said to establish the \textit{prior measure} that a Gaussian process with a given kernel assigns to a space of functions.

\subsection{Measuring complexity via probability}
In principle, a Gaussian process provides a simple means to compare the complexity of functions. For instance, given two functions $f_1$ and $f_2$, one may simply declare the less likely of $f_1$ and $f_2$ under the Gaussian process prior to be more complex. To make this idea workable, one notices in Definition \ref{def:gp} that it is easier to assess the probability of a Gaussian process function when it is inspected only on a finite collection of inputs. As such, one could collect a collection of $m$ inputs $X\in\mathcal{X}^m$ and compare the density that $\normal(0,K_{XX})$ assigns to $\{f_1(x)\}_{x\in X}$ versus $\{f_2(x)\}_{x\in X}$. This procedure is not entirely satisfactory, since the answer depends on the choice of the set of inputs $X$. Still it conveys the spirit of the idea that prior probability may be used to assess complexity.

That being said, the complexity of individual functions does not usually play a major role in discussion of Gaussian processes. \textit{Distributions over functions} are considered to be more important, and Bayesians advocate for making predictions by integrating over these distributions of functions \citep{radford}. The next section will discuss how a posterior distribution may be constructed from the Gaussian process prior in light of data.

\subsection{Fitting data via Bayesian inference}

An attractive feature of Gaussian processes is that, given data, one may use Bayes' rule to derive a posterior distribution over functions. Given a training sample $S=(X,Y)$, Bayes' rule states that:
\begin{equation}\label{eq:bayes}
\Probe[f\mid S] = \frac{\Probe[S\mid f]\cdot \Probe[f]}{\sum_{f^\prime}\Probe[S\mid f^\prime]\cdot \Probe[f^\prime]}.
\end{equation}
There are three important quantities appearing in this equation:
\begin{enumerate}
    \item The \textit{prior} $\Probe[f]$ denotes the prior probability of function $f$.
    \item The \textit{likelihood} $\Probe[S\mid f]$ denotes how likely it is that a training sample $S=(X,Y)$ was obtained from a particular function $f$. The choice of likelihood is a modelling decision. The \textit{zero-one likelihood} is common:
    \begin{equation}\label{eq:zero-one-like}
        \Probe_{0/1}[S\mid f] \coloneqq \mathbb{I}[f(X) = Y].
    \end{equation}
    \item The \textit{posterior} $\Probe[f\mid S]$ denotes the probability of $f$ in light of the training sample under the choice of likelihood.
\end{enumerate}

\begin{figure}
    \centering
    \includegraphics{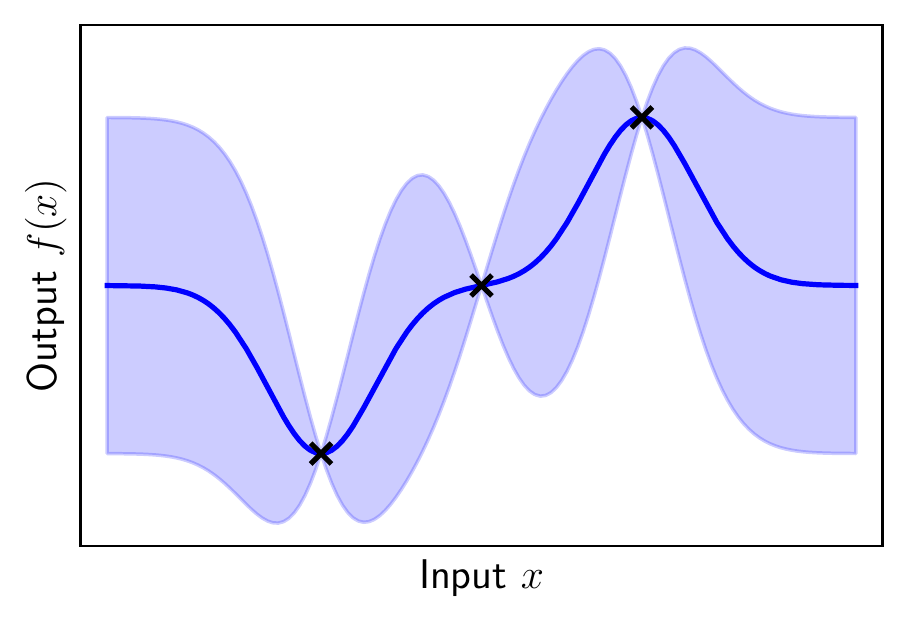}
    \caption[Posterior distribution of a Gaussian process]{Posterior distribution of a Gaussian process, conditioned on fitting data marked by the black crosses. The kernel was set to the Gaussian kernel (Definition \ref{def:kgauss}) and the likelihood was set to zero-one (Equation \ref{eq:zero-one-like}). The solid blue line depicts the posterior mean, and the shaded region represents $\pm 1$ standard deviations about this mean.}
    \label{fig:gp}
\end{figure}

Bayesian inference is generally intractable due to the high-dimensional summation or integration that appears in the denominator of Equation \ref{eq:bayes}. But for Gaussian processes the required integrals are Gaussian in nature and often admit closed-form solutions. The following theorem is an important example:
\begin{theorem}\label{thm:gp-cond} Suppose that $f\sim\gp(0,k)$. Conditioned on $f$ interpolating dataset $(X,Y)$, the distribution of $f$ projected on to a fresh set of inputs $X^\prime$ is:
\begin{equation}\label{eq:gp-cond}
    f_{X^\prime} \sim \normal(K_{X^\prime X}K_{XX}^{-1}Y, K_{X^\prime X^\prime}-K_{X^\prime X}K_{XX}^{-1}K_{X X^\prime}).
\end{equation}
\end{theorem}
\begin{proof}
    Because $f\sim\gp(0,k)$, then $(f_X, f_{X^\prime})$ is jointly Gaussian with covariance $K_{X\cup X^\prime\,X\cup X^\prime}$ under this prior. The conditional distribution of this multivariate Gaussian, given that $f_X=Y$, is given by Equation \ref{eq:gp-cond} \citep{bishop}.
\end{proof}

Theorem \ref{thm:gp-cond} allows one to sample functions from the Gaussian process conditioned on interpolating a set of $m$ training examples. This process is illustrated in Figure \ref{fig:gp}. Just as was the case for kernel methods, the dominant computational cost is that of inverting the $m\times m$ kernel Gram matrix $K_{XX}$ which is naïvely $\mathcal{O}(m^3)$. This means that Gaussian processes, like kernel methods, are said to have a cost that is cubic in the size of the training set.

\section{Neural networks}

The function spaces corresponding to kernel methods and Gaussian processes are both derived starting from a kernel function. This renders certain global properties of their function spaces amenable to analysis---for instance, one can write down the kernel interpolator that \textit{globally minimises} the RKHS norm. Or one may write down a Gaussian process posterior distribution that includes within its support \textit{all functions} that interpolate a particular training set. Neural networks eschew these properties, in favour of something else.

A neural network function space is constructed by composing many operators, where each operator has a set of weights. Adjusting the weights in each operator adjusts the function realised by the entire network. The choice of operators and how they are connected to each other is referred to as the \textit{network architecture}. 

Given a particular neural network architecture, it has so far been hard to characterise global properties of the space of functions that it realises, unlike kernel methods and Gaussian processes. Though some things can be said in certain limiting regimes, as will be discussed in Chapter \ref{chap:correspondences}.

\begin{figure}
    \centering
    \input{figures/mlp}
    \caption[A multilayer perceptron]{A multilayer perceptron of depth $L=3$.}
    \label{fig:mlp}
\end{figure}
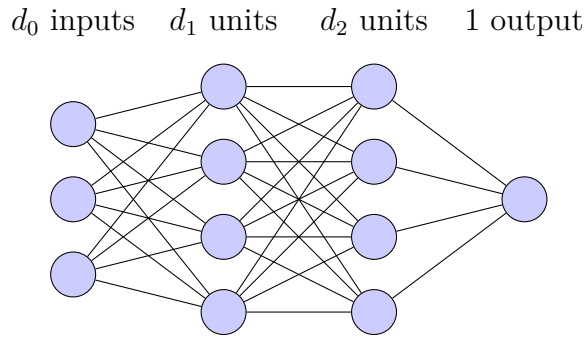

\subsection{Constructing a space of functions by composing parameterised operators}

The simplest example of a neural network architecture is the \textit{multilayer perceptron}---depicted in Figure \ref{fig:mlp}. The multilayer perceptron comprises a composition of matrices interspersed by elementwise nonlinearities, meaning that it encapsulates many of the key features of more general neural networks. As such, it will serve as a \textit{model organism} for detailed study in this thesis.

\begin{definition}[Multilayer perceptron]\label{def:mlp}
A \textit{multilayer perceptron} $f$ of depth $L$ maps an input $x\in\R^{d_0}$ to an output $f(x;w) \in \R^{d_L}$ via the map:
\begin{equation}\label{eq:mlp}
f(x; w) \coloneqq W_L \circ (\phi \circ W_{L - 1})\circ \dots \circ (\phi \circ W_1)\circ x.
\end{equation}
In this expression $\phi$ denotes an elementwise nonlinearity, $W_l$ denotes a matrix of dimension $d_l\times d_{l-1}$, and $w$ denotes the tuple of $L$ matrices $(W_1,...,W_L)$.
\end{definition}

Equation \ref{eq:mlp} provides a simple and direct means of constructing a space of functions. Without a nonlinearity, or with the nonlinearity set to the identity $\phi \gets \Id$, the overall function would be linear. The canonical choice of nonlinearity is known as the \textit{relu nonlinearity}:

\begin{definition}[Relu nonlinearity] The \textit{relu nonlinearity} is given by:
\begin{equation}
    \mathrm{relu}(\cdot) \coloneqq \max(0,\cdot).
\end{equation}
\end{definition}

So the relu nonlinearity retains the positive part of its input. It derives its name from the \textit{rectified linear unit}. The relu nonlinearity is both simple and works well in applications \citep{relu}.

\subsection{Measuring complexity via normalised margin}

Unlike kernel methods and Gaussian processes, the tools for studying the complexity of neural network functions are not yet mature. Developing such tools and their understanding is one of the aims of this thesis. That said, there have been some fairly natural proposals. In the case of a binary classification problem, the \textit{spectrally-normalised margin} is one such example:

\begin{definition}[Spectrally-normalised margin]\label{def:spec-norm-margin} Given a set of training data $S \in \{\R^{d_0}\times\pm 1\}^m$ and a multilayer perceptron $f:\R^{d_0}\times \mathcal{W}\to\R$ with matrices $w=(W_1,...,W_L)\in\mathcal{W}$, the \textit{spectrally-normalised margin} $\rho_\star$ is given by:
\begin{equation}\label{eq:spec-norm-magin}
    \rho_* \coloneqq \min_{(x,y)\in S} \frac{f(x;w)\cdot y}{\norm{x}_2\cdot\prod_{l=1}^L \norm{W_l}_*}.
\end{equation}
\end{definition}

The idea behind this definition is that, assuming that all training points are correctly classified, then the quantity $\min_{(x,y)\in S} f(x;w)\cdot y$ measures how close the closest training point is to being misclassified. But the problem with this measure is that a multilayer perceptron possesses various \textit{rescaling symmetries}---in particular, scaling up the input or a weight matrix at any layer scales up the margin too. Normalising by the product of norms that appear in the denominator of Equation \ref{eq:spec-norm-magin} yields a notion of margin that is invariant to these trivial rescaling symmetries. This definition of spectrally-normalised margin is related to one given by \citet{bartlett}.

Of course, there are other ways to measure margin modulo rescaling symmetries. For instance, the spectral norms appearing in Definition \ref{def:spec-norm-margin} measure the \textit{largest} singular value of each weight matrix in the multilayer perceptron. One may just as well measure the \textit{average} singular value. This motivates the following:

\begin{definition}[Frobenius-normalised margin]\label{def:frob-norm-margin} Given a set of training data $S \in \{\R^{d_0}\times\pm 1\}^m$ and a multilayer perceptron $f:\R^{d_0}\times \mathcal{W}\to\R$ with matrices $w=(W_1,...,W_L)\in\mathcal{W}$, let $\overline{d}_l\coloneqq \min(d_l,d_{l-1})$ be the minimum dimension of matrix $W_l\in\R^{d_l \times d_{l-1}}$. Then the \textit{Frobenius-normalised margin} $\rho_F$ is given by:
\begin{equation}
    \rho_F \coloneqq \min_{(x,y)\in S} \frac{f(x;w)\cdot y}{\norm{x}_2\cdot\prod_{l=1}^L \norm{W_l}_F/\sqrt{\overline{d}_l}}.
\end{equation}
\end{definition}
To understand this definition, note that the squared Frobenius norm $\norm{W_l}_F^2$ of a matrix is equal to the sum of its squared singular values. Also, a matrix $W_l$ has a number $\overline{d_l}$ of singular values in total. Therefore the Frobenius-normalised margin is just the spectrally-normalised margin with the largest singular value $\norm{W_l}_*$ replaced by the root-mean-square singular value $\norm{W_l}_F/\sqrt{\overline{d}_l}$. This definition is related to one given by \citet{my-margin}.

But what do these notions of normalised margin have to do with the complexity of a neural network function? Consider a neural network that perfectly classifies a particular training set $S$, but with very small normalised margin. This network is, in a sense, \textit{close} to a second network that misclassifies $S$. It would only take a small perturbation of the function realised by the former network to make it match the latter. Therefore, it seems reasonable that these two networks should have similar generalisation behaviour despite their differences on the training set. On the other hand, a network that classifies $S$ perfectly and with large normalised margin is not close to a network that misclassifies $S$. Based on this line of thinking, it may seem reasonable to declare networks with small normalised margin \textit{complex} on the grounds that they can mimic networks with different training error.

Of course, this is not a rigorous argument, and Definitions \ref{def:spec-norm-margin} and \ref{def:frob-norm-margin} are suggested only as candidate measures of complexity. Various modifications to these measures may work better. For example, one may consider replacing the minimum over the training set $\min_{(x,y)\in S}$ with the expectation $\Expect_{(x,y)\sim\uniform(S)}$. The role of these complexity measures in generalisation is studied further in Chapter \ref{chap:bpm}.

\subsection{Fitting data by gradient descent}

To use a neural network in a machine learning application, one needs a way of selecting a neural network that fits a particular training set. In the classification setting, a natural goal is to seek a neural network that perfectly classifies the training set. This could be measured by the \textit{zero-one loss}, say.

\begin{definition}[Zero-one loss]\label{def:zero-one-loss} For a neural network $f:\mathcal{X}\times\mathcal{W}\to\R$ and a training set $S\in(\mathcal{X}\times \pm 1)^m$, the \textit{zero-one loss} of weight vector $w\in\mathcal{W}$ is:
\begin{equation}
    \mathcal{L}_{0/1}(w) \coloneqq \frac{1}{m}\sum_{(x,y)\in S} \mathbb{I}[\sign f(x;w) \neq y].
\end{equation}
\end{definition}

Unfortunately, directly minimising the zero-one loss is not feasible, since its gradient with respect to the weights is either zero or undefined. Instead, a continuous proxy is used, such as the \textit{square loss}:

\begin{definition}[Square loss]\label{def:sq-loss} For a neural network $f:\mathcal{X}\times\mathcal{W}\to\R$ and a training set $S\in(\mathcal{X}\times \pm 1)^m$, the \textit{square loss} of weight vector $w\in\mathcal{W}$ is:
\begin{equation}
    \mathcal{L}_2(w) \coloneqq \frac{1}{2m}\sum_{(x,y)\in S} \left(f(x;w) - y\right)^2.
\end{equation}
\end{definition}

Observe that a neural network $f(\cdot,w)$ attaining square loss $\el_2(w) = 0$ also attains zero-one loss $\el_{0/1}(w)=0$. But the square loss can be conveniently minimised via gradient descent. In spirit, such a procedure is given by:
\begin{equation}
    w \gets w - \eta \cdot \nabla_w \el_2(w),
\end{equation}
where the constant $\eta>0$ denotes the user-prescribed learning rate. For computational efficiency, the gradient of the loss with respect to only a sub-sample of data is typically used, rather than with respect to the full dataset. This is referred to as either \textit{stochastic} or \textit{mini-batch} gradient descent. The theory of \textit{full-batch} gradient descent will be carefully studied in Part \ref{part:opt}. 

The square loss also admits a compact description via function projection:

\begin{proposition}[Square loss of projected function]\label{ex:sq-loss-projected} Consider a training set of $m$ examples $S=(X,Y)$. Let $f_X(w)$ denote the projection (Definition \ref{def:project}) of neural network $f(\cdot;w)$ on to $X$. The square loss may be written:
\begin{equation}\label{eq:sq-loss-projected}
    \el_2(w) = \frac{1}{2m}\cdot\norm{f_X(w)-Y}_2^2.
\end{equation}
\end{proposition}

Beyond square loss, other proxies for the zero-one loss are in use. For example:

\begin{definition}[Logistic loss]\label{def:log-loss} For a neural network $f:\mathcal{X}\times\mathcal{W}\to\R$ and a training set $S\in(\mathcal{X}\times \pm 1)^m$, the \textit{logistic loss} of weight vector $w\in\mathcal{W}$ is:
\begin{equation}
    \mathcal{L}_{\log}(w) \coloneqq \frac{1}{m}\sum_{(x,y)\in S} \log\left(1+\econst^{-y\cdot f(x;w)}\right).
\end{equation}
\end{definition}

The logistic loss is \textit{margin maximising} \citep{rosset2003margin}. This means that in contrast to the square loss, which is minimised by setting the training outputs to fixed values, the logistic loss is reduced by making the outputs on correctly classified training points larger in magnitude. 

A margin-maximising loss function would tend to increase the notions of normalised margin given in Definitions \ref{def:spec-norm-margin} and \ref{def:frob-norm-margin}---at least provided the norms of the weights in the network do not themselves grow. To prevent weight norms growing, an \textit{L2 penalty} is often added to the loss function:
\begin{definition}[L2 penalty]\label{def:l2-penalty} Given a neural network $f:\mathcal{X}\times\mathcal{W}\to\mathcal{Y}$ with $L$ layers and weight tuple $w=(W_1, ..., W_L)\in\mathcal{W}$, the \textit{L2 penalty} is given by:
\begin{equation}
    \norm{w}_2^2 \coloneqq \sum_{l=1}^L \norm{W_l}_F^2.
\end{equation}
\end{definition}
In words, the L2 penalty penalises the size of the Frobenius norm of the weight matrix at each layer. To train a network, one might then minimise the following \textit{L2-regularised} loss function:
\begin{equation}
    \el_{\mathrm{log}}(w) + \lambda \cdot \|w\|_2^2.
\end{equation}
This is often referred to as adding \textit{weight decay}, since in the gradient descent update the size of the weights are damped by a factor $(1-\eta \lambda)$ at each iteration:
\begin{equation}
    w \gets w \cdot (1 - \eta\lambda) - \eta \cdot \nabla_w \el_\mathrm{log}(w).
\end{equation}

The final comment of this chapter is that it has so far been difficult to theoretically characterise the computational cost of neural network training, although efforts have been made to obtain empirical scaling laws \citep{scaling-laws}. What can be said is that neural network training appears to overcome the unfavourable cubic cost of kernel methods and Gaussian processes.

\printbibliography[heading=subbibliography]
\end{refsection}

%% file: figures/mlp.tex
\def\layersep{2cm}
\begin{tikzpicture}[-,draw=black, node distance=\layersep]
    \tikzstyle{neuron}=[circle,draw=black,minimum size=17pt,inner sep=0pt,fill={rgb,255:red,204;green,204;blue,255}]
    \tikzstyle{annot} = [text width=4em, text centered, text height=3ex, text depth=3ex]

    \foreach \name / \y in {1,...,3}
        \node[neuron] (I-\name) at (0,-\y) {};

    \foreach \name / \y in {1,...,4}
        \path[yshift=0.5cm]
            node[neuron] (H-\name) at (\layersep,-\y cm) {};
            
    \foreach \name / \y in {1,...,4}
        \path[yshift=0.5cm]
            node[neuron] (H1-\name) at (\layersep*2,-\y cm) {};

    \node[neuron] (O) at (\layersep*3,-2 cm) {};

    \foreach \source in {1,...,3}
        \foreach \dest in {1,...,4}
            \path (I-\source) edge (H-\dest);
            
    \foreach \source in {1,...,4}
        \foreach \dest in {1,...,4}
            \path (H-\source) edge (H1-\dest);

    \foreach \source in {1,...,4}
        \path (H1-\source) edge (O);

    \node[annot,above of=H-1, node distance=0.75cm] (hl) {$d_1$ units};
    \node[annot,left of=hl] {$d_0$ inputs};
    \node[annot,right of=hl] (h2) {$d_2$ units};
    \node[annot,right of=h2] {$1$ output};
\end{tikzpicture}

%% file: chapters/3-correspondences.tex
\begin{refsection}

\chapter{Correspondences Between Function Spaces}
\label{chap:correspondences}

\begin{tcolorbox}
This chapter introduces various useful correspondences between kernel methods, Gaussian processes and neural networks. The material is expositional and is included for the reader's aid.
\end{tcolorbox}

Kernel methods, Gaussian processes and neural networks were introduced in Chapter \ref{chap:functionspaces} as different philosophical approaches to constructing spaces of functions that may be used to fit data. Kernel methods construct functions by linearly combining kernel basis functions, and are thus amenable to a particular kind of functional analysis. Gaussian processes draw functions from a probability measure, allowing the laws of probability to be applied to both analyse complexity and fit data. Neural networks eschew these analytical considerations, preferring a very flexible means of building functions by composing parameterised operators that may be fit to data via gradient descent.

It may seem surprising, then, that multiple connections exist between these seemingly disparate spaces of functions. These \textit{correspondences} provide a promising route toward better understanding each function space individually.

\section{GP posterior mean is a kernel interpolator}
\label{sec:gp-k-mean}

One of the attractive features of a Gaussian process is that, after observing data, a full posterior distribution over possible explanations of the data may be obtained. But sometimes this full posterior distribution is overkill---for example, in some situations all that is needed is the single \textit{best prediction} for a fresh test input. In these situations, it may be enough to simply return the mean of the posterior distribution (solid blue line in Figure \ref{fig:gp}). 

Perhaps surprisingly, the mean of a Gaussian process posterior distribution is equivalent to the kernel interpolator of minimum RKHS norm, where the kernel is set to to the covariance function of the Gaussian process. Formally, suppose that $f\sim\gp(0,k)$. Conditioned on $f$ interpolating a training set $S=(X,Y)$, by Theorem \ref{thm:gp-cond} the mean prediction on a fresh input $x$ is given by:
\begin{equation}\label{eq:gp-cond-mean}
    \Expect\left[f(x)\mid f_X = Y\right] = K_{xX}K_{XX}^{-1}Y,
\end{equation}
But by Theorem \ref{thm:min-rkhs}, this is equivalent to the minimum RKHS norm kernel interpolator of $(X,Y)$ with kernel $k$.

\section{GP posterior variance bounds the error of kernel interpolation}

A slightly more subtle relationship between Gaussian processes and kernel methods connects the posterior variance of a Gaussian process to the worst case error of kernel interpolation. This correspondence is based on the fact that the Gaussian process posterior variance admits the following geometric interpretation in the corresponding RKHS:
\begin{lemma}[Distance between a point and a subspace]\label{lem:var-geom} Consider an RKHS induced by kernel $k:\mathcal{X}\times\mathcal{X}\to\R$. Given a set of inputs $X=\{x_1,...,x_m\}$ and fresh input $x$, the shortest distance between the kernel basis function centred on $x$ and the span of the kernel basis functions centred on points in $X$ is given by:
\begin{align}\label{eq:k-gp-var}
    \mathrm{dist}^2\big(k(\cdot,x), \mathrm{span}\{k(\cdot,x_i)\}_{i=1}^m\big)&\coloneqq \min_{\alpha\in\R^m}\left\|k(\cdot,x)-\sum_{i=1}^m\alpha_i\, k(\cdot,x_i)\right\|_{\mathrm{RKHS}}^2 \nonumber\\ &= K_{xx} - K_{xX}K_{XX}^{-1}K_{Xx}.
\end{align}
\end{lemma}
In words: the Gaussian process posterior variance (right-hand side of Equation \ref{eq:k-gp-var}) measures the shortest RKHS-distance from the kernel basis function centred at $x$ and the span of the kernel basis functions centred on the training points.
\begin{proof}[Proof of Lemma \ref{lem:var-geom}] Evaluating the RKHS norm for arbitrary $\alpha \in \R^m$ yields:
    \begin{equation*}
        \left\|k(\cdot,x)-\sum_{i=1}^m\alpha_i\, k(\cdot,x_i)\right\|_{\mathrlap{\mathrm{RKHS}}}^{\mathrlap{2}} = k_{xx} - 2\cdot\alpha^\top K_{Xx} + \alpha^\top K_{XX}\alpha.
    \end{equation*}
    Setting the derivative with respect to $\alpha$ to zero yields that $K_{Xx}=K_{XX}\alpha$ and hence the RKHS norm is minimised by setting $\alpha = K_{XX}^{-1}K_{Xx}$. Substituting this result back into the expression for the RKHS norm yields the result.
\end{proof}

With Lemma \ref{lem:var-geom} in hand, the next theorem follows readily. The result is related to Corollary 3.11 of \citet{Kanagawa2018GaussianPA}.

\begin{theorem}[Error of kernel interpolation] Given a function $g$ in an RKHS induced by kernel $k$ and any set of distinct inputs $X$, the deviation between $g$ and the minimum RKHS norm interpolator $f_\star$ of $(X,g(X))$ satisfies:
\begin{equation}
    \abs{g(x) - f_{\star}(x)} \leq \sqrt{\norm{g}_\mathrm{RKHS}^2-\norm{f_\star}_\mathrm{RKHS}^2} \cdot \sqrt{K_{xx} - K_{xX}K_{XX}^{-1}K_{Xx}}.
\end{equation}
\end{theorem}
\begin{proof} By the reproducing property, the fact that $g$ and $f_\star$ agree on the training points, and the Cauchy-Schwarz inequality:
    \begin{align*}
        \abs{g(x) - f_{\star}(x)} &= \left|\langle g - f_\star, k(\cdot,x)\rangle_\mathrm{RKHS}\right| \\
        &= \left|\langle g - f_\star, k(\cdot,x) - \sum_{i=1}^m \alpha_i\,k(\cdot,x_i)\rangle_\mathrm{RKHS}\right| \text{ for all } \alpha\in\R^m\\
        &\leq \left\|g-f_\star\right\|_\mathrm{RKHS} \cdot \min_{\alpha\in\R^m} \left\|k(\cdot,x) - \sum_{i=1}^m \alpha_i\,k(\cdot,x_i)\right\|_{\mathrlap{\mathrm{RKHS}}}.
    \end{align*}
    The result follows by observing that $\left\|g\right\|_\mathrm{RKHS}^2 = \left\|f_\star\right\|_\mathrm{RKHS}^2 +  \left\|g-f_\star\right\|_\mathrm{RKHS}^2$ by Equation \ref{eq:orthog-functions}, and an application of Lemma \ref{lem:var-geom}.
\end{proof}

In words: if one is approximating an unknown function $g$ by the the minimum RKHS norm interpolator $f_\star$ of a set of samples from $g$, then the error of this approximation is bounded by the posterior standard deviation of the corresponding Gaussian process fit to those samples, scaled by a constant depending on the difference in RKHS norm between $g$ and approximation $f_\star$.

\section{Making the GP posterior concentrate on a kernel interpolator}

A particular kind of aggregate Gaussian process prediction---the posterior mean---is directly available via Equation \ref{eq:gp-cond-mean}. As mentioned in Section \ref{sec:gp-k-mean}, the posterior mean may suffice for practical applications and often one may not bother computing the posterior variance.

But for other function spaces, it may be the case that the posterior distribution may only be accessed through samples---and these samples may be expensive to obtain. In fact, Chapter \ref{chap:bpm} will argue that a neural network trained to fit data by gradient descent is analogous to a \textit{single draw} from a Gaussian process posterior distribution. In this situation, computing the posterior mean would involve averaging over many samples, which might be prohibitively expensive.

The following theorem observes that, for Gaussian processes, this issue may be circumvented---by forcing the entire posterior distribution to concentrate on its mean. In this case, \textit{all} posterior samples may be trusted to faithfully report the mean. This mean is itself a kernel interpolator by the results of Section \ref{sec:gp-k-mean}.

\begin{theorem}[Posterior concentration]\label{thm:force-gp} Given a kernel $k$ and a training sample $(X,Y)$, define two constants: a ``margin'' $\gamma>0$ and a ``normalisation'' $\tau^2>0$. Sample a function $f\sim\gp(0,\tau^2\cdot k)$ conditioned on $f$ interpolating $(X,\gamma\cdot Y)$:
\begin{equation}
    f\sim\gp(0, \tau^2 \cdot k \mid f_X=\gamma\cdot Y ).
\end{equation}
In the limit that the ``normalised margin'' $\gamma / \tau \to \infty$, then the rescaled function $f/\gamma$ converges to the posterior mean:
\begin{equation}
    f(x)/\gamma =K_{xX}K_{XX}^{-1}Y, \; \text{with probability one}.
\end{equation}
\end{theorem}
\begin{proof} By Theorem \ref{thm:gp-cond}, the sampled function $f$ evaluated at $x$ follows:
\begin{equation*}
    f(x) \sim \normal\left(K_{xX}K_{XX}^{-1} (\gamma \cdot Y),\tau^2\cdot (K_{xx}-K_{xX}K_{XX}^{-1}K_{Xx})\right).
\end{equation*}
Dividing through by $\gamma$ then yields:
\begin{equation*}
    f(x)/\gamma \sim \normal\left(K_{xX}K_{XX}^{-1} Y,\frac{\tau^2}{\gamma^2}\cdot (K_{xx}-K_{xX}K_{XX}^{-1}K_{Xx})\right).
\end{equation*}
Taking the normalised margin $\gamma/\tau \to \infty$ causes the variance to vanish.
\end{proof}

The language used in this theorem statement surrounding ``normalised margin'' is intended to evoke the concepts of normalised margin given in Definitions \ref{def:spec-norm-margin} and \ref{def:frob-norm-margin}. This foreshadows a contribution of Chapter \ref{chap:bpm}, where a formal connection is made between Theorem \ref{thm:force-gp} and a notion of normalised margin in neural networks. This formal connection leverages a correspondence between neural networks and Gaussian processes, which is presented in the next section.

\section{Neural network--Gaussian process correspondence}
\label{sec:nngp}

Informally, the \textit{neural network--Gaussian process} (NNGP) correspondence states that the space of functions realised by a sufficiently wide neural network is equivalent to a Gaussian process. Since a key feature of a Gaussian process is that it assigns probabilities to functions, this informal statement is not yet fully meaningful. What is needed is a means of assigning probabilities to the functions realised by a neural network. The key idea is to consider randomly sampling the network weights to obtain a distribution over neural network functions. This leads to the more accurate---but still informal---statement:
\begin{quote}
    The distribution over functions realised by a wide enough neural network with randomly drawn weights is a Gaussian process.
\end{quote}

While this correspondence may seem surprising, the mechanism by which it works ends up being fairly simple. Consider the $l$th layer of a neural network:
\begin{equation}\label{eq:nn-layer}
    f_l(x) = W_l\cdot \phi(f_{l-1}(x)),
\end{equation}
where $x\in\R^{d_0}$ is the network input, $f_l(x)\in\R^{d_l}$ is the layer output, $W_l\in\R^{d_l\times d_{l-1}}$ is the weight matrix, $\phi$ is the nonlinearity and $f_{l-1}(x)\in\R^{d_{l-1}}$ is the layer input. Suppose that the different components of $f_{l-1}(x)$ are iid random variables and that the entries of weight matrix $W_l$ are also drawn iid. Then the components of the output $f_l(x)$, being the sum over a large number of iid contributions, are themselves iid Gaussian by the central limit theorem.

This idea is formalised in the following lemma, essentially due to \citet{radford}.

\begin{lemma}[NNGP correspondence]\label{lem:nngp}

For the neural network layer given by Equation \ref{eq:nn-layer}, if the following conditions hold:
\begin{enumerate}[label=(\roman*)]
    \item Inputs: For every network input $x \in \R^{d_0}$, the components of the layer input $f_{l-1}(x)\in\R^{d_{l-1}}$ are iid with finite first and second moment.
    \item Weights: Entries of $W_l$ are iid with zero mean and variance $\sigma^2/d_{l-1}<\infty$.
    \item Nonlinearity: For any random variable $z$ with finite first and second moment, $\phi(z)$ also has finite first and second moment.
\end{enumerate}
Then, in the limit of width $d_{l-1}\rightarrow\infty$, the distribution of layer outputs satisfies:
\begin{enumerate}
    \item IID outputs. For every input $x \in \R^{d_0}$, the components of the layer output $f_l(x)$ for that input are iid with finite first and second moment.
    \item Gaussian outputs. For any collection of $m$ network inputs $x_1, ..., x_m$, the output components $f_l^i(x_1),...,f_l^i(x_m)$ are jointly Gaussian for $i=1,...,d_l$.
\end{enumerate}
\end{lemma}

Condition (ii) is satisfied simply by sampling the network weights iid Gaussian, say. Similarly, condition (iii) is easy to check for a given nonlinearity. Condition (i) on the layer inputs seems the most non-trivial. But notice that the lemma propagates this condition to the next layer via entailment 1). This means that provided condition (i) holds at the first layer, then recursive application of the lemma implies that condition (i) will hold at all layers.

Entailment 2) is equivalent to saying that the $i$th component of the layer output forms a Gaussian process (Definition \ref{def:gp}). The kernel of this Gaussian process depends on the specific architecture of the network preceding the $l$th layer. For example, the kernel for multilayer perceptrons with relu nonlinearity is given by Theorem \ref{thm:relu}. But first, a proof of the lemma is given.

\begin{proof}[Proof of Lemma \ref{lem:nngp}] 
To prove entailment 1), write output $f_l(x)\in\R^{d_l}$ as:
\begin{gather}\label{eq:z1}
    f_l(x) = \sum_{j=1}^{d_{l-1}} \left[W_l^{1j}, ..., W_l^{d_{l}j} \right]\cdot\phi \left(f_{l-1}^j(x)\right)\eqqcolon\sum_{j=1}^{d_{l-1}}v_j.
\end{gather}
By conditions (i) and (ii), the summands $v_j$ in Equation \ref{eq:z1} are iid random vectors. To apply a central limit theorem, these random vectors must have finite mean and variance. Since the weights and inputs are independent, the mean is finite and is given by:
\begin{equation*}
    \Expect[v_j] = \Expect [W_l^{1j},...,\Expect W_l^{d_l j}] \cdot \Expect [\phi(f_{l-1}^j(x))] = 0 \cdot \Expect [\phi(f_{l-1}^j(x))] = 0.
\end{equation*}
This result relies on the quantity $\Expect [\phi(f_{l-1}^j(x))]$ being finite by conditions (i) and (iii). Similarly, the variance is given by:
\begin{equation*}
    \Expect[v_j^iv_j^k] = \Expect [W_l^{ij}W_l^{kj}] \cdot \Expect [\phi(f_{l-1}^j(x))^2] = \delta_{ik}\cdot\sigma^2/d_{l-1} \cdot \Expect [\phi(f_{l-1}^j(x))^2],
\end{equation*}
where $\delta_{ik}$ is the Kronecker delta. Since $\sigma^2/d_{l-1}<\infty$ and, by conditions (i) and (iii), the quantity $\Expect [\phi(f_{l-1}^j(x))^2]$ appearing on the right-hand side is finite too, then the variance of random vector $v_j$ is finite. In turn, by the multivariate central limit theorem \citep{vaart_1998}, in the limit that $d_{l-1}\to\infty$, the layer output $f_l(x)\sim\normal(0, \sigma^2 \cdot \Expect [\phi(f_{l-1}^1(x))^2]\cdot\Id)$. In particular, this implies that the components of $f_l(x)$ are iid with finite first and second moment.

Entailment 2) is established by a similar argument that considers the $i$th component of the layer output $f_l$ projected on to $m$ samples $X=\{x_1,...,x_m\}$:
\begin{gather}\label{eq:z2}
    (f_l^{i})_X = \sum_{j=1}^{d_{l-1}} W_l^{ij}\cdot\left[\phi \left(f_{l-1}^j(x_1)\right), ..., \phi \left(f_{l-1}^j(x_m)\right) \right].
\end{gather}
Again, by combining conditions (i), (ii) and (iii), the summands in Equation \ref{eq:z2} are iid random vectors with finite mean and finite covariance. Then as $d_{l-1}\rightarrow\infty$, the distribution of $(f_l^{i})_X$ is multivariate Gaussian---again by the multivariate central limit theorem. This completes the proof.
\end{proof}

Working out the details of the neural network--Gaussian process correspondence for a specific network architecture involves computing the Gaussian process kernel that is induced by the given network topology and choice of nonlinearity. The following theorem demonstrates this process for the \textit{model organism} of this thesis: the multilayer perceptron with relu nonlinearity. The essence of the following theorem appears in a paper by \citet{lee2018deep}, building on the work of \citet{choandsaul}.

\begin{theorem}[NNGP for relu networks]\label{thm:relu} Consider a multilayer perceptron $f$ (Definition \ref{def:mlp}) with $L$ layers, output dimension $d_L=1$ and nonlinearity:
\begin{equation}\label{eq:scaled-relu}
    \phi(\cdot) = \sqrt{2}\cdot\max(0,\cdot).
\end{equation}
For each layer $l=1,...,L$, sample weight entries $W_l^{ij} \overset{\mathrm{iid}}{\sim} \normal(0,1/d_{l-1})$. Consider any collection $X$ of $m$ inputs constrained to the hypersphere of radius $\sqrt{d_0}$: $x_1, ..., x_m\in\sqrt{d_0}\cdot \Sph^{d_0-1}$. Then, as hidden layer widths $d_1,...,d_{L-1}\to \infty$, the network outputs $f(x_1),...,f(x_m)$ are jointly Gaussian with:
\begin{align}
\Expect \left[ f(x_i) \right] &= 0;\\
\Expect \smash{\left[ f(x_i)^2 \right]} &= 1;\\ 
\Expect \left[ f(x_i) f(x_j)\right] &= \underbrace{h \circ ... \circ h}_{L-1\text{  times}}\smash{\left(\frac{x_i^\top x_j}{d_0}\right)}; \label{eq:comp-arccos}
\end{align}
for all $x_i, x_j \in X$, and where $h(t)\coloneqq  \tfrac{1}{\pi}\left[\sqrt{1-t^2} + t\cdot(\pi- \arccos t)\right]$.
\end{theorem}

The kernel appearing in Equation \ref{eq:comp-arccos} is the \textit{compositional arccosine kernel} \citep{choandsaul}. This kernel will be used in Chapter \ref{chap:bpm} to study generalisation in the multilayer perceptron.

\begin{proof} First, Lemma \ref{lem:nngp} will be applied recursively over the layers of the network. Condition (ii) of Lemma \ref{lem:nngp} holds at all layers with $\sigma^2 = 1$, and condition (iii) holds trivially for the scaled relu nonlinearity of Equation \ref{eq:scaled-relu}. Condition (i) holds at the first layer since, in the notation of Equation \ref{eq:nn-layer},
\begin{align}
\Expect[f_1(x)] &= \Expect[W_1 \cdot x] = \Expect[W_1] \cdot x = 0, \nonumber\\ \Expect[(f_1^i(x))^2] &= \sum_{j,k=1}^{d_0} \Expect[W_1^{ij}W_1^{ik}] \cdot x^j x^k = \sum_{j=1}^{d_0} \Expect[(W_1^{ij})^2] \cdot (x^j)^2 = 1, \label{eq:var-layer-1}
\end{align}
and $\phi$ preserves both iid-ness and finite-ness of the first and second moment. Then, by recursive application of Lemma \ref{lem:nngp}, the outputs at any layer are jointly Gaussian. All that remains is to compute their moments.

For the $i$th component of layer $l$, the first moment $\Expect \left[ f_l^{i}(x)\right] = 0$. This can be seen by taking the expectation of Equation \ref{eq:nn-layer} and using the fact that the $W_l^{ij}$ are independent of the layer inputs and have mean zero.
    
    Since any two layer outputs $f_l^i(x)$ and $f_l^i(x^\prime)$ are jointly Gaussian with mean zero, their distribution is completely described by the $2\times 2$ covariance matrix:
    \begin{align*}
        \Sigma_{l}(x,x^\prime)&\coloneqq 
          \begin{bmatrix}
            \rho_{l}(x,x) & \rho_{l}(x,x^\prime)  \\
            \rho_{l}(x,x^\prime) & \rho_{l}(x^\prime,x^\prime)
          \end{bmatrix},
    \end{align*}
    where $\rho_{l}(x,x^\prime)\coloneqq\Expect \left[ f_l^i(x) f_l^i(x^\prime)\right]$ and the index $i$ is unimportant since different components in the same layer have identical distributions.
    
    The theorem statement will follow from an effort to express $\Sigma_l(x,x^\prime)$ in terms of $\Sigma_{l-1}(x,x^\prime)$ and then recursing back through the network. By Equation \ref{eq:nn-layer} and independence of the $W_l^{ij}$, the covariance $\rho_l(x,x^\prime)$ may be expressed as: 
    \begin{equation}\label{eq:covar}
        \rho_{l}(x,x^\prime) = \Expect \left[ \phi \left(f_{l-1}^j(x)\right) \phi \left(f_{l-1}^j(x^\prime)\right) \right],
    \end{equation}
    where $j$ indexes an arbitrary component of layer $l-1$. To make progress, it helps to first evaluate:
    \begin{gather*}
        \rho_l(x,x) = \Expect \left[ \phi \left(f_{l-1}^j(x)\right)^2 \right] = \half \cdot 2 \cdot \rho_{l-1}(x,x),
    \end{gather*}
    which follows by the definition of $\phi$ and symmetry of the Gaussian expectation around zero. Then, by recursion:
    \begin{gather*}
        \rho_{l}(x,x) = \rho_{l-1}(x,x) =... = \rho_{1}(x,x) = 1,
    \end{gather*}
    where the final equality holds because $\Expect[(f_1^i(x))^2] = 1$ by Equation \ref{eq:var-layer-1}. Then the covariance $\Sigma_{l-1}$ at layer $l-1$ simplifies to:
    \begin{equation*}
        \Sigma_{l-1}(x,x^\prime)=
          \begin{bmatrix}
            1 & \rho_{l-1}(x,x^\prime)  \\
            \rho_{l-1}(x,x^\prime) & 1
          \end{bmatrix}.
    \end{equation*}
    Equation \ref{eq:covar} may now be used to express $\rho_l(x,x^\prime)$ in terms of $\rho_{l-1}(x,x^\prime)$. Dropping the $(x,x^\prime)$ indexing for brevity:
    \begin{align*}
        \rho_l &= \Expect_{u,v\sim \mathcal{N}\left(0,\Sigma_{l-1}\right)} \left[ \phi \left(u\right) \phi \left(v\right) \right] \\
        &= \frac{1}{\pi \sqrt{1-\rho_{l-1}^2}} \iint_{u,v\geq0}\diff{u}\idiff{v}\, \exp\left[-\frac{u^2 - 2 \rho_{l-1} uv + v^2}{2(1-\rho_{l-1}^2)}\right]uv.
    \end{align*}
    By making the substitution $\rho_{l-1}=\cos\theta$, this integral becomes equivalent to $\frac{1}{\pi}\cdot J_1(\theta)$ as expressed in \citet{choandsaul}'s Equation 15. Substituting in the evaluation of this integral \citep[Equation 6]{choandsaul}, one obtains:
    \begin{equation}\label{eq:recur}
      \rho_{l}(x,x^\prime) = h(\rho_{l-1}(x,x^\prime)).
    \end{equation}
    
    What remains is to evaluate $\rho_1(x,x^\prime)$. Since $\Expect\left[W_1^{ij}W_1^{ik}\right] = \delta_{jk}/d_0$, this equals:
    \begin{align*}
        \rho_1(x,x^\prime) &\coloneqq  \Expect \left[ f_1^i(x) f_1^i(x^\prime)\right] = \sum_{j,k=1}^{d_0} \Expect\left[W_1^{ij}W_1^{ik}\right] x^j (x^\prime)^k = \frac{x^\top x^\prime}{d_0}.
    \end{align*}
    The proof is completed by combining this expression for $\rho_1(x,x^\prime)$ with the recurrence relation in Equation \ref{eq:recur}.
\end{proof}

This completes the chapter on correspondences between function spaces.

\printbibliography[heading=subbibliography]
\end{refsection}

%% file: chapters/4-optimisation-frameworks.tex
\begin{refsection}

\chapter{The Majorise-Minimise Meta-Algorithm}
\label{chap:perturb}

\begin{tcolorbox}
This chapter introduces several classic iterative optimisation methods. The derivations, while themselves not novel, are put on common footing by showing how each corresponds to a form of perturbation analysis.
\end{tcolorbox}

This chapter surveys classic techniques for formally deriving gradient-based optimisation methods. The survey covers first-order methods---such as \textit{gradient descent} and \textit{mirror descent}---as well as second-order methods---such as the \textit{cubic regularised} version of  \textit{Newton's method}. These varied techniques are put on a consistent theoretical footing by showing that one step of each method may be viewed as the minimisation of a particular local description of the loss function. In each case, this local description takes the form of a \textit{perturbation bound} on the error of a \textit{truncated perturbation expansion} of the loss. Such bounds, known as \textit{majorisations}, assess the region within which the perturbation expansion can be trusted.

The optimisation methods described in this chapter make use of perturbation expansions and bounds related to the Taylor series expansion of the loss function $\el$ in the weight space $\mathcal{W}$ of the optimisation problem. As such, they may construct a perturbation $\Delta w\in\mathcal{W}$ to the weights $w\in\mathcal{W}$ by considering any of the following information:
\begin{enumerate}
    \item \textit{First-order information}: The gradient $\nabla_w \el(w)$ of the loss function.
    \item \textit{Second-order information}: The Hessian $\nabla_w^2 \el(w)$ of the loss function.
    \item \textit{Trust regions}: Bounds on the accuracy of the above information.
\end{enumerate}

In contrast to the techniques considered in this chapter, Chapter \ref{chap:maj-min} deals more closely with neural networks. There the loss $\el$ depends on the weights $w\in\mathcal{W}$ indirectly via the network architecture $f(\cdot;w):\mathcal{X}\to\mathcal{Y}$, and it will help to model this dependence. But in this chapter, the loss function will be viewed as a straightforward function of the weights---formally, the loss $\el : \mathcal{W} \to \R$.

\section{Perturbation analysis: Expansions and bounds}
This section introduces the general ideas of perturbation analysis. Section \ref{sec:mm} will show how these ideas may be used to solve optimisation problems.

\subsection{Perturbation expansions}
The simplest example of a perturbation expansion is a straightforward Taylor series. Provided a function $g:\R\to\R$ is analytic, it may be expanded about a point $x\in\R$ as a Taylor series in powers of some small perturbation $\delta x\in\R$:
\begin{equation}\label{eq:taylor}
    g(x+\delta x) = g(x) + \frac{\partial g}{\partial x}\cdot\delta x + \frac{1}{2} \frac{\partial^2 g}{\partial x^2}\cdot\delta x^2 + ...
\end{equation}

A \textit{truncated} perturbation expansion refers to cutting off this series at some order. For instance, truncating Equation \ref{eq:taylor} yields:
\begin{equation}\label{eq:p-expand}
    g^{(2)}(x+\delta x) \coloneqq g(x) + \frac{\partial g}{\partial x}\cdot\delta x + \frac{1}{2} \frac{\partial^2 g}{\partial x^2}\cdot\delta x^2,
\end{equation}
where the superscript $(2)$ indicates a second-order truncation.

\subsection{Perturbation bounds}
While the truncated perturbation expansion in Equation \ref{eq:p-expand} is accurate for sufficiently small $\delta x$, the error in this approximation is unknown without computing the truncated part of the series. A \textit{perturbation bound} deals with this issue by bounding the error in a truncated perturbation expansion. For example, if one can derive a function $h(x,\delta x)$ such that:
\begin{equation}\label{eq:p-bound}
    \left|g(x+\delta x) - g^{(2)}(x+\delta x)\right| \leq h(x,\delta x),
\end{equation}
then Inequality \ref{eq:p-bound} would constitute a perturbation bound. For the perturbation bound to be useful, one aspires to finding a bounding function $h(\cdot,\cdot)$ that:
\begin{enumerate}
    \item provides a reasonably tight bound on the error in the truncated series.
    \item is much cheaper to compute than the omitted terms in the truncation.
\end{enumerate}

\section{Majorise-minimise}\label{sec:mm}

In an iterative minimisation method, one is interested in selecting a perturbation $\Delta w_*$ such that the loss after the perturbation, $\el(w+\Delta w_*)$, is smaller than the loss beforehand, $\el(w)$. A good starting point for the design of such a method is the Taylor expansion of the loss in weight perturbations $\Delta w$:
\begin{equation}\label{eq:weight-taylor}
    \el(w+\Delta w) = \el(w) + \nabla_w \el(w)^\top \Delta w + \frac{1}{2}\cdot \Delta w ^\top \nabla_w^2\el(w)\Delta w + ...
\end{equation}

It is tempting to minimise, say, the first few terms in this Taylor expansion as a proxy for reducing the loss function. Formally, letting $\el^{(k)}(w+\Delta w)$ denote the Taylor expansion in Equation \ref{eq:weight-taylor} truncated to $k$th order, one might be inclined to select a perturbation $\Delta w_*$ via:
\begin{equation}
    \Delta w_* = \argmin_{\Delta w}\left[\el^{(k)}(w+\Delta w)\right].
\end{equation}
For instance, truncating to first order would correspond to the minimisation:
\begin{equation}\label{eq:fo-trunc}
    \Delta w_* =\argmin_{\Delta w}\left[\el^{(1)}(w+\Delta w)\right]= \argmin_{\Delta w}\left[\el(w) + \nabla_w \el(w)^\top \Delta w \right].
\end{equation}
Unfortunately, this procedure is not well-founded. To see this, observe that the minimand appearing in Equation \ref{eq:fo-trunc} can be made arbitrarily negative by setting the perturbation $\Delta w$ to point in the direction of the negative gradient $-\nabla_w \el(w)$ with an arbitrarily large magnitude. This holds even for bounded loss functions which cannot themselves be made arbitrarily negative.

The core issue being highlighted in the previous paragraph is that minimising a truncated series expansion can result in perturbations so large that the truncation is no longer a good approximation to the original loss function. This issue may be addressed by employing a special form of perturbation bound known as a \textit{majorisation}: 

\begin{definition}[Majorisation]\label{def:majorisation} Given an analytic loss function $\el:\mathcal{W}\to\R$, let $\el^{(k)}(w+\Delta w)$ denote the Taylor expansion of $\el$ about $w$ truncated to $k$th order. An analytic function $h:\mathcal{W}\times\mathcal{W}\to\R_{\geq0}$ is a \textit{majorisation} of $\el$ provided that:
\begin{enumerate}[label=(\roman*)]
    \item $h$ gives a one-sided bound on the error of the truncated Taylor series:
    \begin{equation}\label{eq:major}
    \el(w+\Delta w) \leq \el^{(k)}(w+\Delta w) + h(w,\Delta w).
    \end{equation}
    \item $h$ is zero whenever the perturbation $\Delta w$ is zero:
    \begin{equation}
        h(w, 0) = 0 \text{ for all } w\in\mathcal{W}.
    \end{equation}
\end{enumerate}
\end{definition}

Taken together, these two conditions imply that a majorisation provides an upper bound to the perturbed loss $\el(w+\Delta w)$ (Inequality \ref{eq:major}) that lies \textit{tangent} to the loss $\el$ at weight setting $w$. A graphical example is provided in Figure \ref{fig:major-minor}. The reason that such a construction is helpful is that minimising a majorisation will also reduce the original loss. This idea is also visualised in Figure \ref{fig:major-minor}.

Formally, if $h$ is a majorisation and a perturbation $\Delta w_*$ is selected via:
\begin{equation}\label{eq:minor}
    \Delta w_* =\argmin_{\Delta w} \left[\el^{(k)}(w+\Delta w) +  h(w,\Delta w) \right],
\end{equation}
then $\el(w+\Delta w_*) \leq \el(w)$, with equality only when $w$ was already a minimum.

This process, of minimising a tangent upper bound to a loss function, is known as the \textit{majorise-minimise} meta-algorithm \citep{mm}. It is a \textit{meta-algorithm} in the sense that many different methods may be derived by following this procedure in different situations. Examples are given in the next section.

\begin{figure}
    \centering
    \includegraphics{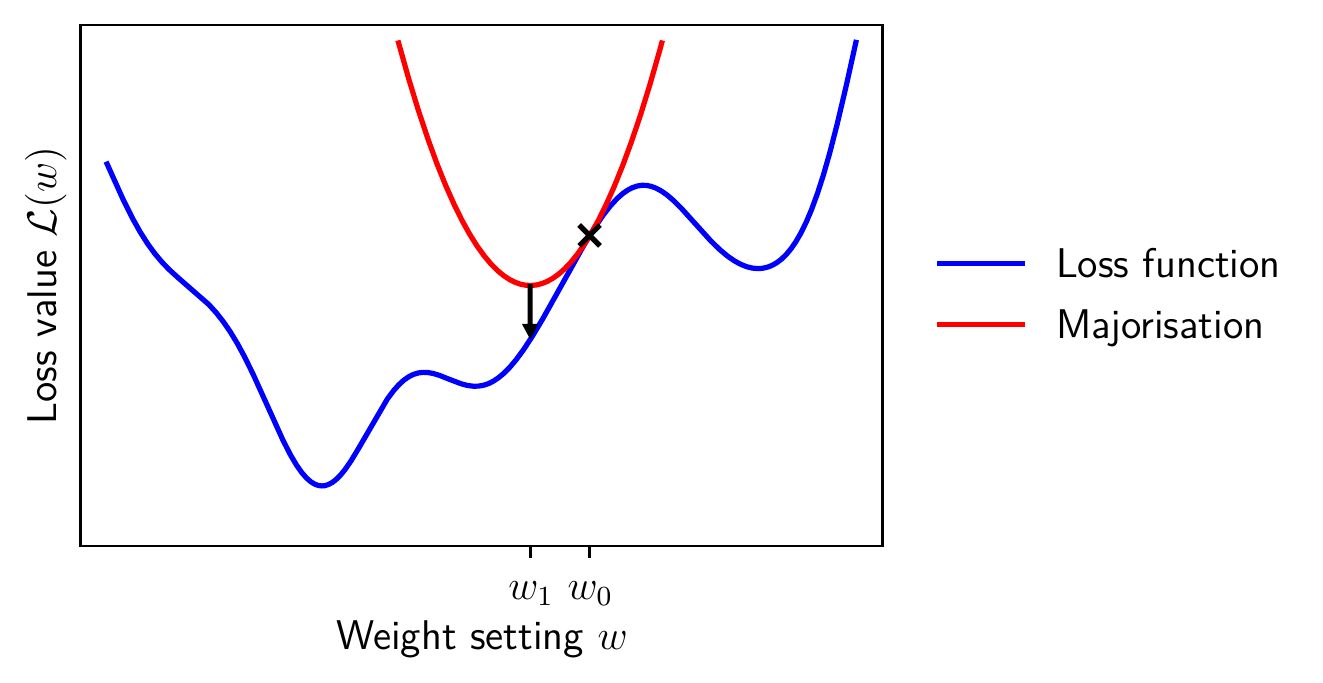}
    \caption[The majorise-minimise meta-algorithm]{The majorise-minimise meta-algorithm. The blue curve denotes a loss function that one would like to reduce, starting from a point $w_0$. The red curve denotes a \textit{majorisation} of the loss about $w_0$. Since the majorisation is both an upper bound to the loss that is also tangent at $w_0$, minimising the majorisation to obtain a new weight setting $w_1$ also reduces the original loss.}
    \label{fig:major-minor}
\end{figure}

\section{Instantiations of the meta-algorithm}
\label{sec:examples}

First-order optimisation methods minimise a majorisation of the first-order Taylor expansion of the loss function. They therefore rely on first-order \textit{gradient} information and do not have access to second-order \textit{Hessian} information. Examples include basic gradient descent and also \textit{mirror descent}.

Meanwhile, second-order optimisation methods minimise a majorisation of the second-order Taylor expansion of the loss function. Both first-order \textit{gradient} information as well as second-order \textit{Hessian} information are used. The hope is that including this extra information will lead to a more effective optimisation step. One example is the \textit{cubic-regularised} version of \textit{Newton's method}.

\subsection{First example: Gradient descent}

This subsection shows that gradient descent in its most basic form arises from the following majorisation of the first-order Taylor expansion of the loss:
\begin{equation}\label{eq:gd-major}
    \el(w+\Delta w) \leq \el(w) + \nabla_w \el(w)^\top \Delta w + \frac{\lambda}{2}\cdot \norm{\Delta w}_2^2,
\end{equation}
for some constant $\lambda>0$. This is a \textit{Euclidean} majorisation, since the Euclidean norm $\|\cdot\|_2$ characterises the realm of validity of the first-order Taylor expansion.

One must ask, for which loss functions is Inequality \ref{eq:gd-major} a valid majorisation? Here it helps to define the class of \textit{gradient-Lipschitz} loss functions.
\begin{definition}[Gradient-Lipschitz loss function] A loss function $\el:\R^d\to\R$ is \textit{gradient-Lipschitz} with constant $\lambda>0$ if:
\begin{equation}
    \norm{\nabla_w\el(w+\Delta w) - \nabla_w\el(w)}_2 \leq \lambda\cdot\norm{\Delta w}_2.
\end{equation}
\end{definition}

The following lemma shows that a gradient Lipschitz loss function is majorised according to Inequality \ref{eq:gd-major}.
\begin{lemma}[Gradient-Lipschitz Euclidean majorisation]\label{lem:g-lipsch-major} Given that a loss function $\el:\R^d \to \R$ is gradient-Lipschitz with constant $\lambda$, then the Euclidean majorisation (Inequality \ref{eq:gd-major}) holds.
\end{lemma}
\begin{proof} By the fundamental theorem of calculus, the Cauchy-Schwarz inequality and finally the gradient-Lipschitz property:
\begin{align*}
    &\el(w+\Delta w) - \left[\el(w) + \nabla_w \el(w)^\top \Delta w\right]\\ &\qquad\qquad= \int_0^1\diff{t}\; \left(\nabla_w \el(w+t\cdot \Delta w)-\nabla_w\el(w)\right)^\top\Delta w\\
    &\qquad\qquad\leq \int_0^1\diff{t}\; \norm{\nabla_w \el(w+t\cdot \Delta w)-\nabla_w\el(w)}_2\cdot\norm{\Delta w}_2\\
    &\qquad\qquad \leq \int_0^1\diff{t}\; t \cdot \lambda \cdot \norm{\Delta w}_2^2 = \frac{\lambda}{2}\cdot\norm{\Delta w}_2^2.
\end{align*}
The proof is completed by adding the first-order Taylor series to both sides.
\end{proof}

Finally, the following theorem shows that the gradient descent optimisation algorithm emerges via minimisation of the majorisation in Inequality \ref{eq:gd-major}.

\begin{theorem}[Gradient descent]\label{thm:gd} The following holds:
\begin{equation}
    \argmin_{\Delta w} \left[\el(w) + \nabla_w \el(w)^\top \Delta w + \frac{\lambda}{2} \cdot \norm{\Delta w}_2^2\right] = - \frac{1}{\lambda} \cdot \nabla_w \el(w).
\end{equation}
\end{theorem}
\begin{proof}
    Take the derivative of the minimand on the left-hand side with respect to $\Delta w$, set this derivative to zero, and solve for $\Delta w$.
\end{proof}

\subsection{Second example: Mirror descent}

The previous subsection showed that gradient descent in its most basic form arises from an intrinsically Euclidean majorisation of the loss. This subsection shows that a certain kind of \textit{non-Euclidean} majorisation leads to a variant of gradient descent known as \textit{mirror descent} \citep{nemirovsky_yudin_1983}.

The kind of non-Euclidean majorisation in question has a particular form: it assesses the error in the first-order Taylor expansion of a convex function:

\begin{definition}[Bregman divergence] Given a convex function $\psi:\mathcal{W} \to \R$, the \textit{Bregman divergence} $h_\psi:\mathcal{W}\times\mathcal{W} \to \R$ corresponding to $\psi$ is given by:
\begin{equation}
    h_\psi(w,w+\Delta w) \coloneqq \psi(w+\Delta w) - \left[\psi(w) + \nabla \psi(w)^\top \Delta w\right].
\end{equation}
\end{definition}
Note that by a basic property of convexity, $\psi(w+\Delta w)$ always lies above the tangent to $\psi$ at $w$, and therefore $h_\psi \geq 0$. Also, it is quick to check that $h_\psi$ satisfies the conditions of Definition \ref{def:majorisation} to be a valid majorisation of $\psi$.

One hopes to apply mirror descent in situations where the trust region of the loss function of interest $\el$ is well-modeled by $h_\psi$:

\begin{assumption}[Majorisation via Bregman divergence]\label{ass:bregman-major} For a loss function $\el: \mathcal{W} \to \R$ and a convex function $\psi:\mathcal{W} \to \R$, assume that:
\begin{equation}
    \el(w+\Delta w) \leq \el(w) + \nabla_w \el(w)^\top \Delta w + h_\psi(w,w+\Delta w).
\end{equation}
\end{assumption}
In words: the majorisation $h_\psi$ of convex function $\psi$ is assumed to also majorise the loss function $\el$. The reason that this assumption is interesting is that the corresponding majorise-minimise algorithm has a particularly elegant form:

\begin{proposition}[Mirror descent] Let perturbation $\Delta w_*$ denote the minimiser of the majorisation given in Assumption \ref{ass:bregman-major}:
\begin{equation}
    \Delta w_* \coloneqq \argmin_{\Delta w} \left[\el(w) + \nabla_w \el(w)^\top \Delta w + h_\psi(w,w+\Delta w) \right].
\end{equation}
Then $\Delta w_*$ satisfies the following first-order optimality condition:
\begin{equation}\label{eq:mirror-domain}
    \nabla \psi(w+\Delta w_*) = \nabla \psi(w) - \nabla_w\el(w).
\end{equation}
Furthermore, when $\nabla \psi$ is invertible, the optimal perturbation satisfies:
\begin{equation}\label{eq:mirror-back}
    w+\Delta w_* = \nabla \psi^{-1}\left[\nabla \psi(w) - \nabla_w\el(w)\right].
\end{equation}
\end{proposition}

Equation \ref{eq:mirror-domain} may be interpreted as a vanilla gradient descent update to optimisation variables that have been transformed by the map $\nabla\psi:\R^d\to\R^d$. So for the kind of non-Euclidean majorisation given in Assumption \ref{ass:bregman-major}, optimisation still admits a simple additive structure when viewed in the \textit{mirror domain} $\nabla \psi (\mathcal{W}) \coloneqq \{ \nabla \psi (w) : w\in\mathcal{W}\}$.

Mirror descent makes an important departure from the Euclidean structure of gradient descent by leveraging a particular kind of non-Euclidean majorisation given by a Bregman divergence (Assumption \ref{ass:bregman-major}). But there may be other means of constructing non-Euclidean majorisations. Non-Euclidean majorisations directly tailored to learning problems will be explored in Chapters \ref{chap:maj-min} and \ref{chap:nn-maj-min}.

\subsection{Third example: Cubic regularised Newton}

A first attempt toward building a second-order optimisation method is known as \textit{Newton's method}. Newton's method attempts to minimise the Taylor series expansion of the loss function truncated to second order:
\begin{equation}\label{eq:so-trunc}
    \Delta w^* = \argmin_{\Delta w} \left[\el(w) + \nabla_w \el(w)^\top \Delta w + \frac{1}{2}\cdot\Delta w^\top \nabla^2_w\el(w) \Delta w\right].
\end{equation}

Unfortunately, as was the case for minimising the first-order truncation (Equation \ref{eq:fo-trunc}), this procedure is generally not well-founded. For instance, if the Hessian $\nabla^2_w\el(w)$ has any negative eigenvalues, then the minimand appearing in Equation \ref{eq:so-trunc} can be made arbitrarily negative by selecting a $\Delta w$ in the corresponding negative eigenspace, with arbitrarily large magnitude.

Again, this issue may be solved by a majorisation. Just as Lemma \ref{lem:g-lipsch-major} showed that a gradient-Lipschitz loss function admits a quadratic majorisation of the first-order Taylor expansion, a Hessian-Lipschitz loss function admits a cubic majorisation of the second-order expansion. This insight leads to the \textit{cubic regularised} version of Newton's method due to \citet{Nesterov2006CubicRO}:
\begin{equation}
    \Delta w_* = \argmin_{\Delta w} \left[\el^{(2)}(w+\Delta w) + \frac{\lambda}{6}\cdot\norm{\Delta w}_2^3\right].
\end{equation}
\input{tables/opt-frameworks}

The authors show how to solve this optimisation sub-problem, and provide some global convergence results for this method in certain settings.

\clearpage

\section{Trade-off between computation and fidelity}

This chapter has presented derivations of various optimisation methods, summarised in Table \ref{tab:trust}. The derivations work by majorising a truncated Taylor expansion in weight perturbations, and then minimising this majorisation. This process is formally described by Equation \ref{eq:minor} and illustrated in Figure \ref{fig:major-minor}.

What has not been discussed is the computational complexity of this procedure. In general, there is a trade-off between the fidelity of any particular majorisation, and the computational cost of its evaluation. For instance, majorising a higher-order Taylor expansion may lead to a tighter perturbation bound with a larger region of validity, allowing each optimisation step to make more progress. But one needs to ask if this larger per-step improvement is worth the extra computational overhead of evaluating a higher-order perturbation bound.

A good case in point is first-order versus second-order methods. While the per-step improvement of a second-order method may well exceed that of a first-order method, the per-step cost of the second-order method may be prohibitively expensive. On a $d$-dimensional weight space, $\mathcal{W}=\R^d$:
\begin{enumerate}
    \item The Hessian matrix $\nabla_w^2\el(w)\in\R^{d\times d}$ requires $\mathcal{O}(d^2)$ memory to store.
    \item The gradient vector $\nabla_w\el(w)\in\R^{d}$ requires $\mathcal{O}(d)$ memory to store.
\end{enumerate}

So first-order methods may be preferable on high-dimensional weight spaces.

\printbibliography[heading=subbibliography]
\end{refsection}

%% file: tables/opt-frameworks.tex
\begin{table}[p]
    \centering
    \def\arraystretch{2}
    \begin{tabular*}{\textwidth}{l @{\extracolsep{\fill}} cc}
    \textbf{Optimiser} & \textbf{Truncation order, $\mathbf{k}$} &  \textbf{Majorisation, $\mathbf{h}$} \\
    gradient descent & $k=1$ & $\displaystyle \frac{\lambda}{2} \cdot \norm{\Delta w}_2^2$ \\
    mirror descent & $k=1$ & $\displaystyle h_\psi(w,w+\Delta w)$\\
    cubic regularised Newton & $k=2$ & $\displaystyle \frac{\lambda}{6}\cdot\norm{\Delta w}_2^3$ \\
    \end{tabular*}
    \vspace{4ex}
    \caption[Optimisation methods and their corresponding majorisations]{Optimisation methods and their corresponding majorisations. Each optimiser perturbs a weight vector $w\mapsto w+\Delta w_*$ where the perturbation $\Delta w_*$ is selected by solving $\Delta w_* = \argmin_{\Delta w}[\el^{(k)}(w+\Delta w) + h(w,\Delta w)]$. In this expression, $\el^{(k)}(w+\Delta w)$ refers to the Taylor series expansion of $\el(w+\Delta w)$ in perturbation $\Delta w$ truncated to $k$th order.}
    \label{tab:trust}
\end{table}

%% file: chapters/5-majorisation-minimisation.tex
\begin{refsection}

\chapter{Majorise-Minimise for Learning Problems}
\label{chap:maj-min}

\begin{tcolorbox}
This chapter introduces two novel techniques for machine learning optimisation problems: \textit{functional majorisation} of the loss function and \textit{architectural perturbation bounds} for machine learning models. Together these techniques allow the majorise-minimise meta-algorithm to be applied to generic learning problems. This chapter uses the techniques to re-derive gradient descent and the Gauss-Newton method. Chapter \ref{chap:nn-maj-min} will use the techniques to derive architecture aware optimisation algorithms.
\end{tcolorbox}

While Chapter \ref{chap:perturb} dealt with somewhat generic optimisation algorithms, the focus of this chapter is a framework for deriving optimisation algorithms for machine learning problems. A central feature of this type of optimisation problem is that the loss function $\el$ is affected by the weight vector $w$ only indirectly through its appearance in the machine learning model $f(\cdot;w)$. This presents an opportunity: one can design optimisation algorithms that leverage the architecture of the machine learning model $f(\cdot;w)$ in interesting and useful ways. This idea will be termed \textit{architecture aware optimisation}.

This chapter makes an important conceptual shift from considering a weight perturbation $\Delta w$ directly, as in Chapter \ref{chap:perturb}, to first studying the functional perturbation $\Delta f$  that is induced by the weight perturbation $\Delta w$:
\begin{equation}
        \Delta f(\cdot) \coloneqq f(\cdot;w+\Delta w)-f(\cdot;w).
\end{equation}

After making this shift, the following three-step framework is proposed for deriving architecture aware optimisation algorithms:
\begin{enumerate}[label=Step \arabic*:, leftmargin=*, font=\sffamily]
    \item \textit{Functional majorisation of the loss}. Expand the loss function as a series in functional perturbations, and majorise this expansion in terms of the size of functional perturbations. Lemma \ref{lem:sq-major} gives an example.
    \item \textit{Architectural perturbation bounds.} Derive bounds that relate the size of functional perturbations to the size of weight perturbations by analysing the model architecture. Lemma \ref{lem:deep_perturbation_bounds} gives an example.
    \item \textit{Majorise--minimise.} Substitute the architectural perturbation bounds into the functional majorisation of the loss and minimise with respect to the weight perturbation to obtain an optimisation algorithm.
\end{enumerate}

For the case of linear regression, this framework turns out to reproduce the classic gradient descent algorithm. Under the assumption that functional perturbations are linear in weight perturbations, the framework reproduces the classic \textit{Gauss-Newton method}, which is closely related to \textit{natural gradient descent}. The real payoff comes in Chapter \ref{chap:nn-maj-min}, where the framework is applied to deep neural networks---yielding architecture aware optimisation methods.

\section{Expanding the loss as a series in functional perturbations}

This section derives a novel series expansion of machine learning loss functions in terms of functional perturbations. The key idea is to Taylor expand the loss function in functional perturbations, and then to transform the linear terms in this expansion back to weight space. This last step renders the expansion more suitable for deriving optimisation algorithms that operate in weight space.

This expansion relies on the fact that a machine learning loss function $\el$ can be regarded either as a function of a weight vector $w$ or a projected function $f_X(w)$ (Definition \ref{def:project}). With this in mind, this chapter will sometimes abuse notation by insisting that $\el(w) \equiv \el(f_X(w))$. Concretely, for square loss:
\begin{equation*}
    \el_2(w) = \frac{1}{2m}\cdot\norm{f_X(w)-Y}_2^2\;\;\implies\;\;
    \el_2(f_X) = \frac{1}{2m}\cdot\norm{f_X-Y}_2^2.
\end{equation*}

To derive this expansion, it will help to define a notion of functional perturbation projected over a set of inputs:
\begin{definition}[Projected functional perturbation] Given a function $f:\mathcal{X}\times \mathcal{W}\to\R$ and a collection of $m$ inputs $X=\{x_1,...,x_m\}$, the \textit{projected functional perturbation} $\Delta f_X\in\R^m$ corresponding to unperturbed weight vector $w$ and weight perturbation $\Delta w$ is given by the difference of projections:
\begin{equation}
    \Delta f_X \coloneqq f_X(w+\Delta w) - f_X(w).
\end{equation}
\end{definition}
The projected functional perturbation $\Delta f_X$ implicitly depends on both a weight vector $w$ and a weight perturbation $\Delta w$, but this dependence is suppressed for brevity. Given this definition, the loss function may be expanded as follows:

\begin{lemma}[Series expansion in functional perturbations]\label{lem:function-taylor} Given a function $f: \mathcal{X}\times \mathcal{W} \to \R$, a set of training inputs $X = \{x_1,...,x_m\}$, and a loss $\el$ that is an analytic function of the function space projection $f_X$, the following holds:
    \begin{align}
        &\el(w+\Delta w) -\left[ \el(w) + \nabla_w\el(w)^\top \Delta w\right] \nonumber \\
        &\qquad= \nabla_{f_X} \el(f_X)^\top \left[\Delta f_X - \nabla_w f_X(w) \Delta w \right] + \half \Delta f_X^\top \nabla^2_{f_X} \el(f_X) \Delta f_X + ...
    \end{align}
\end{lemma}
\begin{proof} First, Taylor expand the loss in functional perturbations $\Delta f_X$:
    \begin{equation*}
        \el(f_X+\Delta f_X) = \el(f_X) + \nabla_{f_X} \el(f_X)^\top \Delta f_X + \half \Delta f_X^\top \nabla^2_{f_X} \el(f_X) \Delta f_X + ...
    \end{equation*}
Next, make the substitutions $\el(f_X+\Delta f_X) \equiv \el(w+\Delta w)$ and $\el(f_X) \equiv \el(w)$:
\begin{equation*}
        \el(w+\Delta w) = \el(w) + \nabla_{f_X} \el(f_X)^\top \Delta f_X + \half \Delta f_X^\top \nabla^2_{f_X} \el(f_X) \Delta f_X + ...
    \end{equation*}
    Finally, subtract $\left[ \el(w) + \nabla_w\el(w)^\top \Delta w\right]$ from both sides and apply the chain rule $\nabla_w\el(w)^\top \Delta w = \nabla_{f_X}\el(f_X)^\top\nabla_w f_X \Delta w$ on the right-hand side.
\end{proof}

The requirement of Lemma \ref{lem:function-taylor} that the loss function be analytic in the function space projection $f_X$ is mild. Most common loss functions---including square loss (Definition \ref{def:sq-loss}) and logistic loss (Definition \ref{def:log-loss}) satisfy this requirement---even when the loss is not an analytic function of the weight vector $w$.

\section{Functional majorisation of square loss}

This section specialises the series expansion in functional perturbations (Lemma \ref{lem:function-taylor}) to the case of square loss. This is a particularly convenient example to consider, since the series expansion terminates at second order.

\begin{lemma}[Expanding square loss in functional perturbations]\label{lem:sq-expand} The square loss (Proposition \ref{ex:sq-loss-projected}) admits the following expansion:
    \begin{align}
        &\el_2(w+\Delta w) -\left[ \el_2(w) + \nabla_w\el_2(w)^\top \Delta w\right] \nonumber \\
        &\qquad\qquad= \frac{1}{m} (f_X-Y)^\top \left[\Delta f_X - \nabla_w f_X \Delta w \right] + \frac{1}{2m} \|\Delta f_X\|_2^2.
    \end{align}
\end{lemma}
\begin{proof}
For the square loss, $\nabla_{f_X} \el(f_X) = \frac{1}{m}(f_X-Y)$ and $\nabla^2_{f_X} \el(f_X) = \frac{1}{m} \cdot \Id$. All higher-order derivatives with respect to $f_X$ are zero. Substituting these relations in to Lemma \ref{lem:function-taylor} yields the result.
\end{proof}

By an application of the Cauchy-Schwarz inequality, Lemma \ref{lem:sq-expand} leads directly to the following majorisation of the square loss in terms of functional perturbations:

\begin{lemma}[Functional majorisation of square loss]\label{lem:sq-major} The square loss (Proposition \ref{ex:sq-loss-projected}) admits the following majorisation:
    \begin{align}\label{eq:sq-major}
        &\el_2(w+\Delta w) -\left[ \el_2(w) + \nabla_w\el_2(w)^\top \Delta w\right] \nonumber \\
        &\qquad\qquad\leq \frac{1}{m}\cdot \norm{f_X-Y}_2\cdot\norm{\Delta f_X - \nabla_w f_X \Delta w}_2 + \frac{1}{2m} \cdot \|\Delta f_X\|_2^2.
    \end{align}
\end{lemma}
It is worth drawing attention to the three important quantities appearing on the right-hand side of Inequality \ref{eq:sq-major}:
\begin{enumerate}
    \item $\norm{f_X-Y}_2$ measures the size of the current misfit of the training sample.
    \item $\norm{\Delta f_X - \nabla_w f_X \Delta w}_2$ measures the degree to which the projected functional perturbation $\Delta f_X$ deviates from its linearisation in $\Delta w$.
    \item $\|\Delta f_X\|_2^2$ measures the size of the projected functional perturbation.
\end{enumerate}

While the current data misfit is usually easy to compute or estimate in a machine learning problem, the latter two quantities are considerably more subtle. This thesis suggests relating these quantities back to weight perturbations via \textit{architectural perturbation bounds}. In turn, this will open the door to architecture aware optimisation algorithms.

\section{First application: Deriving gradient descent}

As a first complete example of the framework proposed in this chapter, this section derives architectural perturbation bounds for linear regression. Combining these bounds with the majorisation of square loss (Lemma \ref{lem:sq-major}) leads back to the simple gradient descent optimisation algorithm. While in practice one might prefer to use less rudimentary means of fitting a linear regressor, the power of the argument outlined here is that it will generalise to deep networks.

A linear regressor is a function $f:\R^d \times \R^d \to \R$ of the form:
\begin{equation}\label{eq:lin-regress}
    f(x;w) \coloneqq w^\top x.
\end{equation}
Suppose one wishes to fit a linear regressor to data $X=\{x_1,...,x_m\}$ by running iterative minimisation of the square loss (Example \ref{ex:sq-loss-projected}). Further, suppose that the data is constrained to the unit hypersphere: $x_1,...,x_m \in \Sph^{d-1}$. Then one may leverage the following architectural perturbation bounds:

\begin{lemma}[Architectural perturbation bounds for linear regression]\label{lem:arch-perturb-linear} Given a set $X$ of $m$ training inputs supported on the hypersphere $\sqrt{d}\cdot\Sph^{d-1}$, the linear regressor of Equation \ref{eq:lin-regress} satisfies:
\begin{align}
    \|\Delta f_X\|_2 &\leq \sqrt{md} \cdot \norm{\Delta w}_2. \label{eq:lin-func-change} \\
    \norm{\Delta f_X - \nabla_w f_X \Delta w}_2 &= 0; \label{eq:lin-lin}
\end{align}
\end{lemma}
\begin{proof}
By linearity of Equation \ref{eq:lin-regress}, the projected functional perturbation is:
\begin{equation*}
    \Delta f_X = (\Delta w^\top x_1, ..., \Delta w^\top x_m).
\end{equation*}
Inequality \ref{eq:lin-func-change} follows from an application of the Cauchy-Schwarz inequality:
\begin{equation*}
    \norm{\Delta f_X}_2^2 = \sum_{i=1}^m (\Delta w^\top x_i)^2 \leq \sum_{i=1}^m \norm{\Delta w}_2^2\cdot\norm{x_i}_2^2 = md \cdot \norm{\Delta w}_2^2,
\end{equation*}
where the last equality follows since $x_i \in \sqrt{d}\cdot \Sph^{d-1}$. Finally, Equation \ref{eq:lin-lin} follows by observing that:
\begin{equation*}
    \Delta f_X - \nabla_w f_X \Delta w= (\Delta w^\top x_1, ..., \Delta w^\top x_m) - (x_1,...,x_m)\cdot\Delta w = 0.
\end{equation*}
This completes the proof.
\end{proof}

By combining these architectural perturbation bounds (Lemma \ref{lem:arch-perturb-linear}) with the functional majorisation of the square loss (Lemma \ref{lem:sq-major}), one obtains:
\begin{theorem}[Majorisation of the square loss for linear regression]
\begin{align}
        &\left|\el_2(w+\Delta w) -\left[ \el_2(w) + \nabla_w\el_2(w)^\top \Delta w\right]\right| \leq \frac{d}{2}\cdot \|\Delta w\|_2^2.
    \end{align}
\end{theorem}

Then, by Theorem \ref{thm:gd}, the optimisation algorithm that minimises this majorisation with respect to $\Delta w$ is gradient descent with step-size $1/d$.

\section{Second application: Deriving the Gauss-Newton method}
\label{sec:ngd}

This section shows that the classic \textit{Gauss-Newton method} \citep{gauss-newton}, which is closely related to \textit{natural gradient descent} \citep{revisiting-ngd,amari}, may be derived under the proposed framework in a straightforward manner. In particular, the Gauss-Newton method is the minimiser of the functional majorisation of square loss (Lemma \ref{lem:sq-major}) under the assumption that functional perturbations are linear in weight perturbations:

\begin{assumption}[Functional perturbations are linear]\label{ass:func-linear}
The functional perturbation $\Delta f_X$ corresponding to a weight perturbation $\Delta w$ is given by:
\begin{equation}
    \Delta f_X = \nabla_w f_X \Delta w.
\end{equation}
\end{assumption}
In words: Assumption \ref{ass:func-linear} amounts to approximating the functional perturbation $\Delta f_X$ by its Taylor series in weight perturbations $\Delta w$ truncated to first order. When combined with the following definition, this assumption leads to very simple architectural perturbation bounds:

\begin{definition}[Squared Jacobian]\label{def:sq-jacobian} Consider a machine learning model $f:\mathcal{X}\times\mathcal{W}\to\R$ and a set of $m$ training inputs $X=\{x_1,...,x_m\}$. The squared Jacobian $F_X$ is given by:
\begin{equation}
    F_X \coloneqq \frac{1}{m}\cdot\nabla_w f_X^\top \nabla_w f_X.
\end{equation}
To make Defintion \ref{def:sq-jacobian} more explicit, for a $d$-dimensional weight space $\mathcal{W}=\R^d$, the squared Jacobian is the $d\times d$ matrix whose $(ik)$th entry is given by:
\begin{equation}
    F_X^{ik} = \frac{1}{m}\cdot \sum_{j=1}^m \frac{\partial f(x_j;w)}{\partial w_i}\cdot\frac{\partial f(x_j;w)}{\partial w_k}.
\end{equation}
\end{definition}
In the literature on natural gradient descent, the matrix $F_X$ is connected to the \textit{Fisher information matrix} of information geometry \citep{info-geom}. 

With Definition \ref{def:sq-jacobian} in hand, the following lemma is immediate:
\begin{lemma}[Architectural perturbation bounds for linear functional perturbations]\label{lem:arch-perturb-linear-functional}
Under Assumption \ref{ass:func-linear}, given a set of $m$ inputs $X$, the following hold:
\begin{align}
    \|\Delta f_X\|_2 &= \sqrt{m}\cdot\sqrt{\Delta w^\top F_X \Delta w}; \label{eq:ngd-func-change} \\
    \norm{\Delta f_X - \nabla_w f_X \Delta w}_2 &= 0. \label{eq:ngd-lin}
\end{align}
\end{lemma}
Referring to the results in Lemma \ref{lem:arch-perturb-linear-functional} as architectural perturbation \textit{bounds} is technically a misnomer since these results are actually \textit{equalities}. The thesis persists with this misnomer to emphasise the connection to Lemmas \ref{lem:arch-perturb-linear} and \ref{lem:deep_perturbation_bounds}. These architectural perturbation bounds lead to the following majorisation:

\begin{lemma}[Majorisation of square loss for linear functional perturbations]\label{lem:sq-major-ngd} 
Under Assumption \ref{ass:func-linear}, the square loss (Proposition \ref{eq:sq-loss-projected}) admits majorisation:
    \begin{align}\label{eq:sq-major-ngd} &\el_2(w+\Delta w) -\left[ \el_2(w) + \nabla_w\el_2(w)^\top \Delta w\right] \leq \frac{1}{2} \Delta w^\top F_X \Delta w.
    \end{align}
\end{lemma}
\begin{proof}
Substitute the results of Lemma \ref{lem:arch-perturb-linear-functional} into Lemma \ref{lem:sq-major}.
\end{proof}

Finally, the Gauss-Newton method is obtained:
\begin{theorem}[Gauss-Newton method] Under Assumption \ref{ass:func-linear}, Lemma \ref{lem:sq-major-ngd} produced a majorisation of square loss given by Inequality \ref{eq:sq-major-ngd}. The minimiser of this majorisation is as follows:
\begin{equation}
    \argmin_{\Delta w}\left[\el_2(w) + \nabla_w \el_2(w)^\top\Delta w + \frac{1}{2}\cdot\Delta w^\top F_X \Delta w\right] = - F_X^{-1}\cdot\nabla_w \el_2(w).
\end{equation}
\end{theorem}
\begin{proof}
    Set to zero the derivative of the minimand on the left-hand side with respect to $\Delta w$: $\nabla_w \el_2(w) + F_X\Delta w = 0$. Solve for $\Delta w$ to yield the result.
\end{proof}

In summary: this chapter developed a framework for deriving optimisation algorithms for generic learning problems. The framework involves minimising a functional majorisation of the loss. Architectural perturbation bounds are employed to relate functional perturbation to weight perturbations. The chapter concluded by demonstrating that both gradient descent and the Gauss-Newton method may be re-derived under this framework. The next chapter will apply this framework to deep neural networks.

\printbibliography[heading=subbibliography]
\end{refsection}

%% file: chapters/6-neural-net-optimiser.tex
\begin{refsection}

\chapter{Majorise-Minimise for Deep Networks}
\label{chap:nn-maj-min}

\begin{tcolorbox}
This chapter derives architectural perturbation bounds for deep linear networks. When combined with the functional majorisation of a loss function, the bounds yield novel \textit{architecture aware} optimisation methods.
\end{tcolorbox}

Chapter \ref{chap:maj-min} developed a framework for deriving optimisation algorithms for generic machine learning problems. In essence, the framework describes a majorise-minimise meta-algorithm \citep{mm} for composite optimisation problems that apply an error measure to a function projected on to data.

The present chapter specialises this framework to machine learning problems that involve deep neural networks. Again, the framework works in three steps:

\begin{enumerate}[label=Step \arabic*:, leftmargin=*, font=\sffamily]
    \item Majorise a series expansion of the loss function in perturbations to the network output projected on to the training set.
    \item Derive architectural perturbation bounds that express the sensitivity of the network output to weight perturbations. The form of these bounds depends on details such as the width and depth of the network.
    \item Substitute the architectural perturbation bounds into the majorisation of the loss and minimise to obtain an optimisation algorithm.
\end{enumerate}

\subsection{A note on related work}

\citet{cohen2021gradient} express concern that majorise-minimise style optimisation theories (or ones based on Majorisation \ref{eq:gd-major} in particular) may be too pessimistic for deep learning. On the contrary, this chapter obtains various useful architectural scaling rules for learning rate from a majorise-minimise analysis.

While the depth scaling relation presented in this chapter is an original contribution of the thesis author and collaborators \citep{my-fromage}, the width scaling relation was first derived separately via the \textit{tensor programs} framework \citep{yang2021tuning}. That framework amounts to a perturbation analysis of neural networks with random weights operating in the asymptotic limit of infinite width. The width scaling relation in this chapter resulted from discussions with Greg Yang about how to reconcile the tensor programs framework with the non-asymptotic, non-random framework presented in this thesis. The experiments in Figures \ref{fig:width-depth} and \ref{fig:width-depth-lesion} were run as part of that collaboration.

\section{The deep linear network}

This chapter deals mainly with deep networks with identity nonlinearity $\phi = \Id$.
\begin{definition}[Deep linear network]\label{def:dln}
A \textit{deep linear network} $f$ of depth $L$ maps an input $x\in\R^{d_0}$ to an output $f(x;w) \in \R^{d_L}$ via $L$ matrix multiplications:
\begin{equation}\label{eq:dln}
f(x; w) \coloneqq W_L W_{L - 1} \dots W_1 \cdot x.
\end{equation}
In this expression, $w$ denotes the tuple $w = (W_1,...,W_L)$ for $W_l\in \R^{d_l\times d_{l-1}}$.
\end{definition}
This choice of \textit{non-nonlinearity} greatly simplifies the analysis. Although architectural perturbation bounds have been obtained for more general nonlinearities \citep{my-fromage}, the results obtained for deep linear networks in this chapter were already found to transfer experimentally to deep relu networks (Figure \ref{fig:width-depth}). It is important to note that the optimisation landscape of the deep linear network is still non-convex due to the product of weight matrices.

The following definition will aid the analysis of deep linear networks:
\begin{definition}[Output scale of the deep linear network]\label{def:output-scale}
The \textit{output scale} of a deep linear network $f$ with weight matrices $w = (W_1,...,W_L)$ is given by:
\begin{equation}
F(w) \coloneqq \sqrt{d_0}\cdot\prod_{l=1}^L \norm{W_l}_*.
\end{equation}
\end{definition}

The output scale provides a simple bound on the magnitude of network outputs:
\begin{lemma}[Output bound]
\label{lem:lipschitz} For a deep linear network $f(\cdot,w)$ and all hyperspherically constrained inputs $x\in\sqrt{d_0}\cdot\Sph^{d_0-1}$, the output magnitude obeys:
\begin{equation}
    \Vert f(x; w) \Vert_2 \leq F(w).
\end{equation}
\end{lemma}
\begin{proof} Recursively extract operator norms from the matrix product:
    \begin{align*}
        \norm{f(x;w)}_2 &= \norm{W_L W_{L - 1} \dots W_1 \cdot x}_2 \\&\leq \norm{W_L}_* \cdot \norm{W_{L - 1} \dots W_1 \cdot x}_2 
        \leq ... \leq \prod_{l=1}^L \norm{W_l}_* \cdot \norm{x}_2.
    \end{align*}
    Finally, substituting $\norm{x}_2 = \sqrt{d_0}$ completes the proof.
\end{proof}

\section{Architectural perturbation bounds for deep linear networks}

The deep linear network admits the following architectural perturbation bounds:
\begin{lemma}[Architectural perturbation bounds for the deep linear network]
\label{lem:deep_perturbation_bounds} Consider perturbing the weights of a deep linear network $f:\R^{d_0}\times\mathcal{W}\to\R$ from $w=(W_1,...,W_L)\in\mathcal{W}$ to $w+\Delta w=(W_1+\Delta W_1,...,W_L+\Delta W_L)\in\mathcal{W}$. For any collection of $m$ inputs $X\in (\sqrt{d_0}\cdot\Sph^{d_0-1})^m$, the following bounds hold:
\begin{align}
    \frac{\norm{\Delta f_X}_2}{\sqrt{m}} &\leq F(w) \cdot \left[ \prod_{l = 1}^L \left( 1 + \frac{\Vert \Delta W_l \Vert_{*}}{\Vert W_l \Vert_{*}}\right)  - 1 \right];\label{eq:drt}\\
    \frac{\norm{\Delta f_X - \nabla_w f_X \Delta w}_2}{\sqrt{m}} &\leq F(w) \cdot \left[ \prod_{l = 1}^L \left( 1 + \frac{\Vert \Delta W_l \Vert_{*}}{\Vert W_l \Vert_{*}}\right)  - 1 - \sum_{l = 1}^L \frac{\Vert \Delta W_l \Vert_{*}}{\Vert W_l \Vert_{*}} \right].\label{eq:jrt}
\end{align}
\end{lemma}
These architectural perturbation bounds involve a product over the relative size of the perturbation to each network layer, reflecting the product structure of the network itself. At large depth, the bounds are roughly exponential in the layerwise perturbation size.

While the best way to understand the results is by working out a few examples (for two, three and four layer networks, say) a formal proof is now given.

\begin{proof}[Proof of Lemma \ref{lem:deep_perturbation_bounds}] The result is shown by induction over depth $L$ for a network with multiple outputs $d_L>1$ and a single input $X=\{x\}$. The result for multiple inputs $m>1$ and a single output $d_L=1$ is then immediate.

For the base case $L=1$, the relevant network to consider is given by $f(x;w) = W_1\cdot x.$ Observe that $\nabla_w f_{\{x\}} \Delta w = \Delta W_1 \cdot x$, and also:
\begin{align*}
    \Delta f_{\{x\}} &\coloneqq f(x;w+\Delta w) - f(x;w) = (W_1 + \Delta W_1)\cdot x - W_1\cdot x = \Delta W_1 \cdot x.
\end{align*}
The base case is established by noting $\norm{\Delta f_{\{x\}} - \nabla_w f_{\{x\}} \Delta w}_2=0$, and also:
\begin{align*}
\norm{\Delta f_{\{x\}}}_2 &\leq \norm{\Delta W_1}_* \cdot \norm{x}_2 = \norm{\Delta W_1}_* \cdot \sqrt{d_0}.
\end{align*}

For the inductive step, the relevant network is given by $f(x;w) = W_L\mydots W_1\cdot x$. To tackle Inequality \ref{eq:drt}, observe that:
\begin{align*}
    &\norm{\Delta f_{\{x\}}}_2 \coloneqq \norm{f(x;w+\Delta w) - f(x;w)}_2\\
    & = \norm{(W_L+\Delta W_L) \mydots (W_1+\Delta W_1)\cdot x - W_L \mydots W_1\cdot x}_2\\
    & = \|(W_L+\Delta W_l)\cdot  \left[(W_{L-1}+\Delta W_{L-1}) \mydots (W_1+\Delta W_1)\cdot x - W_{L-1} \mydots W_1\cdot x\right] \\
    & \qquad + \Delta W_L W_{L-1} \mydots W_1\cdot x \|_2 \\
    &\leq (\norm{W_L}_* + \norm{\Delta W_L}_*)\cdot\norm{(W_{L-1}+\Delta W_{L-1}) \mydots (W_1+\Delta W_1) x - W_{L-1} \mydots W_1x}_2 \\
    & \qquad + \norm{\Delta W_L}_*\cdot \norm{W_{L-1}}_* \cdot\mydots\cdot \norm{W_1}_*\cdot\sqrt{d_0},
\end{align*}
where the last line follows by several applications of the triangle inequality and the operator norm bound. Then, by the inductive hypothesis:
\begin{align*}
    &\norm{\Delta f_{\{x\}}}_2 \leq  (\norm{W_L}_* + \norm{\Delta W_L}_*) \cdot \frac{F(w)}{\norm{W_L}_*} \cdot \left[ \prod_{l=1}^{L-1} \left( 1 + \frac{\Vert \Delta W_l \Vert_{*}}{\Vert W_l \Vert_{*}}\right)  - 1 \right]\\
    &\qquad\qquad\qquad\qquad+ \norm{\Delta W_L}_* \cdot \norm{W_{L-1}}_* \cdot\mydots\cdot \norm{W_1}_*\cdot \sqrt{d_0} \\
    &\qquad= \left(1 + \frac{\norm{\Delta W_L}_*}{\norm{W_L}_*}\right) \cdot F(w) \cdot \left[ \prod_{l=1}^{L-1} \left( 1 + \frac{\Vert \Delta W_l \Vert_{*}}{\Vert W_l \Vert_{*}}\right)  - 1 \right]+ \frac{\norm{\Delta W_L}_*}{\norm{W_L}_*}\cdot F(w)\\
    &\qquad= F(w)\cdot \left[ \prod_{l=1}^{L} \left( 1 + \frac{\Vert \Delta W_l \Vert_{*}}{\Vert W_l \Vert_{*}}\right)  - 1 \right],
\end{align*}
which establishes Inequality \ref{eq:drt}. Next, to tackle Inequality \ref{eq:jrt}, observe that:
\begin{align*}
    &\norm{\Delta f_{\{x\}} - \nabla_w f_{\{x\}}\Delta w}_2 \coloneqq \norm{f(x;w+\Delta w) - f(x;w)-\nabla_w f_{\{x\}}\Delta w}_2\\
    & = \big\|(W_L+\Delta W_L) \mydots (W_1+\Delta W_1)\cdot x - W_L \mydots W_1\cdot x \\
    &\qquad - \sum_{l=1}^L W_L\mydots W_{l+1} \Delta W_l W_{l-1}\mydots W_1 \cdot x \big\|_2 \\
    & = \big\|(W_L+\Delta W_L)\cdot\big[ (W_{L-1}+\Delta W_{L-1}) \mydots (W_1+\Delta W_1)\cdot x - W_{L-1} \mydots W_1\cdot x \\
    & \qquad - \sum_{l=1}^{L-1} W_{L-1}\mydots W_{l+1} \Delta W_l W_{l-1}\mydots W_1 \cdot x \big] \\ &\qquad +\Delta W_L \sum_{l=1}^{L-1} W_{L-1}\mydots W_{l+1} \Delta W_l W_{l-1}\mydots W_1 \cdot x \big\|_2 \\
    & \leq (\norm{W_L}_*+\norm{\Delta W_L}_*)\cdot\big\|(W_{L-1}+\Delta W_{L-1}) \mydots (W_1+\Delta W_1)x - W_{L-1} \mydots W_1 x \\
    & \qquad - \sum_{l=1}^{L-1} W_{L-1}\mydots W_{l+1} \Delta W_l W_{l-1}\mydots W_1 \cdot x \big\|_2 \\ &\qquad +\norm{\Delta W_L}_* \cdot \sum_{l=1}^{L-1} \norm{W_{L-1}}_*\mydots \norm{W_{l+1}}_* \norm{\Delta W_l}_* \norm{W_{l-1}}_*\mydots \norm{W_1}_* \cdot \sqrt{d_0},
\end{align*}
where the last line follows by several applications of the triangle inequality and the operator norm bound. Then, by the inductive hypothesis:
\begin{align*}
&\norm{\Delta f_{\{x\}} - \nabla_w f_{\{x\}}\Delta w}_2\\
& \leq (\norm{W_L}_*+\norm{\Delta W_L}_*)\cdot
    \frac{F(w)}{\norm{W_L}_*} \cdot \left[ \prod_{l=1}^{L-1} \left( 1 + \frac{\Vert \Delta W_l \Vert_{*}}{\Vert W_l \Vert_{*}}\right)  - 1 - \sum_{l=1}^{L-1} \frac{\Vert \Delta W_l \Vert_{*}}{\Vert W_l \Vert_{*}} \right] \\
&\qquad+ \norm{\Delta W_L}_* \cdot \frac{F(w)}{\norm{W_L}_*} \cdot \sum_{l=1}^{L-1}\frac{\norm{\Delta W_l}_*}{\norm{W_l}_*} \\
& = \left(1+\frac{\norm{\Delta W_L}_*}{\norm{W_L}_*}\right)\cdot
    F(w) \cdot \left[ \prod_{l=1}^{L-1} \left( 1 + \frac{\Vert \Delta W_l \Vert_{*}}{\Vert W_l \Vert_{*}}\right)  - 1 - \sum_{l=1}^{L-1} \frac{\Vert \Delta W_l \Vert_{*}}{\Vert W_l \Vert_{*}} \right] \\
&\qquad+ \frac{\norm{\Delta W_L}_*}{\norm{W_L}_*} \cdot F(w) \cdot \sum_{l=1}^{L-1}\frac{\norm{\Delta W_l}_*}{\norm{W_l}_*} \\
&= F(w) \cdot \left[ \prod_{l=1}^L \left( 1 + \frac{\Vert \Delta W_l \Vert_{*}}{\Vert W_l \Vert_{*}}\right)  - 1 - \sum_{l=1}^L \frac{\Vert \Delta W_l \Vert_{*}}{\Vert W_l \Vert_{*}} \right],
\end{align*}
which establishes Inequality \ref{eq:jrt} and completes the proof.
\end{proof}

\section{Majorise-minimise for deep linear networks}

This section converts the architectural perturbation bounds of Lemma \ref{lem:deep_perturbation_bounds} into an optimisation algorithm. The algorithm is \textit{architecture aware} in the sense that it automatically accounts for details of the network architecture such as the scale of the weights, the number of layers and the width of each layer.

Solving the full majorise-minimise problem obtained via the architectural perturbation bounds of Lemma \ref{lem:deep_perturbation_bounds} is challenging, due to the many degrees of freedom in how perturbation strengths could be assigned to different layers. To simplify matters, a restricted solution is presented under the following ansatz:

\begin{ansatz}[Equal layerwise updates]\label{ansatz} For some $\eta>0$ that is independent of layer, the perturbation $\Delta W_l$ to the weight matrix $W_l$ at layer $l$ is given by:
\begin{equation}
    \Delta W_l = - \eta \cdot \frac{1}{L} \cdot \norm{W_l}_* \cdot \frac{ \nabla_{W_l} \el(w)}{\norm{ \nabla_{W_l} \el(w)}_*}.
\end{equation}
\end{ansatz}
The content of this ansatz is that across all layers $l=1,...,L$, the perturbation $\Delta W_l$ is aligned with the negative gradient and has relative magnitude $\norm{\Delta W_l}_*/\norm{W_l}_*  = \eta/L$ independent of layer. The factor of $1/L$ is only included for later convenience---it could just as well be folded into the factor of $\eta$. Under Ansatz \ref{ansatz}, the majorise-minimise problem is reduced to solving for a single variable $\eta>0$. The architectural perturbation bounds simplify as follows:
\begin{lemma}[Architectural perturbation bounds under equal layerwise updates]\label{lem:arch-perturb-ansatz}
Consider perturbing the weights of a deep linear network $f:\R^{d_0}\times\mathcal{W}\to\R$ from $w=(W_1,...,W_L)\in\mathcal{W}$ to $w+\Delta w=(W_1+\Delta W_1,...,W_L+\Delta W_L)\in\mathcal{W}$. For any collection of $m$ inputs $X\in (\sqrt{d_0}\cdot\Sph^{d_0-1})^m$, under Ansatz \ref{ansatz}:
    \begin{align}
        \norm{\Delta f_X}_2 &\leq \sqrt{m} \cdot F(w) \cdot [\exp \eta - 1];\label{eq:drt-ansatz}\\
        \norm{\Delta f_X - \nabla_w f_X \Delta w}_2 &\leq \sqrt{m} \cdot F(w) \cdot [\exp \eta - \eta - 1].\label{eq:jrt-ansatz}
    \end{align}
\end{lemma}
\begin{proof}
Under Ansatz \ref{ansatz}:
\begin{align*}
    \prod_{l=1}^L \left( 1 + \frac{\Vert \Delta W_l \Vert_{*}}{\Vert W_l \Vert_{*}}\right) &= \left( 1 + \frac{\eta}{L}\right)^L \leq \lim_{L\to\infty} \left( 1 + \frac{\eta}{L}\right)^L = \exp \eta,\\
    \sum_{l=1}^L \frac{\Vert \Delta W_l \Vert_{*}}{\Vert W_l \Vert_{*}} &= L\cdot \frac{\eta}{L} = \eta.
\end{align*}
Substituting these relations into Lemma \ref{lem:deep_perturbation_bounds} yields the results.
\end{proof}

These architectural perturbation bounds may be combined with the functional majorisation of square loss (Lemma \ref{lem:sq-major}) to obtain:

\begin{lemma}[Majorisation of square loss under equal layerwise updates]\label{lem:sq-major-nn} Under Ansatz \ref{ansatz}, the square loss of a deep linear network with $m$ training inputs $X\in (\sqrt{d_0}\cdot\Sph^{d_0-1})^m$ and corresponding label vector $Y\in \R^m$ satisfies:
    \begin{align}
        &\el_2(w+\Delta w) -\left[ \el_2(w) + \nabla_w\el_2(w)^\top \Delta w\right]\nonumber \\
        &\qquad\qquad\leq \frac{1}{2} \cdot F(w)\cdot\left(F(w) + \frac{\norm{Y}_2}{\sqrt{m}}\right)\cdot[ \exp(2\eta) -2 \eta - 1 ].
    \end{align}
\end{lemma}
\begin{proof} Substituting the architectural perturbation bounds from Lemma \ref{lem:arch-perturb-ansatz} into Lemma \ref{lem:sq-major} yields:
\begin{align*}
    &\el_2(w+\Delta w) -\left[ \el_2(w) + \nabla_w\el_2(w)^\top \Delta w\right]\nonumber \\
        &\qquad\qquad\leq \frac{\norm{f_X-Y}_2}{\sqrt{m}}\cdot F(w)\cdot [\exp \eta - \eta - 1] + \frac{1}{2}\cdot F(w)^2 \cdot [\exp \eta - 1]^2.
\end{align*}
To simplify this expression, one can observe that:
\begin{align*}
    \norm{f_X-Y}_2 \leq \norm{f_X}_2 + \norm{Y}_2 \leq  \sqrt{m}\cdot F(w) + \norm{Y}_2,
\end{align*}
where the last inequality follows from Lemma \ref{lem:lipschitz}. When combined with the relaxation that $F(W)^2 \leq F(w) \cdot (F(w) + \norm{Y}_2/\sqrt{m})$, the result is obtained.
\end{proof}

With the majorisation of Lemma \ref{lem:sq-major-nn} in hand, the majorise-minimise principle may be applied as follows:

\begin{theorem}[Log learning rates]\label{thm:log-lr} Lemma \ref{lem:arch-perturb-ansatz}'s majorisation of square loss for deep linear networks under Ansatz \ref{ansatz} is minimised by setting $\eta$ to:
\begin{equation}\label{eq:eta_star}
    \eta_\star \coloneqq \frac{1}{2} \log\left(1 + \frac{1}{F(w)\left(F(w) + \frac{\norm{Y}_2}{\sqrt{m}}\right)}\cdot \frac{1}{L}\sum_{l=1}^L \norm{W_l}_* \frac{\norm{ \nabla_{W_l} \el_2(w)}_F^2}{\norm{ \nabla_{W_l} \el_2(w)}_*}\right).
\end{equation}
\end{theorem}
\begin{proof} 
    Under Ansatz \ref{ansatz}, the first-order Taylor expansion of square loss is:
    \begin{align*}
        \el_2^{(1)}(w+\Delta w) &\coloneqq \el_2(w) + \sum_{l=1}^L \nabla_{W_l}\el_2(w)^\top \Delta W_l \\
        &= \el_2(w) - \frac{\eta}{L}\sum_{l=1}^L \norm{W_l}_* \cdot \frac{\norm{ \nabla_{W_l} \el_2(w)}_F^2}{\norm{ \nabla_{W_l} \el_2(w)}_*}.
    \end{align*}
    Substituting this form of the first-order Taylor expansion into the majorisation of Lemma \ref{lem:sq-major-nn} implies that $\el_2(w+\Delta w)$ is upper bounded by:
    \begin{equation*}
        \el_2(w) - \frac{\eta}{L}\sum_{l=1}^L \norm{W_l}_* \frac{\norm{ \nabla_{W_l} \el_2(w)}_F^2}{\norm{ \nabla_{W_l} \el_2(w)}_*}  + \frac{F(w)}{2}\left(F(w) + \frac{\norm{Y}_2}{\sqrt{m}}\right)[ \exp(2\eta) -2 \eta - 1 ].
    \end{equation*}
    Setting the derivative of this expression with respect to $\eta$ to zero yields:
    \begin{equation*}
        \frac{1}{L}\sum_{l=1}^L \norm{W_l}_* \frac{\norm{ \nabla_{W_l} \el_2(w)}_F^2}{\norm{ \nabla_{W_l} \el_2(w)}_*}= F(w)\left(F(w) + \frac{\norm{Y}_2}{\sqrt{m}}\right)[\exp(2\eta) -1 ].
    \end{equation*}
    Finally, solving for $\eta$ yields the result.
\end{proof}
Theorem \ref{thm:log-lr} was derived in close collaboration with Kevin Huang. In short, the theorem suggests a learning rule where layer $l$ is perturbed via:
\begin{equation}\label{eq:opt1}
    W_l \mapsto W_l - \eta_\star \cdot \frac{1}{L} \cdot \norm{W_l}_* \cdot \frac{ \nabla_{W_l} \el(w)}{\norm{ \nabla_{W_l} \el(w)}_*},
\end{equation}
with $\eta_\star$ given by Equation \ref{eq:eta_star}. A curious aspect of this update is that the scale of the gradient only enters logarithmically through the $\eta_\star$ term. This may explain why popular neural net optimisers, such as \textit{Adam} \citep{kingma_adam:_2015}, more-or-less completely remove the gradient scale from their update.

Another feature of Update \ref{eq:opt1} is that explicit dependence on both the network depth $L$ and the scale of the weight matrices $\norm{W_l}_\star$ are encoded. But, as of yet, there is no explicit dependence on the network width. This omission is rectified in the next subsection.

\subsection{Width scaling}

\citet{my-fromage} assumed---without real evidence---that the weight matrices and gradients of a deep network have roughly the same conditioning. In turn, this meant that the architecture aware optimisation method developed in their paper does not scale properly with network width. A better conditioning assumption was employed in a paper by \citet{yang2021tuning}:

\begin{assumption}[Weight matrix and gradient conditioning]\label{ass:cond} For all $l=1,...,L$:
\begin{align}
    \norm{W_l}_* &= \frac{\norm{W_l}_F}{\sqrt{\min(d_l,d_{l-1})}}; \label{eq:w-cond} \\ \norm{\nabla_{W_l}\el(w)}_* &= \norm{\nabla_{W_l}\el(w)}_F.\label{eq:g-cond}
\end{align}
\end{assumption}
To understand this assumption, one needs to be familiar with the following aspect of matrix conditioning. Given a matrix $A\in\R^{d_l\times d_{l-1}}$ with $\overline{d} \coloneqq \min(d_l,d_{l-1})$ singular values denoted $\sigma_1, ..., \sigma_{\overline{d}}$, the norms of $A$ satisfy:
\begin{equation}\label{eq:frob-vs-op}
    \norm{A}_F^2 = \sum_{i=1}^{\overline{d}} \sigma_i^2 \geq \max_{i\in\{1,...,\overline{d}\}} \sigma_i^2 = \norm{A}_*^2.
\end{equation}
This means that $\norm{A}_F/\sqrt{\min(d_l,d_{l-1})}$ reports the root-mean-square singular value of $A$, while the operator norm $\norm{A}_*$ reports the largest singular value. 

Under this interpretation of matrix norms, Equation \ref{eq:w-cond} is stating that the largest singular value of weight matrix $W_l$ is equal to the average singular value---meaning that $W_l$ is \textit{well-conditioned}. The justification for this assumption is that the weight matrices in a deep network are typically initialised randomly, and random matrices are fairly well-conditioned.

On the other hand, Equation \ref{eq:g-cond} states that the operator norm and Frobenius norm of the gradient $\nabla_{W_l}\el(w)$ at layer $l$ are equal. By Equation \ref{eq:frob-vs-op}, this happens when $\nabla_{W_l}\el(w)$ has only one non-zero singular value---meaning that the gradient is very \textit{low rank}. The justification for this assumption is that the gradient is, in a sense, an \textit{optimal} object: it reports the perturbation direction that elicits the largest change in loss. This makes it reasonable that the gradient would not spread itself too thin in the sense of rank. This intuition is simple to prove for the gradient of the loss over a single training input $x$, which may be written directly as the rank-one outer product $\nabla_{W_l} \el = \nabla_{f_l(x)}\el\otimes \phi(f_{l-1}(x))$ in the notation of Equation \ref{eq:nn-layer}. Rigorously extending this argument to larger numbers of training examples appears challenging.

Combining Assumption \ref{ass:cond} with Theorem \ref{thm:log-lr} leads to the following update:
\begin{theorem}[Architecture aware deep network update]\label{thm:log-lr-width}
Lemma \ref{lem:arch-perturb-ansatz}'s majorisation of square loss for deep linear networks under Ansatz \ref{ansatz} and Assumption \ref{ass:cond} is minimised by the following update:
\begin{equation}\label{eq:opt2}
    W_l \mapsto W_l - \eta_\dagger \cdot \frac{1}{L} \cdot \frac{\norm{W_l}_F}{\sqrt{\min(d_l,d_{l-1})}} \cdot \frac{ \nabla_{W_l} \el(w)}{\norm{ \nabla_{W_l} \el(w)}_F},
\end{equation}
where the learning rate $\eta_\dagger$ is given by:
\begin{equation}\label{eq:log-lr}
    \eta_\dagger \coloneqq \frac{1}{2} \log\left(1 + \frac{\frac{1}{L}\sum_{l=1}^L \frac{\norm{W_l}_F}{\sqrt{\min(d_l,d_{l-1})}} \cdot \norm{\nabla_{W_l} \el_2(w)}_F}{F(w)\left(F(w) + \frac{\norm{Y}_2}{\sqrt{m}}\right)}\right).
\end{equation}
\end{theorem}
\begin{proof}
    The result follows by substituting Assumption \ref{ass:cond} into Theorem \ref{thm:log-lr}.
\end{proof}
Update \ref{eq:opt2} explicitly depends on the width, depth and weight scale of the deep network. For this reason, Theorem \ref{thm:log-lr-width} is tagged \textit{architecture aware}. The theorem unifies various heuristic and theoretical ideas explored in the literature:
\begin{enumerate}
    \item \textit{Relative updates.} The gradient is rescaled by $\norm{W_l}_F / \norm{\nabla_{W_l}\el(w)}_F$. This means that the magnitude of the update is in proportion to the magnitude of the weight matrix to which it is applied. Such a scaling was proposed on heuristic grounds by \citet{You:EECS-2017-156} and explored theoretically by \citet{my-fromage}. It also relates to ideas explored by \citet{carbonnelle2019layer} and \citet{my-nero}.
    \item \textit{Depth scaling.} Scaling the perturbation strength like $1/L$ for networks of depth $L$ was proposed on theoretical grounds by \citet{my-fromage}.
    \item \textit{Width scaling.} Scaling the perturbation size by $\norm{W_l}_F/\sqrt{\min(d_l,d_{l-1})}$ relates to a theoretical technique proposed by \citet{yang2021tuning}.
    \item \textit{Adaptive gradient clipping.} The logarithmic dependence of the update on the gradient scale relates to a heuristic technique known as \textit{adaptive gradient clipping} \citep{pmlr-v139-brock21a} which clips the gradient once its magnitude surpasses a certain threshold.
\end{enumerate}

\section{Experimental tests with relu networks}

The performance of Update \ref{eq:opt2} was tested for training multilayer perceptrons (Definition \ref{def:mlp}) with relu nonlinearity and of varying width and depth. The performance of the update was measured as a function of $\eta_\dagger$, where $\eta_\dagger$ was held constant during training. Testing the logarithmic form of $\eta_\dagger$ (Equation \ref{eq:log-lr}) is part of ongoing research with Kevin Huang, Chris Mingard and Yisong Yue.

The networks were trained on the CIFAR-10 dataset \citep{Krizhevsky09learningmultiple}. CIFAR-10 consists of sixty thousand 32px by 32px RGB input images that each fall into one of ten classes. This means that each image is described by $32\times 32 \times 3 = 3072\eqqcolon d_0$ real numbers along with a class index in $\{1,...,10\}$. The input images were pre-processed as follows: each image was flattened to lie in $\R^{d_0}$, centred to have mean zero and then projected on to the hypersphere of radius $\sqrt{d_0}$. The nonlinearity $\phi$ was set to $\phi(\cdot)=\sqrt{2}\cdot\max(0,\cdot)$. The weight matrices were initialised iid Gaussian and scaled such that the root mean square singular value at each layer was approximately one. The loss function was set to measure the square loss between the 10-dimensional network output and a one-hot encoding of the class index. The networks were trained for 19 epochs with 1000 train images used to compute the gradient at each step.

The results are presented in Figure \ref{fig:width-depth}. As can be seen in that figure, the train loss as a function of learning rate $\eta_\dagger$ is quite stable as both the width and depth of the network are varied. For comparison, Figure \ref{fig:width-depth-lesion} displays the behaviour without explicit depth scaling, by plotting train loss as a function of $\eta_\dagger/L$. As can be seen, this causes performance to drift with depth.

\begin{figure}[p]
    \centering
    \includegraphics{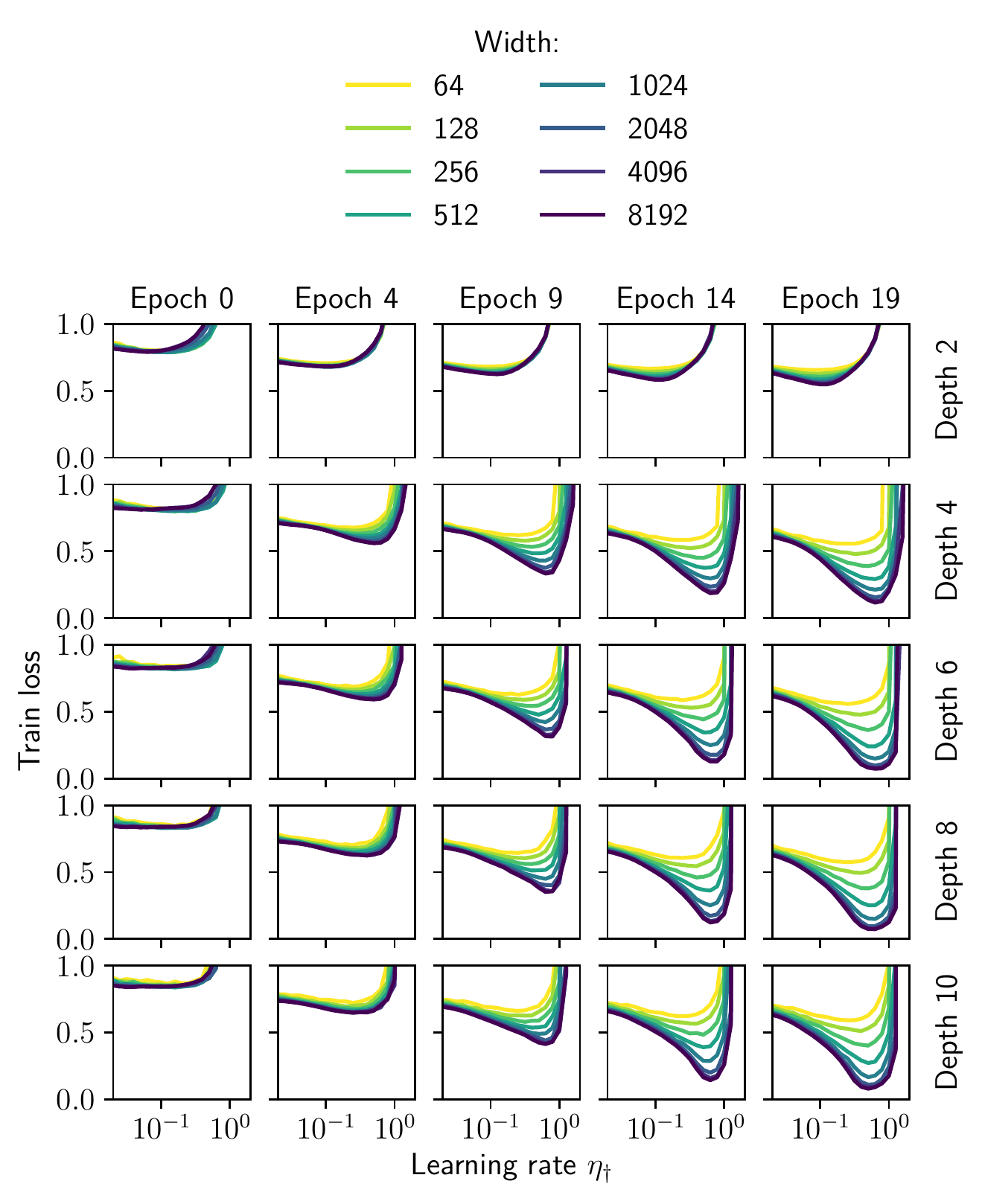}
    \caption[Learning rate transfer across width and depth]{Learning rate transfer across width and depth. Update \ref{eq:opt2} was used to train relu multilayer perceptrons of varying width and depth on the CIFAR-10 dataset. As can be seen, training performance was quite stable as a function of learning rate $\eta_\dagger$, as both width and depth were varied.}
    \label{fig:width-depth}
\end{figure}
\begin{figure}[p]
    \centering
    \includegraphics{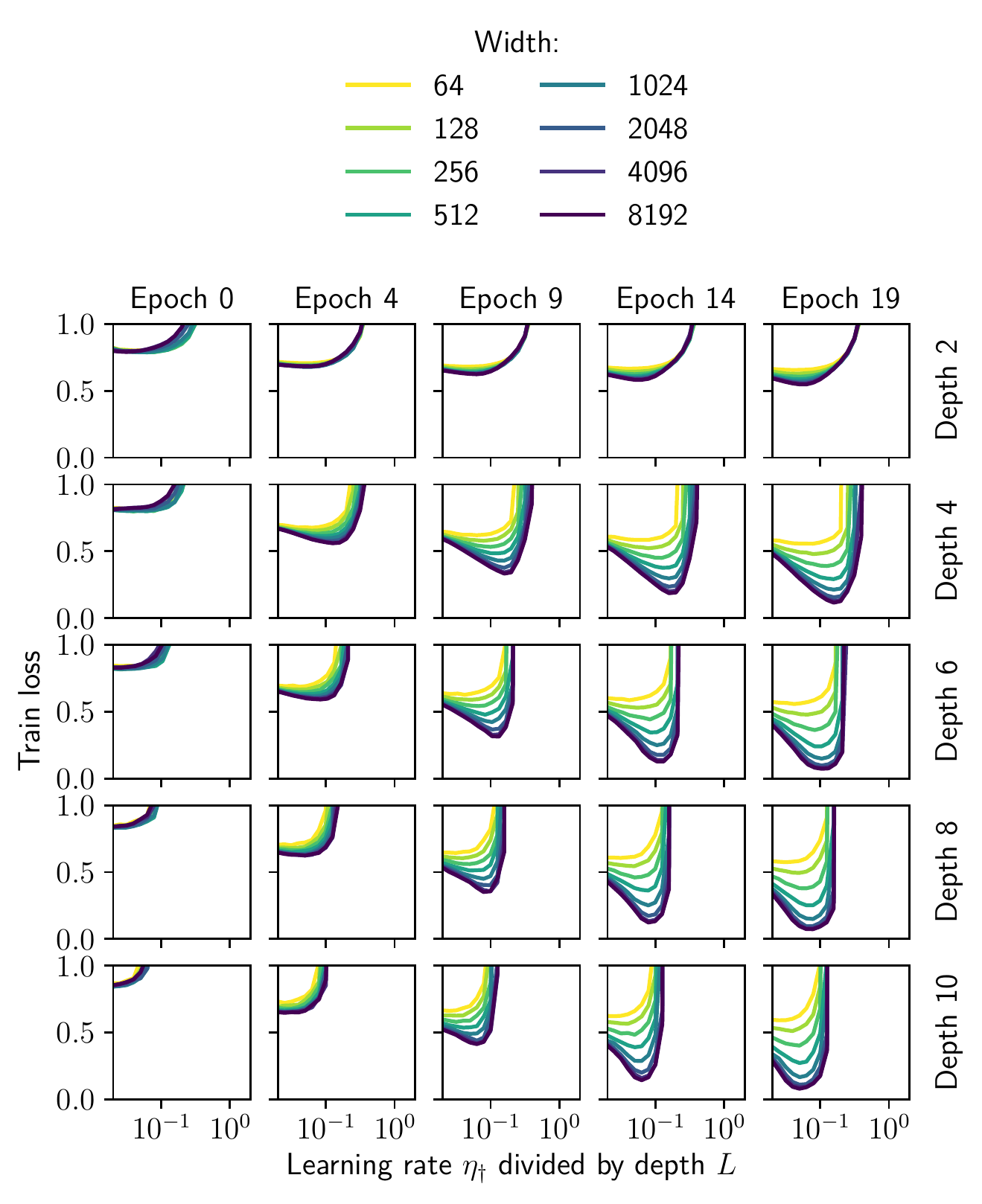}
    \caption[Learning rate transfer without explicit depth scaling]{Learning rate transfer without explicit depth scaling. The same results as Figure \ref{fig:width-depth} are plotted, except as a function of $\eta_\dagger/L$. This displays the behaviour of Update \ref{eq:opt2} with the explicit depth scaling removed. As can be seen, the tuning curves shift left for increasing depth.}
    \label{fig:width-depth-lesion}
\end{figure}

\clearpage
\printbibliography[heading=subbibliography]
\end{refsection}

%% file: chapters/7-learning-theory-frameworks.tex
\begin{refsection}

\chapter{Classic Generalisation Theories}
\label{chap:g-theory}

\begin{tcolorbox}
This chapter provides an introduction to three classic generalisation theories: uniform convergence, algorithmic stability and PAC-Bayes. The material is expositional and is included for the reader's aid.
\end{tcolorbox}

The goal of this chapter is to review three classic theoretical frameworks for reasoning about generalisation of machine learning algorithms. The intention is not to provide a comprehensive survey, but more to provide an introduction to the landscape of thought surrounding generalisation. For each framework, the chapter discusses its general relevance to neural networks, including key hurdles that may need to be overcome to make the approach workable.

The three frameworks that are considered are:
\begin{enumerate}
    \item \textit{Uniform convergence.} This framework seeks to assess the amount of training data needed such that train error converges to test error for all functions within a function space. To achieve this, the techniques tend to rely on limiting the size of the function space in some way.
    \item \textit{Algorithmic stability.} This framework makes an important departure from uniform convergence by seeking to characterise generalisation of only those functions returned by a learning algorithm. In particular, the framework uses the insensitivity of the learning algorithm to modifications of its training set in order to guarantee generalisation.
    \item \textit{PAC-Bayes.} This framework also departs from uniform convergence by assessing generalisation only of a user-specified ensemble of functions within the function space. What's more, the guarantee is on the average generalisation error over the ensemble, meaning that a small fraction of ensemble members are allowed to generalise poorly. The guarantee is good when the \textit{posterior measure} that defines the ensemble is sufficiently close to a \textit{prior measure} chosen before seeing the training data.
\end{enumerate}

\section{Uniform convergence}

Uniform convergence is a classic framework for studying generalisation in machine learning. The name refers to the idea of test error converging to train error \textit{for all} functions in a function class. The framework dates back at least to the work of \citet{vcpaper}.

The mechanism for proving uniform convergence bounds relies on the size of a function space being, in a sense, \textit{small} compared to the number of training examples. Since function spaces of interest may naïvely contain infinitely many functions, one needs to be careful about how one counts functions.

In the case of binary classification, a suitable means of counting functions is to count the number of distinct binary labellings that the function space is able to assign to a finite set of inputs. This is known as the \textit{shattering coefficient}, which given the notion of \textit{function projection} (Definition \ref{def:project}), is easy to define:

\begin{definition}[Shattering coefficient] Given a machine learning model $f:\mathcal{X}\times \mathcal{W}\to\R$, the \textit{shattering coefficient} $\mathcal{N}(f,m)$ is the largest number of binary labellings that $f$ can realise on any set of $m$ training inputs $X\in\mathcal{X}^m$:
\begin{equation}
    \mathcal{N}(f,m) \coloneqq \max_{X\in\mathcal{X}^m} \,\# \{\sign f_X(w) \mid w\in\mathcal{W}\}.
\end{equation}
\end{definition}

Since there are only $2^m$ possible binary labellings of $m$ points, it holds that $\mathcal{N}(f,m) \leq 2^m$. A VC (Vapnik-Chervonenkis) bound relies on the machine learning model $f$ being so limited in its capacity that $\mathcal{N}(f,m) \ll 2^m$:

\begin{theorem}[VC bound]\label{thm:vc} Consider a machine learning model $f:\mathcal{X}\times \mathcal{W}\to\R$ that makes classifications by binarising its output $f \mapsto \sign f$. For a training set $S \in (\mathcal{X} \times \{\pm1\})^m$ drawn iid from a data distribution $\mathcal{D}$, let $\el_{0/1}^S(w)$ denote the zero-one loss of weight vector $w\in\mathcal{W}$ with respect to training set $S$, and define $\el_{0/1}^\mathcal{D}(w)\coloneqq \Expect_{x,y\sim\mathcal{D}}\mathbb{I}[\sign f(x,w) \neq y]$ to be the average misclassification rate over $\mathcal{D}$. Then for a fraction $1-\delta$ of such training sets drawn in this way:
\begin{equation}
\text{For all } w\in\mathcal{W},\quad
    \el_{0/1}^\mathcal{D}(w) \leq \el_{0/1}^S(w) + \sqrt{\frac{8}{m}\left(\log  \mathcal{N}(f,2m)+\log \frac{4}{\delta}\right)}.
\end{equation}
\end{theorem}

This example of a VC bound is taken from the textbook of \citet{lwk}. A core feature of this bound is that the amount by which the train error can exceed the test error is bounded \textit{for all} weight vectors. Provided that the shattering coefficient saturates at some number of training examples $m$, then the bound is $\mathcal{O}(1/\sqrt{m})$.

\subsection{Applicability to neural networks} 
The relevance of a result like Theorem \ref{thm:vc} for deep learning practice is unclear. Often, neural networks operate in a regime where they can realise essentially any labelling of their training set \citep{Zhang2017UnderstandingDL}. This means that the shattering coefficient $\mathcal{N}(f,m) \approx 2^m$. In turn, the complexity term $\log \mathcal{N}(f,2m) / m$ appearing in Theorem \ref{thm:vc} does not decay as $m$ increases.

One could resolve this problem by restricting the neural network function space, excluding the kinds of highly expressive function that allow the shattering coefficient $\mathcal{N}(f,2m)$ to approach $2^m$. It is not obvious how to accomplish this, although there have been proposals. For instance, \citet{specnorm} propose an approach based on restricting to neural network functions with bounded layerwise spectral norms. Such approaches have not yet led to non-vacuous generalisation guarantees for deep neural networks.

\section{Algorithmic stability}

The distinguishing feature of a uniform convergence bound is that the bound holds for all functions in the function space. Usually, in practice, one is interested in the generalisation performance of only a single function---the one that is returned by the learning algorithm. If the learning algorithm has special properties, then perhaps the function that it returns could generalise significantly better than other functions in the function space. This motivates the framework of \textit{algorithmic stability} \citep{stability}. To introduce this framework, it first helps to formally define the learning algorithm:

\begin{definition}[Learning algorithm]\label{def:algorithm} A \textit{learning algorithm} $A$ is a deterministic mapping $S \mapsto A(S)$ from a training set $S$ to a weight vector $A(S)$.
\end{definition}

In practice, a learning algorithm need not be deterministic. For instance, the order in which the training set is fed into the algorithm will often affect the returned weight vector. But Definition \ref{def:algorithm} glosses over this detail.

Algorithmic stability provides generalisation guarantees for learning algorithms that obey some kind of stability property. The following definition establishes a notion of stability for learning algorithms with respect to manipulating a training set $S$ by removing the $i$th example $S\mapsto S^{\backslash i}$:

\begin{definition}[Uniform stability]\label{def:stability} Let $\ell(w,z)\in\R$ denote the loss incurred by weight vector $w\in \mathcal{W}$ on a single training example $z\in(\mathcal{X}\times\mathcal{Y})$. A learning algorithm $A$ has \textit{uniform stability} $\beta$ with respect to loss $\ell$ if the following holds: for all training sets $S\in(\mathcal{X}\times\mathcal{Y})^m$ and training examples $i\in\{1,...,m\}$, 
\begin{equation}
    \sup_{z\in \mathcal{X}\times\mathcal{Y}} \left|\ell\big(A(S),z\big) - \ell\big(A(S^{\backslash i}),z\big)\right|<\beta. 
\end{equation}
\end{definition}
In words: if for any training set, removing any training example leads to a small change in loss on any test point, then the learning algorithm is uniformly stable. This definition leads to the following generalisation bound for regression:

\begin{theorem}[Uniform stability bound]\label{thm:stability} Let $A$ be a learning algorithm with uniform stability $\beta$ with respect to a bounded loss $0\leq\ell(\cdot,\cdot)\leq 1$. For a data distribution $\mathcal{D}$, define the test loss $\el_\mathcal{D}(w)\coloneqq \Expect_{z\sim\mathcal{D}}\ell(w,z)$ and train loss $\el_S(w)\coloneqq \Expect_{z\sim\uniform(S)}\ell(w,z)$. Then, for a fraction $1-\delta$ of training sets $S\overset{\text{iid}}{\sim}\mathcal{D}^m$:
\begin{equation}\label{eq:stability}
    \el_\mathcal{D}(A(S)) \leq \el_S(A(S)) + 2\beta + (4m\beta + 1)\cdot\sqrt{\frac{\log 1/\delta}{2m}}.
\end{equation}
\end{theorem}

This theorem is due to \citet{stability}. The theorem is most interesting for a learning algorithm with uniform stability $\beta = \mathcal{O}(1/m)$, for which the complexity term in Equation \ref{eq:stability} decays like $\mathcal{O}(1/\sqrt{m})$, similar to the VC bound of Theorem \ref{thm:vc}.

\subsection{Applicability to neural networks}
It is not obvious whether neural networks satisfy a stability condition such as Definition \ref{def:stability} at a non-trivial level $\beta$. Sometimes neural networks are seemingly able to interpolate any training set \citep{Zhang2017UnderstandingDL}, meaning that they are highly sensitive to the inclusion or exclusion of any particular data point.

To make this framework workable for neural networks, one would need to produce a stability condition that is satisfied by a neural network. This could potentially be a weaker notion of stability than uniform stability, and some alternatives are proposed in the paper by \citet{stability}.

\section{PAC-Bayes}

PAC-Bayes theory presents another approach to introducing algorithm dependence into a generalisation bound. The theory makes a conceptual shift from the frameworks of both uniform convergence as well as algorithmic stability. It is worth introducing this new perspective in three steps:

\begin{enumerate}[label=Step \arabic*:, leftmargin=*, font=\sffamily]
    \item \textit{Specify prior belief.} Before seeing the training data, assign a prior measure $P$ to the function space. The idea is that functions that are believed more likely to explain the training data (upon its arrival) should be assigned higher probability under this prior measure.
    \item \textit{Given data, construct a posterior.} Once the training data has arrived, a posterior measure $Q$ should be constructed that assigns higher probability to functions that now seem more likely. The posterior is, in a sense, \textit{algorithm dependent}: the choice of which functions to include in the posterior constitutes the learning algorithm. Although it can be, the posterior $Q$ need not be set to the Bayesian posterior for prior $P$.
    \item \textit{Measure distance between prior and posterior.} According to PAC-Bayes theory, the closer the posterior $Q$ is to the prior $P$, the better the functions in the posterior $Q$ will generalise on average. The idea is that if the posterior and prior are very similar, then less information was extracted from the training data. To quantify this, a distance on probability measures is required, such as the KL divergence $\kl(Q||P)$.
\end{enumerate}

These three steps underlie the following theorem of \citet{Langford01boundsfor}:

\begin{theorem}[PAC-Bayes bound]\label{thm:pac-bayes}
Let $P$ be a prior over functions realised by a classifier and let $\mathcal{D}$ denote the data distribution. For a fraction $1-\delta$ of training sets $S\overset{\text{iid}}{\sim}\mathcal{D}^m$, the following holds for all posterior distributions $Q$:
\begin{equation}
    \kl(\bern_{S,Q}||\bern_{\mathcal{D},Q}) \leq \frac{\kl(Q||P) + \log (2m/\delta)}{m-1},
\end{equation}
where $\bern_{S,Q}$ is Bernoulli with probability $\Probe_{w\sim Q, (x,y)\sim\uniform(S)}[\sign f(x;w)\neq y]$ and $\bern_{\mathcal{D},Q}$ is Bernoulli with probability $\Probe_{w\sim Q, (x,y)\sim\mathcal{D}}[\sign f(x;w)\neq y]$.
\end{theorem}
In words: the KL divergence appearing on the left-hand side measures the distance between train error and test error averaged over the posterior. This is upper bounded by the KL divergence between prior and posterior (plus a logarithmic confidence term) divided by the number of training examples. 

The following corollary may be slightly easier to parse than Theorem \ref{thm:pac-bayes}. It specialises to the \textit{realisable} setting of functions that attain zero train error.

\begin{corollary}[Realisable PAC-Bayes]\label{cor:pac-bayes}
Let $P$ be a prior over functions realised by a classifier. For a fraction $1-\delta$ of training sets $S\overset{\text{iid}}{\sim}\mathcal{D}^m$, the following holds for all posterior distributions $Q$ over functions that correctly classify $S$:
\begin{equation}\label{eq:pac-bayes-corollary}
\Probe_{w\sim Q, (x,y)\sim\mathcal{D}}[\sign f(x;w)\neq y] \leq 1 - \exp \left[-\frac{\kl(Q||P) + \log (2m/\delta)}{m-1}\right].
\end{equation}
\end{corollary}
\begin{proof} For a posterior $Q$ that correctly classifies $S$, the Bernoulli random variable $\bern_{S,Q}$ is zero with probability one. In turn, this implies that:
\begin{equation*}
    \kl(\bern_{S,Q}||\bern_{\mathcal{D},Q})= - \log (1-\Probe_{w\sim Q, (x,y)\sim\mathcal{D}}[\sign f(x;w)\neq y]).
\end{equation*}
Substituting this relation into Theorem \ref{thm:pac-bayes} and rearranging yields the result.
\end{proof}

\subsection{Applicability to neural networks}

PAC-Bayes has been shown to yield non-vacuous generalisation guarantees both for neural networks \citep{DR17} and for infinitely wide neural networks \citep{Prez2020GeneralizationBF,my-bpm}. Indeed, by inspection of Corollary \ref{cor:pac-bayes}, when $\kl(Q||P)$ is finite, the right-hand side of Inequality \ref{eq:pac-bayes-corollary} is smaller than one and the bound is never vacuous.

One of the main problems with applying Theorem \ref{thm:pac-bayes} or Corollary \ref{cor:pac-bayes} to neural networks is that there is a mismatch between what is bounded and what matters in a standard machine learning problem. In particular, PAC-Bayes bounds are on the average test error over an ensemble of functions, described by posterior distribution $Q$. But in a standard machine learning scenario, one is most interested in the generalisation error of an individual function. Steps towards resolving this issue will be taken in Chapter \ref{chap:bpm}.

\printbibliography[heading=subbibliography]
\end{refsection}

%% file: chapters/8-pac-bayes-for-gps.tex
\begin{refsection}

\chapter{PAC-Bayes for Gaussian Processes}
\label{chap:gp-pac-bayes}

\begin{tcolorbox}
This chapter derives PAC-Bayes risk bounds for Gaussian process classification. While the specific calculations are novel, the general framework is due to \citet{seeger}. The calculations will be useful in Chapter \ref{chap:bpm}.
\end{tcolorbox}

Gaussian processes already contain many of the perplexing phenomena present in neural networks. They represent a function space of essentially unlimited expressivity, that nevertheless tends to generalise well when fit to data.

PAC-Bayes theory applies naturally to Gaussian process classification since Gaussian processes come equipped with prior and posterior measures that can be used in the PAC-Bayes bound. \citet{seeger} developed an extensive PAC-Bayesian treatment of Gaussian process classification. Those results apply to a very flexible set of methods known as \textit{approximate Gaussian process classifiers}. Essentially, these methods use a Gaussian distribution to approximate an intractable Bayesian posterior. Since the PAC-Bayes bound applies to any posterior and not just the Bayesian posterior, PAC-Bayes theory can be straightforwardly applied to control the error of these approximate methods. 

The focus of this chapter is on developing a simple and direct PAC-Bayesian treatment of Gaussian process classification for use in Chapter \ref{chap:bpm}.

\section{Gaussian process classification}\label{sec:gpc}

There is a simple way to convert a Gaussian process into a binary classifier. Given a training sample $S=(X,Y)$, one can simply draw functions from the prior $f\sim\gp(0,k)$ until a function is found that correctly classifies the training sample: $\sign f_X = Y$. This function can then be evaluated on novel test points.

To develop a PAC-Bayesian generalisation theory of this procedure, one needs to write down the corresponding posterior and compute its KL-divergence with the prior. To ease this process, it is helpful to decompose the prior into two separate pieces, each defined implicitly via a sampling procedure:
\begin{definition}[GP binary classifier prior]\label{def:gpc-prior} To construct the \textit{prior} $\pgp(\cdot)$ of a GP binary classifier, take train inputs $X$ and any other inputs $X^\prime$ and decompose:
\begin{equation}
    \pgp(f_{X^\prime}, f_X) = \pgp(f_{X^\prime}|f_X) \cdot \pgp(f_X).
\end{equation}
Now, define each piece of this decomposition separately:
\begin{enumerate}
    \item The measure $\pgp(f_X)$ corresponds to sampling train outputs:
    \begin{equation}
        f_X \sim \normal(0,K_{XX}).
    \end{equation}
    \item The measure $\pgp(f_{X^\prime}|f_X)$ corresponds to sampling the other outputs:
    \begin{equation}
        f_{X^\prime} \sim \normal(K_{X^\prime X}K_{XX}^{-1}f_X, K_{X^\prime X^\prime}-K_{X^\prime X}K_{XX}^{-1}K_{XX^\prime}).
    \end{equation}
\end{enumerate}
\end{definition}
One might worry that this definition of the prior depends on the training sample $X$, which is expressly not allowed. But this dependence is illusory: by Theorem \ref{thm:gp-cond}, this definition is equivalent to sampling  $f_{X^\prime}, f_X \sim \normal(0,K_{X^\prime \cup X\,X^\prime\cup X})$. Breaking the sampling procedure in two like this is helpful for making a direct comparison with the following definition of the posterior:

\begin{definition}[GP binary classifier posterior]\label{def:gpc-posterior}
To construct the \textit{posterior} $\qgp(\cdot)$ of a GP binary classifier, for train inputs $X$ and other inputs $X^\prime$, decompose:
\begin{equation}
    \qgp(f_{X^\prime}, f_X) = \qgp(f_{X^\prime}|f_X) \cdot \qgp(f_X).
\end{equation}
Now, define each piece of this decomposition separately:
\begin{enumerate}
    \item The measure $\qgp(f_X)$ corresponds to sampling train outputs:
    \begin{equation}
        f_X \sim \normal(0,K_{XX}\mid\sign f_X = Y).
    \end{equation}
    \item The measure $\qgp(f_{X^\prime}|f_X) \coloneqq \pgp(f_{X^\prime}|f_X)$ from Definition \ref{def:gpc-prior}.
\end{enumerate}
\end{definition}
The posterior is identical to the prior, except that the distribution on train outputs is truncated to the orthant $\{f_X \in \R^m \mid \sign f_X = Y\}$. 

It will also turn out to be both computationally and analytically convenient to define the following approximate posterior distribution:

\begin{definition}[GP binary classifier spherised posterior]\label{def:gpc-posterior-spherised}
For train inputs $X$ and other inputs $X^\prime$, the \textit{spherised posterior} $\qsph(\cdot)$ is defined by first decomposing:
\begin{equation}
    \qsph(f_{X^\prime}, f_X) = \qsph(f_{X^\prime}|f_X) \cdot \qsph(f_X).
\end{equation}
And, next, defining each piece of this decomposition separately:
\begin{enumerate}
    \item The measure $\qsph(f_X)$ corresponds to sampling train outputs:
    \begin{equation}
        f_X \sim \normal(0,\Id\cdot \abs{K_{XX}}^{1/m}\mid\sign f_X = Y).
    \end{equation}
    \item The measure $\qsph(f_{X^\prime}|f_X) \coloneqq \pgp(f_{X^\prime}|f_X)$ from Definition \ref{def:gpc-prior}.
\end{enumerate}
\end{definition}

The spherised posterior modifies the posterior of Definition \ref{def:gpc-posterior} by replacing the truncated Gaussian distribution over train outputs $\qgp(f_X)$ with a truncated spherical Gaussian distribution $\qsph(f_X)$. The factor $\abs{K_{XX}}^{1/m}$ is included to match the variance scales of these two distributions.

\section{The KL divergence between prior and posterior}\label{sec:kl}

To construct PAC-Bayes bounds, one needs to compute the KL-divergence between these prior and posterior measures over functions. A lemma will help:
\begin{lemma}[Chain rule for the KL divergence]\label{lem:kl-chain} Let $Q$ and $P$ denote measures over functions $f:\mathcal{X}\to\R$, for a finite input space $\mathcal{X}$. For a collection of $m$ inputs $X\in\mathcal{X}^m$, let $X^\mathsf{c}$ denote all inputs excluding $X$: $X^\mathsf{c}\coloneqq \mathcal{X}\setminus X$. Then:
    \begin{align}
        \kl\big(Q(f)\;||\;P(f)\big) = \kl\big(Q(f_{X^\mathsf{c}} | f_X)\;||\;P(f_{X^\mathsf{c}}|f_X)\big) + \kl\big(Q(f_X)\;||\;P(f_X)\big).
    \end{align}
\end{lemma}
\begin{proof}
First, a full function $f \equiv f_\mathcal{X} \equiv (f_{X^\mathsf{c}},f_X)$. The result follows by substituting the chain rule for joint probability into the definition of the KL divergence, separating the logarithm, and recognising the two KL divergences:
    \begin{align*}
    &\kl(Q(f)\;||\;P(f)) = \int \diff{f}\,Q(f) \log \frac{Q(f)}{P(f)}\\
    & \qquad= \int\diff{f_X}\, Q(f_X)\int \diff{f_{X^\mathsf{c}}}\,Q(f_{X^\mathsf{c}}|f_X) \log \frac{Q(f_{X^\mathsf{c}}|f_X)\cdot Q(f_X)}{P(f_{X^\mathsf{c}}|f_X)\cdot P(f_X)}\\
    & \qquad= \int\diff{f_X}\,Q(f_X) \left[\kl(Q(f_{X^\mathsf{c}} | f_X)\;||\;P(f_{X^\mathsf{c}}| f_X)) + \log \frac{Q(f_X)}{P(f_X)}\right]\\
    & \qquad = \kl(Q(f_{X^\mathsf{c}} | f_X)\;||\;P(f_{X^\mathsf{c}}| f_X)) + \kl(Q(f_X)\;||\;P(f_X)).
    \end{align*}
    The proof is complete.
\end{proof}

Note that machine learning methods implemented on computers use finite input spaces. It is possible to generalise the lemma beyond finite input spaces, but this requires concepts from measure theory such as \textit{Radon-Nikodym derivatives} and the \textit{disintegration theorem}, which are beyond the scope of this thesis.

The relevance of this result is that, since the second sampling steps of Definitions \ref{def:gpc-prior}, \ref{def:gpc-posterior} and \ref{def:gpc-posterior-spherised} are identical, it holds that:
\begin{equation}
    \kl(\qgp(f_{X^\mathsf{c}} | f_X)\;||\;P(f_{X^\mathsf{c}}| f_X))=\kl(\qsph(f_{X^\mathsf{c}} | f_X)\;||\;P(f_{X^\mathsf{c}}| f_X)) = 0.
\end{equation}
In turn, by Lemma \ref{lem:kl-chain}, this implies that:
\begin{align}
    \kl(\qgp(f)\;||\;P(f)) &= \kl(\qgp(f_X)\;||\;P(f_X)); \\
    \kl(\qsph(f)\;||\;P(f)) &= \kl(\qsph(f_X)\;||\;P(f_X)).
\end{align}
These relations mean that, in order to evaluate the KL divergence between the full prior and posterior, one only needs to evaluate the KL divergence between the prior and posterior restricted to the training inputs. To derive these KL divergences, it is first helpful to define two quantities:

\begin{definition}[Gaussian orthant probability]\label{def:gop} Given a training set $S=(X,Y)$ with binary labels $Y\in\{\pm1\}^m$, the \textit{Gaussian orthant probability} $P_Y$ is:
\begin{equation}\label{eq:gop}
    P_Y : = \Probe_{f_X \sim P_{\mathrm{GP}}}[\sign f_X = Y].
\end{equation}
\end{definition}
This is termed an \textit{orthant probability} because for a training set $X$ consisting of $m$ examples, the set $\left\{f_X \in \R^m \;\middle|\; \sign f_X = Y\right\}$ is an orthant of $\R^m$.
\begin{definition}[Kernel complexity]\label{def:k-complex} Given a kernel $k$ and a training set $S=(X,Y)$, the \textit{kernel complexity} $\mathcal{A}(k,X,Y)$ is given by:
\begin{align}
    \mathcal{A}(k,X,Y):= m\cdot\left(\log 2 - \frac{1}{2}\right) \;+ \qquad\qquad\qquad\qquad\qquad\qquad \nonumber\\
    \abs{K_{XX}}^{1/m}\cdot\left[ \left(\frac{1}{2} - \frac{1}{\pi}\right)\trace K_{XX}^{-1} + \frac{1}{\pi} Y^T K_{XX}^{-1}Y \right].
\end{align}
\end{definition}
Given these definitions, the following KL divergences may be obtained:
\begin{lemma}[KL divergences for Gaussian process classification]\label{lem:kl-gp}
Given a training set $S=(X,Y)$ with binary labels $Y\in\{\pm1\}^m$ and a kernel $k$:
    \begin{align}
        \kl(\qgp \;||\; \pgp) &= \log (1/P_Y) \leq \mathcal{A}(k,X,Y) = \kl(\qsph \;||\; \pgp).
    \end{align}
\end{lemma}
\begin{proof} First, by Lemma \ref{lem:kl-chain} and the observation that the KL divergences from both $\qgp(f_{X^\mathsf{c}} | f_X)$ and $\qsph(f_{X^\mathsf{c}} | f_X)$ to $\pgp(f_{X^\mathsf{c}} | f_X)$ are zero, it is enough to relate $\kl(\qgp(f_X)\;||\;\pgp(f_X))$ to $\kl(\qsph(f_X)\;||\;\pgp(f_X))$.

To establish the first equality, observe that since $\qgp(f_X)$ and $\pgp(f_X)$ differ on the support of $\qgp(f_X)$ only by normalisation constant $P_Y$, it holds that:
\begin{equation*}
    \kl(Q_{\mathrm{GP}}(f_X) \;||\; P_{\mathrm{GP}}(f_X)) = \Expect_{f_X\sim Q_{\mathrm{GP}}} \log(1/P_Y)= \log (1/P_Y).
\end{equation*}
The last equality is derived by first observing that:
\begin{align*}
    &\kl(\qsph (f_X) \;||\; \pgp(f_X)) = \Expect_{f_X \sim \qsph} \log \frac{2^n\cdot\econst^{-\half\norm{f_X}_2^2\cdot\abs{K_{XX}}^{-1/m}}}{\econst^{-\half f_X^\top K_{XX}^{-1}f_X}} \\
    &\qquad\qquad= n \log 2 + \half \Expect_{f_X \sim \qsph} [f_X^\top (K_{XX}^{-1} - \Id\abs{K_{XX}}^{-1/m})f_X].
\end{align*}
To obtain $\mathcal{A}(k,X,Y)$, one must substitute in the following identity for half-Normal random variables:
\begin{gather*}
    \Expect_{f_X \sim \qsph}[f_X^i\cdot f_X^j] = \abs{K_{XX}}^{1/m}\cdot\left[\delta_{ij} + \tfrac{2}{\pi} Y_i Y_j (1-\delta_{ij})\right].
\end{gather*}

Finally, the inequality follows via:
\begin{align*}
    \kl(\qsph(f_X)\;||\;\pgp(f_X)) &= \Expect_{f_X\sim \qsph}\left[\log\tfrac{\qsph(f_X)}{\qgp(f_X)} + \log\tfrac{\qgp(f_X)}{\pgp(f_X)}\right]\\
    &= \kl(\qsph(f_X)\;||\;\qgp(f_X)) + \log(1/P_Y),
\end{align*}
and noting that $\kl(\qsph(f_X)\;||\;\qgp(f_X)) \geq 0$.
\end{proof}

\section{PAC-Bayes bounds}

Given the results in Sections \ref{sec:gpc} and \ref{sec:kl}, it is now a simple matter to write down PAC-Bayesian generalisation bounds for a Gaussian process binary classifier:

\begin{theorem}[PAC-Bayes for Gaussian process classification]\label{thm:pac-bayes-gpc} Given a kernel $k$ and a training sample $S=(X,Y)$, recall the definitions of the kernel complexity $\mathcal{A}$ (Definition \ref{def:k-complex}) and the Gaussian orthant probability $P_Y$ (Definition \ref{def:gop}). Given a data distribution $\mathcal{D}$ over $\mathcal{X}\times\{\pm1\}$, for a fraction $1-\delta$ of training sets $S\overset{\text{iid}}{\sim}\mathcal{D}^m$, the following bounds hold simultaneously:
\begin{align}
\Probe_{f\sim \qgp, (x,y)\sim\mathcal{D}}[\sign f(x)\neq y] &\leq 1 - \exp \left[-\frac{\log 1/P_Y + \log (2m/\delta)}{m-1}\right] \label{eq:gpc-bound}\\
& \leq 1 - \exp \left[-\frac{\mathcal{A}(k,X,Y) + \log (2m/\delta)}{m-1}\right];\label{eq:gpc-bound-slack}\\
\Probe_{f\sim \qsph, (x,y)\sim\mathcal{D}}[\sign f(x)\neq y] &\leq 1 - \exp \left[-\frac{\mathcal{A}(k,X,Y) + \log (2m/\delta)}{m-1}\right]. \label{eq:sph-bound}
\end{align}
\end{theorem}
\begin{proof} Instantiate Theorem \ref{thm:pac-bayes} with prior $\pgp(f)$, the two posteriors $\qgp(f)$ and $\qsph(f)$ and the KL divergences from Lemma \ref{lem:kl-gp}.
\end{proof}

According to Theorem \ref{thm:pac-bayes-gpc}, the posterior $\qgp$ enjoys a tighter risk bound than the spherised posterior $\qsph$. So why introduce the spherised posterior? There are two reasons:
\begin{enumerate}
    \item It is significantly easier to sample from the spherised posterior (Definition \ref{def:gpc-posterior-spherised}) than the original posterior (Definition \ref{def:gpc-posterior}).
    \item Since the kernel complexity measure $\mathcal{A}(k,X,Y)$ is a closed-form analytical expression, whereas the Gaussian orthant probability $P_Y$ requires computing a high-dimensional integral, Inequality \ref{eq:sph-bound} is much easier to evaluate than Inequality \ref{eq:gpc-bound}.
\end{enumerate}

These results will be used in Chapter \ref{chap:bpm} to study generalisation in neural networks through the lens of the neural network--Gaussian process correspondence.

\printbibliography[heading=subbibliography]
\end{refsection}

%% file: chapters/9-nn-bpm-correspondence.tex
\begin{refsection}

\chapter{Neural Networks as Bayes Point Machines}\label{chap:bpm}

\begin{tcolorbox}
This chapter introduces a novel correspondence between neural networks and kernels. In particular, as both width and normalised margin are sent to infinity, the neural network function space concentrates on a particular kernel classifier. This kernel classifier aggregates over all infinitely wide neural networks that correctly classify the train sample. This offers a new, potentially fruitful perspective as to why neural networks generalise.
\end{tcolorbox}

Chapter \ref{chap:gp-pac-bayes} derived PAC-Bayes bounds for Gaussian process classification. These bounds may be readily transferred to neural networks---albeit infinitely wide ones---by leveraging the neural network--Gaussian process correspondence from Section \ref{sec:nngp}. Having these bounds is certainly nice, and they have been found to be non-vacuous \citep{seeger,Prez2020GeneralizationBF}. But what would be nicer is having the bounds inform some aspect of deep learning practice. For instance: could the bounds tell which single function should generalise best?

There is a problem, however. Theorem \ref{thm:pac-bayes-gpc} bounds the misclassification rate of the Gaussian process averaged over posterior draws. This leaves room for individual draws to generalise either significantly worse or significantly better than average. Then which single function should be returned in practice? 

To answer this question, the chapter takes a detour through \textit{Bayesian classification strategies}. A Bayesian classification strategy is a rule for combining an input $x$ with a posterior distribution $Q$ over classifiers to yield a prediction. The most natural of these strategies involve either randomly sampling a single posterior function or aggregating over the posterior. Bayesian wisdom, and also certain technical results \citep{lacasse}, suggest that aggregation over the posterior should perform best. This is for the simple reason that aggregation removes the variance incurred by sampling a posterior function. 

But the trouble with aggregation is that it is expensive. Naïvely, it involves either integrating or summing over lots of posterior functions. To get the benefits of aggregation without the associated cost, this chapter picks up on an old idea from the kernels literature.  A \textit{Bayes point machine} \citep{bpms} is a single posterior sample that, by itself, approximates the posterior aggregate. In the context of neural networks, this suggests the question:
\begin{quote}
    Can a single neural network report the aggregated classification of an ensemble of networks?
\end{quote}
Given the expressive power of the neural network function class, it might seem reasonable that the answer could be yes. In that case:
\begin{quote}
    How can such a network be found?
\end{quote}
This chapter attempts to resolve these two questions. The chapter argues that, in the limit of large width and large normalised margin, the entire space of neural networks that interpolate a training set \textit{concentrates} on a kernel classifier that itself aggregates over a posterior distribution. This implies that a single wide neural network trained to large normalised margin will attain the same reduction in the variance of its predictions as an aggregated Bayesian method.

\section{Bayesian classification strategies}

From a Bayesian perspective, there are three natural ways to use a posterior $Q$ over functions to classify a fresh input $x\in\mathcal{X}$. The first is the random strategy:
\begin{definition}[Gibbs classifier]
The \textit{Gibbs classifier} returns a random draw:
    \begin{equation}
        f_\mathrm{Gibbs}(x) \coloneqq \sign f(x)  \text{ for } f\sim Q.
    \end{equation}
\end{definition}
PAC-Bayes (Theorem \ref{thm:pac-bayes}) bounds the probability that the Gibbs classifier misclassifies a randomly drawn test point. But the Gibbs classifier, being random, has the misfortune of containing variance. A Bayesian would like to deal with this issue by integrating, or \textit{aggregating}, over the posterior, and thus removing this variance. This motivates the second strategy:
\begin{definition}[Bayes classifier]
The \textit{Bayes classifier} returns the majority vote:
    \begin{equation}
        f_\mathrm{Bayes}(x) \coloneqq \sign \Expect_{f\sim Q} \sign f(x).
    \end{equation}
\end{definition}
The majority is one form of aggregation. The third strategy employs another:
\begin{definition}[BPM classifier]
    The \textit{BPM classifier} returns the simple average:
        \begin{equation}
            f_\mathrm{BPM}(x) \coloneqq \sign \Expect_{f\sim Q} f(x).
        \end{equation}
\end{definition}

The abbreviation BPM is short for \textit{Bayes point machine}. Two observations motivate this terminology:
\begin{enumerate}
    \item The BPM classifier is obtained by reversing the order of the sign and expectation operators in the Bayes classifier:
\begin{equation}\label{eq:approx}
\underbrace{\sign \tikzmarknode{a}{\Expect}_{f\sim Q} \,\tikzmarknode{b}{\sign} f(x)}_{\text{Bayes classifier}} \approx \underbrace{\sign \Expect_{f\sim Q}f(x)}_{\text{BPM classifier}}.
\end{equation}\tikz[remember picture, overlay]{\draw[latex-latex] ([yshift=0.15em,xshift=0.5em]a.north) to[bend left] ([yshift=0.15em,xshift=-0.2em]b.north);}%
This operator exchange amounts to approximating the ensemble majority by a single point in function space: the ensemble centre-of-mass. Approximation \ref{eq:approx} is referred to as the \textit{the BPM approximation}. The quality of this approximation will be considered in this chapter.
\item Suppose that the classifier has \textit{hidden linearity} when represented in weight space. In particular, consider classifier $f_{\mathrm{\phi}}(x;w) \coloneqq \phi(x)^\top w$, where $\phi$ is an arbitrary nonlinear input embedding. Then:
\begin{equation}\label{eq:linear}
\underbrace{\sign \Expect_{w\sim Q}f_\phi(x;w)}_{\text{BPM classifier}} = \sign f_\phi(x;\underbrace{\Expect_{w\sim Q}w}_{\mathclap{\text{weight space centre-of-mass}}}).
\end{equation}
In words: for linear classifiers, the BPM classifier is equivalent to a single point in weight space: the posterior $Q$'s weight space centre-of-mass.
\end{enumerate}

Of course, since the $\sign$ function is nonlinear, the BPM approximation is not always correct. \citet{herbrich_book} calls it a \textit{trick}. Is the approximation ever correct? In the case of hidden linearity (Equation \ref{eq:linear}), the approximation is correct when over half the ensemble agrees with the centre-of-mass on an input. This happens, for example, when the posterior $Q$ is point symmetric about the centre-of-mass \citep{herbrich_book}. But point symmetry is a strong assumption that does not hold for, say, the posterior of a GP classifier (Definition \ref{def:gpc-posterior}).

The next section presents a novel result on the quality of the BPM approximation. This result leverages a novel connection between Bayesian classification strategies and certain objects of study in convex geometry \citep{grunbaum} and social choice theory \citep{meanvoter}. The result leads to a novel bound on the generalisation error of the BPM classifier.

\section{Relationships between classification strategies}

This section presents relations, one of which is novel, between the test error of the three classification strategies introduced in the previous section. In each case, the test error is measured over a data distribution $\mathcal{D}$ on $\mathcal{X}\times\{\pm1\}$. It will help to formally define the three notions of error considered:

First, the Gibbs error measures the misclassification rate averaged over both the data distribution and the posterior:
\begin{definition}[Gibbs error] The \textit{Gibbs error} $\gibbserr\in[0,1]$ is given by:
    \begin{equation}
        \gibbserr \coloneqq \Expect_{f\sim Q} \Expect_{(x,y)\sim\mathcal{D}}\mathbb{I}\big[\sign f(x)\neq y\big].
    \end{equation}
\end{definition}

Meanwhile, the Bayes error measures the misclassification rate of the posterior majority, averaged over the data distribution:
\begin{definition}[Bayes error]The \textit{Bayes error} $\bayeserr \in [0,1]$ is given by:
    \begin{equation}
        \bayeserr \coloneqq \Expect_{(x,y)\sim\mathcal{D}}\mathbb{I}\big[f_\mathrm{Bayes}(x)\neq y\big];
    \end{equation}
\end{definition}

Finally, the BPM error measures the misclassification rate of the posterior mean, averaged over the data distribution:
\begin{definition}[BPM error]The \textit{BPM error} $\bpmerr\in[0,1]$ is given by:
    \begin{equation}
        \bpmerr \coloneqq \Expect_{(x,y)\sim\mathcal{D}}\mathbb{I}\big[f_\mathrm{BPM}(x)\neq y\big].
    \end{equation}
\end{definition}

Various relationships exist between these three notions of error. A classic example is that the Bayes error cannot be more than twice the Gibbs error:
\begin{lemma}[Pessimistic Gibbs--Bayes]\label{lem:gibbs-bayes} For any ensemble of classifiers $Q$,
\begin{equation*}
    \bayeserr \leq 2 \cdot \gibbserr.
\end{equation*}
\end{lemma}
\begin{proof}
    First, consider the Bayes and Gibbs errors on a single datapoint $(x,y)$:
    \begin{align*}
    \bayeserr(x,y) &\coloneqq \mathbb{I}\left[\sign \Expect_{f\sim Q}\sign f(x)\neq y\right]; \\
    \gibbserr(x,y)&\coloneqq\Expect_{f\sim Q} \mathbb{I}\left[\sign f(x)\neq y\right].
    \end{align*}
    When the Bayes classifier is correct, $\bayeserr(x,y)=0$. When the Bayes classifier is incorrect, $\bayeserr(x,y)=1$ and $\gibbserr(x,y) \geq 1/2$. In either case:
    \begin{align*}
        \bayeserr(x,y)\leq 2\cdot \gibbserr(x,y).
    \end{align*}
    Taking the expectation over $(x,y)\sim\mathcal{D}$ yields the result.
\end{proof}

This result is tagged \textit{pessimistic} since one often expects the Bayes classifier to significantly outperform the Gibbs classifier: $\bayeserr \ll \gibbserr$. This is because the Gibbs classifier is noisy, whereas the Bayes classifier aggregates over this noise. For this reason, \citet{seeger} referred to Lemma \ref{lem:gibbs-bayes} as \textit{crude}.

A potentially less crude relationship is given by the following lemma:
\begin{lemma}[Optimistic Gibbs--Bayes]\label{lem:gibbs-bayes-opt} Define the average Gibbs agreement:
\begin{equation*}
\alpha_{\mathrm{Gibbs}}\coloneqq\Expect_{x\sim\mathcal{D}}\left[\left[\Expect_{f\sim Q}\sign f(x)\right]^2\right]\in[0,1].
\end{equation*}
Then, for any ensemble of classifiers $Q$,
\begin{equation*}
    \bayeserr \leq 1 - \frac{(1-2 \cdot \gibbserr)^2}{\alpha_{\mathrm{Gibbs}}}.
\end{equation*}
\end{lemma}
This result is usually known by a different name: \textit{the $\mathcal{C}$-bound}. Its proof is given by \citet{lacasse}. The result is tagged \textit{optimistic} since it is capable of expressing that the Bayes classifier can significantly outperform the Gibbs classifier: $\bayeserr \ll \gibbserr$. In particular, this happens when the ensemble members make very noisy predictions, such that the Gibbs error $\gibbserr$ is large but the Gibbs agreement $\alpha_{\mathrm{Gibbs}}$ is small.

This thesis proves a novel relationship, analogous to Lemma \ref{lem:gibbs-bayes}, between the BPM error and the Gibbs error. The result leverages the following relationship between (sub)majorities and averages:

\begin{lemma}[Weighted Gr\"unbaum's inequality]\label{lem:grunbaum}
Let $Q$ be a log-concave probability density supported on a convex subset of $\R^d$ with positive volume. Let $\mu\in\R^d$ denote the mean $\mu\coloneqq \Expect_{w\sim Q}w$. Then for any vector $x\in\R^d$:
\begin{equation*}
\Probe_{w\sim Q} \left[\sign [w^\top x] = \sign [\mu^\top x]\right] \geq 1/e.
\end{equation*}
\end{lemma}
In words: for any input $x$, a fraction of at least $1/\econst\approx36\%$ of the distribution reports the same binary classification as the mean. This result is due to economists \citet{meanvoter}, who were working in social choice theory. Their interest was in understanding what fraction of an electorate can disagree with the most average individual voter. The result generalises an inequality of \citet{grunbaum} on mass partitions in convex geometry.

Lemma \ref{lem:grunbaum} leads directly to the following analogue of Lemma \ref{lem:gibbs-bayes}:
\begin{lemma}[Pessimistic Gibbs--BPM]\label{lem:gibbs-bpm} Consider an ensemble of classifiers whose distribution at all inputs $x\in\mathcal{X}$ follows:
\begin{equation*}
    f_\phi(x;w) = w^\top \phi(x), \qquad w\sim Q,
\end{equation*}
for arbitrary nonlinear input embedding $\phi$, and log-concave probability density $Q$ supported on a convex subset of $\R^d$ with positive volume. Then:
\begin{equation*}
    \bpmerr \leq \econst \cdot \gibbserr.
\end{equation*}
\end{lemma}
\begin{proof} The proof mirrors the structure of the proof of Lemma \ref{lem:gibbs-bayes}. First, consider the BPM and Gibbs error on a single datapoint $(x,y)$:
    \begin{align*}
    \bpmerr(x,y) &\coloneqq \mathbb{I}\left[\sign \Expect_{w\sim Q}w^\top \phi(x)\neq y\right]; \\
    \gibbserr(x,y)&\coloneqq\Expect_{w\sim Q} \mathbb{I}\left[\sign w^\top \phi(x)\neq y\right].
    \end{align*}
    When the BPM classifier is correct, $\bpmerr(x,y)=0$. When the BPM classifier errs, $\bpmerr(x,y)=1$ and $\gibbserr(x,y) \geq 1/e$ by Lemma \ref{lem:grunbaum}. In either case:
    \begin{align*}
        \bpmerr(x,y)\leq \econst\cdot \gibbserr(x,y).
    \end{align*}
    Taking the expectation over $(x,y)\sim\mathcal{D}$ yields the result.
\end{proof}

So, under the stated conditions of Lemma \ref{lem:gibbs-bpm}, the BPM error cannot be more than $\econst$ times the Gibbs error. The result is tagged \textit{pessimistic} since, in practice, one might expect the BPM classifier to perform significantly better than the noisy Gibbs classifier: $\bpmerr\ll\gibbserr$. While proving this intuition appears to be an open problem, the rest of this section provides one potential route.

The idea is that, when the Gibbs classifier is very noisy, Lemma \ref{lem:gibbs-bayes-opt} suggests a more optimistic relationship between the Gibbs and Bayes errors. But if the BPM approximation (Approximation \ref{eq:approx}) is good, then the BPM classifier should inherit the same favourable properties as the Bayes classifier. To pursue this idea, it will help to formalise the BPM approximation error:

\begin{definition}[BPM approximation error] The \textit{BPM approximation error} $\Delta$ is defined as follows: 
\begin{equation}
    \Delta \coloneqq \Expect_{(x,y)\sim\mathcal{D}}\mathbb{I}\big[f_\mathrm{BPM}(x)\neq f_\mathrm{Bayes}(x)\big].
\end{equation}
\end{definition}
So the BPM approximation error $\Delta$ measures at what rate the BPM classifier and the Bayes classifier disagree. The BPM approximation error $\Delta$ relates the BPM and Bayes errors as follows:

\begin{lemma}[Bayes--BPM]\label{lem:bayes-bpm} For any ensemble of classifiers $Q$,
    \begin{equation*}
        \bpmerr \leq \bayeserr + \Delta.
    \end{equation*}
\end{lemma}
\begin{proof}
    First consider the BPM error, Bayes error and BPM approximation error on a single datapoint $(x,y)$:
    \begin{align*}
    \bpmerr(x,y) &\coloneqq \mathbb{I}\left[f_\mathrm{BPM}(x)\neq y\right]; \\
    \bayeserr(x,y)&\coloneqq \mathbb{I}\left[f_\mathrm{Bayes}(x)\neq y\right] ;\\
    \Delta(x,y) &\coloneqq \mathbb{I}\big[f_\mathrm{BPM}(x)\neq f_\mathrm{Bayes}(x)\big].
    \end{align*}
    When the BPM classifier is correct, $\bpmerr(x,y)=0$. Otherwise, $\bpmerr(x,y)=1$ and either $\bayeserr(x,y)=1$ and $\Delta(x,y)=0$ or vice versa. Thus:
    \begin{align*}
    \bpmerr(x,y) \leq \bayeserr(x,y) + \Delta(x,y).
    \end{align*}
    Taking the expectation over $(x,y)\sim\mathcal{D}$ yields the result.
\end{proof}

Lemmas \ref{lem:gibbs-bayes-opt} and \ref{lem:bayes-bpm} may be directly combined to yield the following result:

\begin{lemma}[Optimistic Gibbs--BPM]\label{lem:gibbs-bpm-opt}
    Let $\alpha_{\mathrm{Gibbs}}$ denote the average Gibbs agreement (Lemma \ref{lem:gibbs-bayes-opt}) and let $\Delta$ denote the BPM approximation error. Then:
    \begin{equation*}
        \bpmerr \leq 1 - \frac{(1-2 \cdot \gibbserr)^2}{\alpha_{\mathrm{Gibbs}}} + \Delta.
    \end{equation*}
\end{lemma}
In words: when the BPM classifier is a good approximation to the Bayes classifier, and when the Gibbs classifier is noisy such that the Gibbs error is large but the Gibbs agreement is small, then the BPM classifier can substantially outperform the Gibbs classifier.

\section{Kernel interpolation as a Bayes point machine}
\label{sec:k-bpm}

This section shows that kernel interpolators of minimum RKHS norm (Equation \ref{eq:min-rkhs-interpolator}) are themselves Bayes point machine classifiers. While, in itself, this observation is not novel \citep{seeger}, it yields novel PAC-Bayes generalisation bounds for kernel interpolators when combined with the novel Lemma \ref{lem:gibbs-bpm}.

To begin, the BPM classifier of the Gaussian process posterior $\qgp$ is the sign of the kernel interpolator of \textit{centre-of-mass labels} $\Expect_{f_X\sim \qgp}f_X$:

\begin{lemma}[BPM of a Gaussian process classifier is a kernel interpolator]\label{lem:gpc-bpm} For the Gaussian process classification posterior $\qgp$ (Definition \ref{def:gpc-posterior}),
\begin{equation}
    f_{\mathrm{BPM}}(x) = \sign [K_{xX} K_{XX}^{-1} \Expect_{f_X\sim \qgp}f_X ].
\end{equation}
\end{lemma}
\begin{proof}
The Gibbs classifier of $\qgp$ classifies a test point $x$ in three steps:
\begin{align*}
\text{Sample train outputs:}\quad &f_X \sim \normal\big(0,K_{XX} \;|\; \sign f_X = Y \big);\\
\text{Sample noise:}\quad &\xi \sim \normal\left(0, K_{xx} - K_{xX}K_{XX}^{-1} K_{Xx}\right);\\
\text{Return:}\quad &\sign [K_{xX} K_{XX}^{-1} f_X + \xi ].
\end{align*}
Then, recalling that the BPM classifier is obtained by reversing the order of sign and expectation in the Bayes classifier, the BPM classifier is given by:
\begin{align*}
	&\sign \tikzmarknode{a}{\Expect}_{\xi,f_X}\,\tikzmarknode{b}{\sign} [K_{xX} K_{XX}^{-1} f_X + \xi ] \qquad \tag{Bayes classifier}\\
	&\qquad\approx\sign \Expect_{\xi,f_X} [K_{xX} K_{XX}^{-1} f_X + \xi ]\tag{BPM classifier}\\
	&\qquad=\sign K_{xX} K_{XX}^{-1} \Expect_{f_X\sim \qgp}f_X. \tag{kernel interpolator}
\end{align*}
\tikz[remember picture, overlay]{\draw[latex-latex] ([yshift=0.15em,xshift=0.5em]a.north) to[bend left] ([yshift=0.15em,xshift=-0.2em]b.north);}%
This completes the proof.
\end{proof}
And second, the BPM of the spherised Gaussian process posterior $\qsph$ is the sign of the kernel interpolator of \textit{centroidal labels} $Y$:
\begin{lemma}[BPM of a spherised Gaussian process classifier is a kernel interpolator]\label{lem:gpc-spherised-bpm} For the spherised posterior $\qsph$ (Definition \ref{def:gpc-posterior-spherised}),
\begin{equation}
    f_{\mathrm{BPM}}(x) = \sign [K_{xX} K_{XX}^{-1} Y ].
\end{equation}
\end{lemma}
\begin{proof}
The Gibbs classifier of $\qsph$ classifies a test point $x$ in three steps:
\begin{align*}
\text{Sample train outputs:}\quad &f_X \sim \normal\big(0,\Id\cdot \abs{K_{XX}}^{1/m} \;|\; \sign f_X = Y \big);\\
\text{Sample noise:}\quad &\xi \sim \normal\left(0, K_{xx} - K_{xX}K_{XX}^{-1} K_{Xx}\right);\\
\text{Return:}\quad &\sign [K_{xX} K_{XX}^{-1} f_X + \xi ].
\end{align*}
Then the BPM classifier is given by exchanging operators in the Bayes classifier:
\begin{flalign*}
	&\sign \tikzmarknode{a}{\Expect}_{\xi,f_X}\,\tikzmarknode{b}{\sign} [K_{xX} K_{XX}^{-1} f_X + \xi ] && \tag{Bayes classifier}\\
	&\qquad\approx\sign \Expect_{\xi,f_X} [K_{xX} K_{XX}^{-1} f_X + \xi ]\tag{BPM classifier} &&\\
	&\qquad=\sign K_{xX} K_{XX}^{-1} \Expect_{f_X\sim \qsph}f_X = \sign K_{xX} K_{XX}^{-1}Y.&& \tag{kernel interpolator}
\end{flalign*}
\tikz[remember picture, overlay]{\draw[latex-latex] ([yshift=0.15em,xshift=0.5em]a.north) to[bend left] ([yshift=0.15em,xshift=-0.2em]b.north);}%
This completes the proof.
\end{proof}

These results lead to the following PAC-Bayesian generalisation guarantees for minimum RKHS norm kernel interpolation.

\begin{theorem}[Kernel PAC-Bayes]\label{thm:k-bpm} Given a kernel $k$ and a training sample $S=(X,Y)$, recall the definitions of the kernel complexity $\mathcal{A}$ (Definition \ref{def:k-complex}) and the Gaussian orthant probability $P_Y$ (Definition \ref{def:gop}). Let $\mathcal{D}$ be a data distribution over $\mathcal{X}\times\{\pm 1\}$. For the minimum RKHS norm kernel interpolator of data sample $(X,\Upsilon)$ given by $f_{\Upsilon}(x) := K_{xX} K_{XX}^{-1} \Upsilon$, define the test error:
\begin{align*}
    \eps[\Upsilon] &:= \Expect_{(x,y)\sim\mathcal{D}} \mathbb{I}[\sign f_{\Upsilon}(x) \neq y].
\end{align*}
Then, with probability $1-\delta$ over a training sample $(X,Y) \overset{\text{iid}}{\sim}\mathcal{D}^m$, the following bounds hold simultaneously:
\begin{align}
\eps[\Expect_{f_X\sim \qgp}f_X] &\leq \econst \cdot\left[ 1 - \exp \left(-\frac{\log 1/P_Y + \log (2m/\delta)}{m-1}\right)\right]\label{eq:k-pb1}\\
& \leq \econst \cdot\left[ 1 - \exp \left(-\frac{\mathcal{A}(k,X,Y) + \log (2m/\delta)}{m-1}\right)\right];\label{eq:k-pb2}\\
\eps[Y] &\leq \econst \cdot\left[ 1 - \exp \left(-\frac{\mathcal{A}(k,X,Y) + \log (2m/\delta)}{m-1}\right)\right].\label{eq:k-pb3}
\end{align}
\end{theorem}
\begin{proof} First, consider the Gaussian process posterior $\qgp$ (Definition \ref{def:gpc-posterior}). The Gibbs error of this classifier is bounded by Inequalities \ref{eq:gpc-bound} and \ref{eq:gpc-bound-slack}. Observe that for all inputs $x\in\mathcal{X}$, the value $f(x)\sim\qgp$ follows:
\begin{equation*}
    f(x) = (f_X, z)^\top (K_{XX}^{-1} K_{Xx}, K_{xx}-K_{xX}K_{XX}^{-1}K_{Xx}),
\end{equation*}
for $f_X\sim\qgp$ and $z\sim\normal(0,1)$. But this distribution over $(f_X,z)$ is log-concave and supported on a convex subset of $\R^{m+1}$ with positive volume. Therefore, Inequalities \ref{eq:k-pb1} and \ref{eq:k-pb2} follow from Inequalities \ref{eq:gpc-bound} and \ref{eq:gpc-bound-slack} by Lemmas \ref{lem:gibbs-bpm} and \ref{lem:gpc-bpm}. Inequality \ref{eq:k-pb3} follows by applying the same argument but switching $\qgp$ for $\qsph$.
\end{proof}

\section{Max-margin neural networks as Bayes point machines}

This section presents the novel argument that infinitely wide neural networks, fit to large normalised margin, are Bayes point machine classifiers. The argument works by considering what happens to the neural network--Gaussian process posterior distribution as a notion of normalised margin is taken large. In short, this posterior distribution concentrates on a minimum RKHS norm kernel interpolator, which is itself a Bayes point machine by the results of Section \ref{sec:k-bpm}.

Consider a hyperspherical input space $\mathcal{X} = \sqrt{d_0}\cdot\Sph^{d_0-1}$ and a multilayer perceptron (Definition \ref{def:mlp}) with $L$ layers and scaled relu nonlinearity $\phi(\cdot) = \sqrt{2} \cdot \max(0,\cdot)$. For each layer $l=1,...,L$, consider sampling the weight matrix at that layer $W_l \in \R^{d_l \times d_{l-1}}$ from the distribution $\normal(0,\Id/d_{l-1})$. Then, as the layer widths $d_1,...,d_{l-1}\to\infty$, the corresponding functions are distributed:
\begin{equation}
    f \sim \gp(0,k_\mathrm{arccos}),
\end{equation}
where the kernel $k_\mathrm{arccos}$ is the compositional arccosine kernel of Theorem \ref{thm:relu}.

Now suppose that, at all layers $l=1,...,L$, the weight matrices are instead sampled $W_l \sim \normal(0,\sigma^2\cdot\Id/d_{l-1})$ for a choice of ``normalisation'' $\sigma>0$. The core observation is that the relu multilayer perceptron is homogeneous of degree $L$ in its weight vector, meaning that $f(x;\sigma\cdot w) = \sigma^L \cdot f(x;w)$ for all inputs $x\in\mathcal{X}$. This means that, as the layer widths $d_1,...,d_{l-1}\to\infty$, the resulting functions are distributed according to:
\begin{equation}
    f \sim \gp(0,\sigma^{2L}\cdot k_\mathrm{arccos}).
\end{equation}
Next, condition this distribution of functions on interpolating a training sample $(X,\gamma\cdot Y)$, where the labels $Y$ were scaled by a ``margin'' $\gamma>0$:
\begin{equation}\label{eq:nngp-norm-margin}
    f\sim\gp(0, \sigma^{2L} \cdot k_\mathrm{arccos} \mid f_X=\gamma\cdot Y ).
\end{equation}
By Theorem \ref{thm:force-gp}, the distribution of $f/\gamma$ concentrates on the kernel interpolator $x\mapsto K_{xX}K_{XX}^{-1}Y$ in the limit that the ``normalised margin'' $\gamma/\sigma^L$ is sent to infinity. In this expression, $K_{xX}$ and $K_{XX}$ are the Gram vector and Gram matrix corresponding to the unscaled kernel $k_\mathrm{arccos}$.

In summary: by defining a notion of \textit{normalised margin} for the neural network--Gaussian process posterior, and taking this normalised margin to infinity, the posterior concentrates on a minimum RKHS norm kernel interpolator. This function is itself a Bayes point machine by the results of Section \ref{sec:k-bpm}. 

This behaviour was tested experimentally, and the results are displayed in Figure \ref{fig:nn-margin}. The plots show the test accuracy of both the NNGP posterior as a function of normalised margin (Equation \ref{eq:nngp-norm-margin}), and also large but finite width multilayer perceptrons trained by a variant of gradient descent as a function of Frobenius-normalised margin (Definition \ref{def:frob-norm-margin}). Qualitatively similar behaviour was observed in both cases: the average of many functions of small normalised margin attained similar accuracy to one function of large normalised margin. 
\begin{figure}[p]
    \centering
    \includegraphics{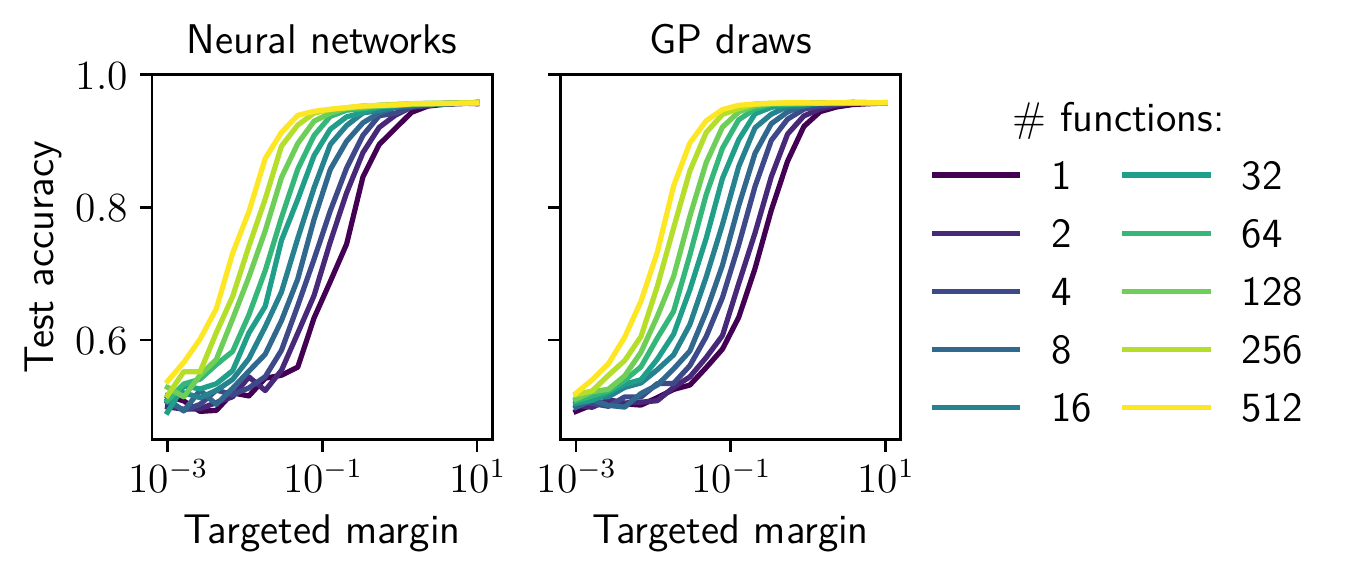}
    \caption[Test accuracy as a function of normalised margin]{Test accuracy as a function of normalised margin, for both width-2048 neural networks (left) and neural network--Gaussian processes (right). The task was binary classification of MNIST digits \citep{lecun2010mnist} using a 5-layer multilayer perceptron. Each curve shows the test accuracy of the average over a number of functions, where each function in the average attains the specified value of normalised margin. The main behaviour visible is that one function of large normalised margin appears equivalent to the average of many functions of small normalised margin. The NNGP experiments used Equation \ref{eq:nngp-norm-margin} to control normalised margin. The neural network experiments controlled Frobenius-normalised margin by minimising square loss with respect to rescaled labels, and projecting every weight matrix to a hypersphere of fixed Frobenius norm at each iteration.}
    \label{fig:nn-margin}
\end{figure}\clearpage

\section{Empirical comparisons between classification strategies}

This section reports an experimental comparison of various binary classification strategies for both Gaussian processes and neural networks. The task was binary classification of MNIST digits \citep{lecun2010mnist} using a 7-layer relu multilayer perceptron. The width was set to either 1000 or infinity (via the neural network--Gaussian process correspondence).

For neural network--Gaussian processes, the classification strategies tested were:
\begin{enumerate}
    \item \textit{Gibbs classifier:} the sign of a random posterior sample.
    \item \textit{Bayes classifier:} the majority vote over the posterior.
    \item \textit{BPM classifier:} the sign of the minimum RKHS norm kernel interpolator.
\end{enumerate}
The spherised Gaussian process posterior (Definition \ref{def:gpc-posterior-spherised}) was used for reasons of computational tractability. Also, several generalisation bounds for Gaussian processes and kernel classifiers are plotted:
\begin{enumerate}
    \item \textit{Rademacher bound:} a uniform convergence bound for kernel classifiers \citep[Theorem 21]{rademacher}.
    \item \textit{Gibbs bound:} Inequality \ref{eq:sph-bound} of Theorem \ref{thm:pac-bayes-gpc}.
    \item \textit{BPM bound:} Inequality \ref{eq:k-pb3} of Theorem \ref{thm:k-bpm}.
\end{enumerate}

For finite width neural networks, the classification strategies considered were:
\begin{enumerate}
    \item \textit{Gibbs classifier:} train a randomly initialised network to fit the train sample to small Frobenius-normalised margin.
    \item \textit{Bayes classifier:} take the majority vote over 501 networks trained from different random initialisations to small Frobenius-normalised margin.
    \item \textit{BPM classifier:} train a randomly initialised network to fit the train sample to large Frobenius-normalised margin.
\end{enumerate}

The results for Gaussian processes and kernel classifiers are presented in
Figure \ref{fig:k-bpm}, while the results for neural networks are presented in Figure \ref{fig:nn-bpm}. The results support the idea that minimum RKHS norm kernel interpolators, and large Frobenius-normalised margin neural networks, are Bayes point machines.

\begin{figure}[p]
    \centering
    \includegraphics{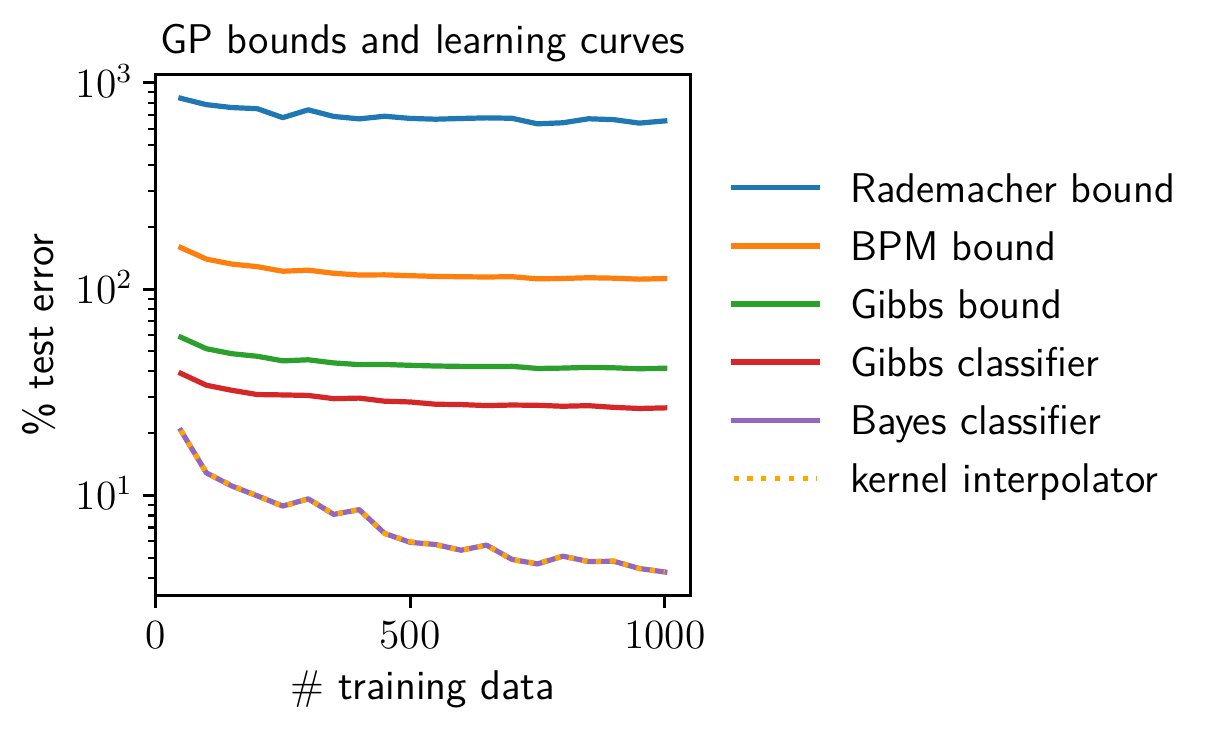}
    \caption[Testing classification strategies for Gaussian processes and kernels]{Testing classification strategies for Gaussian processes and kernels, on an MNIST \citep{lecun2010mnist} binary classification task. The Bayes classifier and kernel interpolator attain indistinguishable performance, supporting the claim that minimum RKHS norm kernel interpolation is a Bayes point machine. Despite the kernel interpolator significantly outperforming the Gibbs classifier in practice, the order of the Gibbs and BPM bounds are reversed. Still, the BPM bound is substantially smaller than the Rademacher bound.}
    \label{fig:k-bpm}
\end{figure}
\begin{figure}[p]
    \centering
    \hspace{.25em}\includegraphics{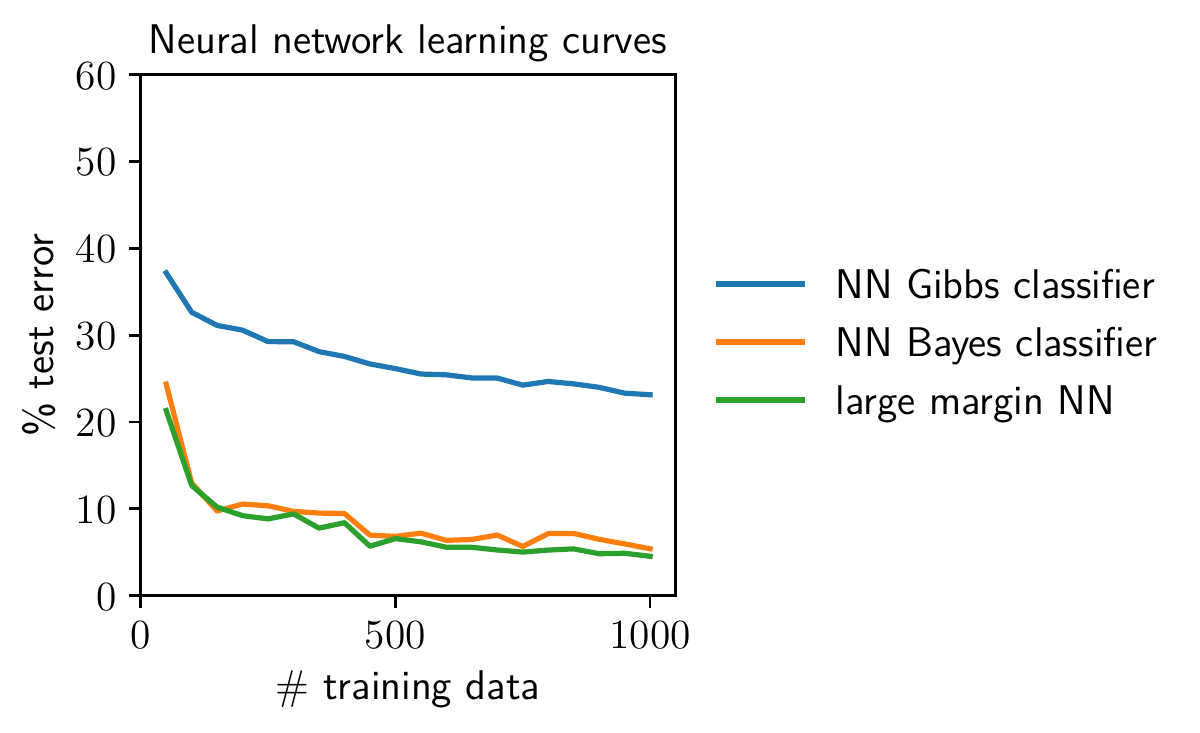}
    \caption[Testing classification strategies for neural networks]{Testing classification strategies for neural networks, on the same task as Figure \ref{fig:k-bpm}. The Bayes classifier reports a majority vote over 501 small margin networks. It attains similar (though not identical) performance to a single neural network of large Frobenius-normalised margin. This supports the idea that large margin neural networks are Bayes point machines. Both classifiers substantially outperform the corresponding Gibbs classifier.}
    \label{fig:nn-bpm}
\end{figure}

\printbibliography[heading=subbibliography]
\end{refsection}